\documentclass[11pt, oneside, hidelinks]{article}
\usepackage{natbib}
\setcitestyle{authoryear,open={(},close={)}}
\renewcommand{\cite}{\citep}
\usepackage{tensor}
\usepackage{amsmath}
\usepackage{amssymb}
\usepackage{amsthm}
\usepackage{graphicx}
\usepackage{color}
\usepackage{macro_math}

\usepackage{hyperref}
\usepackage[us,12hr]{datetime}

\usepackage[margin=1in]{geometry}
\numberwithin{equation}{section}

\newcommand{\veps}{\varepsilon}
\DeclareMathOperator*{\E}{\mathbb{E}}
\newcommand{\tw}{\tilde{w}}
\newcommand{\tDelta}{\tilde{\Delta}}
\newcommand{\tdelta}{\tilde{\delta}}
\newcommand{\tPhi}{\tilde{\Phi}}
\newcommand{\tv}{\tilde{v}}
\newcommand{\tcP}{\tilde{\mathcal{P}}}
\newcommand{\set}[1]{\mathcal{#1}}
\DeclareMathOperator{\poly}{poly}
\DeclareMathOperator*{\Tr}{Tr}
\DeclareMathOperator*{\relu}{ReLU}
\DeclareMathOperator*{\abs}{abs}
\DeclareMathOperator*{\rank}{rank}

\DeclareMathOperator*{\argmax}{arg\,max}
\newcommand{\obj}{L}

\newcommand{\obje}{\hat{L}}
\newcommand{\objpop}{L}
\newcommand{\objinf}{L_{\infty}}
\newcommand{\cs}{{conditionally-symmetric}}
\newtheorem{theorem}{Theorem}[section]
\newtheorem{lemma}[theorem]{Lemma}
\newtheorem{claim}{Claim}[section]
\newtheorem{definition}{Definition}[section]
\newtheorem{proposition}[theorem]{Proposition}
\newtheorem{corollary}[theorem]{Corollary}

\newcommand{\Kappa}{\kappa_1}
\newcommand{\Kapppa}{\kappa_2}
\newcommand{\Kappppa}{\kappa_3}
\newcommand{\lambdaa}{\lambda_1}
\newcommand{\cz}{Q}

\newcommand{\cP}{\mathcal{P}}
\newcommand{\cH}{\mathcal{H}}
\newcommand{\ctP}{\tilde{\mathcal{P}}}
\newcommand{\ctH}{\tilde{\mathcal{H}}}

\def\shownotes{0}  \ifnum\shownotes=1
\newcommand{\authnote}[2]{{[#1: #2]}}
\else
\newcommand{\authnote}[2]{}
\fi
\newcommand{\tnote}[1]{{\color{blue}\authnote{TM}{#1}}}

\begin{document}
\title{{Learning Over-Parametrized Two-Layer ReLU Neural Networks beyond NTK}}
\author{Yuanzhi Li\thanks{Carnegie Mellon University. yuanzhil@andrew.cmu.edu}\and Tengyu Ma \thanks{Stanford University. tengyuma@stanford.edu}\and Hongyang R. Zhang\thanks{University of Pennsylvania. hongyang90@gmail.com}}
\date{\today}
\maketitle

\begin{abstract}
	We consider the dynamic of gradient descent for learning a two-layer neural network. We assume the input $x\in\mathbb{R}^d$ is drawn from a Gaussian distribution and the label of $x$ satisfies  $f^{\star}(x) = a^{\top}|W^{\star}x|$, where $a\in\mathbb{R}^d$ is a nonnegative vector and $W^{\star} \in\mathbb{R}^{d\times d}$ is an orthonormal matrix. We show that an \emph{over-parametrized} two-layer neural network with ReLU activation, trained by gradient descent from \emph{random initialization}, can provably learn the ground truth network with population loss at most $o(1/d)$ in polynomial time with polynomial samples. On the other hand, we prove that any kernel method, including Neural Tangent Kernel, with a polynomial number of samples in $d$, has population loss at least $\Omega(1 / d)$.
\end{abstract}

\section{Introduction}

Gradient-based optimization methods are the method of choice for learning neural networks. However, it has been challenging to understand their working on non-convex functions. Prior works prove that stochastic gradient descent provably convergences to an approximate local optimum \cite{ge2015escaping, sun2015nonconvex,  lee2017first,kleinberg2018alternative}.
Remarkably, for many highly complex neural net models, gradient-based methods can also find high-quality solutions \cite{sun2019optimization} and interpretable features \cite{zeiler2014visualizing}.%

Recent studies  made the connection between training wide neural networks and Neural Tangent Kernels (NTK) \cite{jacot2018neural,arora2019finegrained,cao2019generalization,du2018gradient}.
The main idea is that training neural networks with gradient descent with a particular initialization is equivalent to using kernel methods.
However, the NTK approach has not yet provided a fully satisfactory theory for explaining the  success of neural networks. Empirically, there seems to be a non-negligible gap between the test performance of neural networks trained by SGD and that of the  NTK \cite{arora2019exact,li2019enhanced}. Recent works have suspected that the gap stems from that the NTK approach has difficulty dealing with non-trivial explicit regularizers or does not sufficiently leverage the implicit regularization of the algorithm \cite{wei2019regularization,chizat2018note,li2019towards,haochen2020shape}.



In this work, we provide a new convergence analysis of the gradient descent dynamic on an over-parametrized two-layer ReLU neural network.
We prove that for learning a certain two-layer target network with orthonormal ground truth weights, gradient descent is provably more accurate than any kernel method that uses polynomially large feature maps.

\subsection{Setup and Main Result}

We assume that the input $x\in\real^d$ is drawn from the Gaussian distribution $\cN(0, \id_{d\times d})$.
We focus on the realizable setting, i.e. the label of $x$ is generated according to a target network $f^\star$ with $d$ neurons.
We study a two-layer target neural network with absolute value activation:
\begin{align}
	f^{\star}(x) = \sum_{i=1}^d {a_i} \Abs{{w_i^{\star}}^{\top}x}, \label{eq_abs}
\end{align}
where $a_i$ is in $[\frac 1 {\kappa d}, \frac {\kappa} d]$ for an absolute constant $\kappa \geq 1$ and satisfies $\sum_{i \in [d]} a_i= 1$, and $\{w_i^{\star}\}_{i=1}^d$ forms an orthonormal basis.
Equation \eqref{eq_abs} can also be written as the sum of $2d$ neurons with ReLU activation:
\[ f^{\star}(x) = \sum_{i=1}^d {a_i}\left(\relu({w_i^{\star}}^{\top}x) + \relu(-{w_i^{\star}}^{\top}x)\right). \]%
Let $\set{Z} = \{(x_j, y_j) \}_{j = 1}^N$ be a training dataset of $N$ i.i.d. samples from the Gaussian distribution with identity covariance and $y_j = f^{\star}(x_j)$ for any $1\le j\le N$.

We learn the target network $f^{\star}$ using an over-parametrized two-layer ReLU  network with $m \ge 2d$ neurons $W = \{w_i\}_{i = 1}^m$, given by:
\begin{align}
	f_W(x) = \frac 1 m \sum_{i=1}^m \norm{w_i}\cdot \relu(w_i^{\top}x). \label{eq_labeling}
\end{align}
Note that we have re-parametrized the output layer with the norm of the corresponding neuron, so that we only have one set of parameters $W$.
This is without loss of generality for learning $f^{\star}$ because when $a_i \geq 0$, $a_i \cdot \relu(w_i^{\top}x)$ is equal to $\norm{w_i'} \cdot \relu({{w_i'}^{\top}x})$ where $w_i'= \sqrt{a_i / \norm{w_i}}\cdot w_i$. 
Given a training dataset $\set{Z} = \{(x_i, y_i)\}_{i = 1}^N$, we learn the target network by minimizing the following empirical loss: 
\[ \obje(W) = \frac{1}{N}\sum_{i=1}^N \left( f_W(x_i) - y_i \right)^2.\]
Let $\objpop(W)$ denote the population loss given by the expectation of $\hat{L}(W)$ over $\set{Z}$.

\bigskip
\noindent\textbf{Algorithm.} We focus on truncated gradient descent with random initialization.
Algorithm \ref{alg} describes the procedure.%
An interesting feature is that when a neuron becomes larger than a certain threshold, we no longer update the neuron.
This is a variant of gradient clipping often used in training recurrent neural networks (e.g. \citet{merity2017regularizing,gehring2017convolutional,peters2018deep}) --- here we drop the gradients of the large weights instead of re-scaling them.
The truncation allows us to upper bound the norm of every neuron. 
Our main result is to show that Algorithm \ref{alg} learns the target network accurately in polynomially many iterations.

\begin{algorithm}[t!]
\caption{Truncated gradient descent for two-layer neural nets} \label{alg}
\begin{algorithmic}
	\Input{A training dataset $\set{Z}$.}
	\Req{Network width $m$, learning rate $\eta$, truncation parameters $\lambda_0, \lambda_1$.}
	\Output{The final learned network $\hat{W} = \Set{w_i\in\real^d}_{i=1}^m$.}
	\State Initialization. Initialize $w^{(0)}_i \sim \cN\left(0, \frac 1 {d}\cdot\id_{d\times d}\right)$, for $1\le i\le m$.
	\State Stage 1. Let $\lambda = \lambda_0 = \Theta(\frac{1}{\poly(d)})$. For $t \le \Theta\left( \frac{d^2}{ \eta C(\kappa)\log d} \right)$, update every neuron as follows:
		\[w^{(t+1)}_i = w^{(t)}_i - \eta \cdot \indi{\| w_i^{(t)} \|_2^2 \leq \frac{1}{\lambda_0}} \cdot \nabla_{w_i} \obje(W), \text{ for every } 1 \le i \le m.\]
		\vspace{-0.1in}
	\State Stage 2. Let $\lambda = \lambda_1 = \Theta(\frac{1}{\poly_{\kappa}(d)})$. For $t \le \Theta\left( \frac{d^{1 + 10 \cz} }{\eta}\right)$, update every neuron as follows:
		\[w^{(t+1)}_i = w^{(t)}_i - \eta \cdot \indi{\| w_i^{(t)} \|_2^2 \leq \frac{1}{\lambda_1}} \cdot \nabla_{w_i } \obje(W), \text{ for every } 1 \le i \le m.\]
\end{algorithmic}
	\vspace{-0.05in}
\end{algorithm}

\begin{theorem}[Main result]\label{thm:main}
	Let $\set{Z}$ be a training dataset with $N =  \poly_{\kappa}(d)$ samples generated by the model described above.%
	\footnote{Let $\poly(d)$ denote a polynomial of $d$ and $\poly_\kappa(d)$ denote a polynomial whose degree may depend on $\kappa$.}
	Let $C(\kappa)$ be a sufficiently large constant that only depends on $\kappa$.
	Let $0 < \cz < 1/100$ be a sufficiently small absolute constant that does not depend on $\kappa$.
	Let $\lambda_0$ be a sufficiently small value on the order of $1/\poly(d)$ and $\lambda_1 \le \lambda_0 / O(\poly_{\kappa}(d))$ be a sufficiently small value on the order of $1/\poly_{\kappa}(d)$.
	For a learning rate $\eta < \min\left(\lambda_0^2, O(\frac{1}{\poly_{\kappa}(d)})\right)$, a network width $m \ge \Omega(\poly(d) / \poly(\lambda_1))$, and truncation parameters $\lambda_0, \lambda_1$,
	let $\hat{W}$ be the final network learned by Algorithm \ref{alg}.
	With probability $1 - \frac{1}{\poly(d)}$ over the choice of the random initialization, we have that the population loss of $\hat{W}$ satisfies
	$$\objpop(\hat{W}) \leq O({1} / {d^{1 + \cz}}).$$
\end{theorem}

The intuition behind our main result is as follows.
We build on a connection between the popluation $L(W)$ and tensor decomposition for Gaussian inputs \cite{ge2017learning,ge2018learning}.
By expanding the population loss in the Hermite polynomial basis, the optimization problem becomes an infinite sum of tensor decompositions problems (cf. equation~\eqref{eq:Fjaifosajifa} in Section~\ref{sec_prelim})
To analyze the gradient descent dynamic on the infinite sum tensor decomposition objective, we first analyze the infinite-width case -- when $m$ goes to infinity.
We establish a conditional-symmetry condition on the population of neurons, which greatly simplifies the analysis.
This is established using the fact that our input distribution and labeling function (the absolute value activation) are both symmetric.
Our analysis uncovers a stage-wise convergence of the gradient descent dynamic as follows, which matches our observations in simulations. 
\begin{itemize}
	\item First, Algorithm \ref{alg} minimizes the 0th and 2nd order tensor decompositions. Informally, the distribution of neurons is fitting to the 0th moment and the 2nd moment of $\Set{w_i^{\star}}_{i=1}^d$.
	\item Second, Algorithm \ref{alg} minimizes the 4th and higher order tensor decompositions.
	Initially, there is a long plateau where the evolution is slow, but after a certain point gets faster.
	As a remark, this behavior has been observed for randomly initialized tensor power method \cite{anandkumar2017analyzing}.
	Because the solution to the 4th and higher order orthogonal tensor decomposition problems is unique, we can learn the ground truth weights $\Set{w_i^{\star}}_{i=1}^d$. 
\end{itemize}
Then we show that the sampling error between the infinite-width case and the finite-width case is small.
The finite-width case can be thought of as a finite sample of the infinite-width case.
As the network width increases, the sampling error reduces.
In Section \ref{sec_overview} and \ref{sec_couple}, we will first present a proof overview. The full proof is given in  Section \ref{app_inf} and \ref{app_finite}. 
%

\medskip
As a complement, we show that the generalization error bound of Theorem \ref{thm:main} cannot be achieved by kernel functions with polynomially large feature map.
Hence, by minimizing the higher order tensor decomposition terms, the learned neural network is provably more accurate than kernel functions that simply fit the lower order terms.
Our result is stated as follows.

\tnote{change globally feature mapping to feature map}

\begin{theorem}[Lower bound]\label{thm_lb}
  Under either of the following two situations,
  \begin{itemize}
  	\item[1.] We use a feature map  $\phi(x) : \mathbb{R}^d \to \mathbb{R}^N$ with $N=\poly(d)$
  	\item[2.] We use kernel method with any kernel $K: \mathbb{R}^N  \times \mathbb{R}^N \to \mathbb{R}$ with $N=\poly(d)$ samples.
  \end{itemize}
  There exists a set of orthonormal weights $\{w^*_i\}_{i \in [d]}$ and $ \{a_i\}_{i \in [d]}$ where $a_i \in [\frac{1}{2 d}, \frac{2}{d}]$ for all $1\le i\le d$ satisfying $\sum_{i \in [d]} a_i = 1$,
	such that the following holds:
	With probability at least $0.999$ over the training set $\set{Z}$,
	for any $w_R, w_K \in \mathbb{R}^N$ with $\set{R}(x) :=  w_R^{\top} \phi(x)$ and $\set{K}(x) :=  w_K^{\top} [K(x, x_i)]_{i = 1}^N$, the population loss of the feature map $\phi(x)$ and kernel $\set{K}(x)$, denoted by $L(\set{R}) = \E_{x \sim \cN(0, \id_{d \times d})} (f^{\star}(x) - \set{R}(x))^2 $ and  $L(\set{K}) = \E_{x \sim \cN(0, \id_{d\times d})} (f^{\star}(x) - \set{R}(x))^2$, satisfies 
	\begin{align}
    L(\set{R}) = \Omega\left(\frac 1 d\right) ~~\text{ and }~~ L(\set{K}) = \Omega\left( \frac{1}{d } \right). \label{eq_lb_result}
	\end{align}
\end{theorem}

Comparing the above result with Theorem \ref{thm:main}, we conclude that provided with polynomially many samples, Algorithm \ref{alg} can recover the target two-layer neural network more accurately than the feature map and kernel method described above.
Section \ref{app_lb} shows how to prove Theorem \ref{thm_lb}.
%
\subsection{Related Work}

{\bf Neural tangent kernel (NTK).} A sequence of recent work shows that the learning process of gradient descent on over-parametrized neural networks, under certain initializations, reduces to the learning process of the associated neural tangent kernel. See \citet{jacot2018neural,arora2019finegrained,cao2019generalization,du2018gradient,arora2019exact,al19-rnngen,als18dnn,als18,li2018learning,zou2018stochastic,du2018gradient2,dfs16,ghorbani2019linearized,li2019towards,hanin2019finite,yang2019scaling} and the references therein.
For NTK based results, the learning process of gradient descent can be viewed as solving convex kernel regression.
Our work analyzes a non-convex objective that involves an infinite sum of tensor decomposition problems.
By analyzing the higher order tensor decompositions, we can achieve a smaller generalization error than kernel methods.
\tnote{this is still not right --- NTK also initialize the way. The diff is the parametrization I believe.}\tnote{i thnk the comparison with norm is not great because thee parameterization is different. If we really literally compare the norm without taking into account $1/m$ in the parameterization, NTK is similar to us.}

\citet{AL2019-resnet,al20} show that over-parametrized neural networks can learn certain concept class more efficient than any kernel method.
Their work assumes the target network satisfies a certain ``information gap'' assumption between the first and second layer, while our target network does not require such gaps. \tnote{it's good to extend this sentence above a bit}
\citet{all18,bai2019beyond} go beyond NTK by studying quadratic approximations of neural networks. 
Our work further analyzes higher-order tensor decompositions that are present in the Taylor expansion of the loss objective.


\smallskip
\noindent{\bf Two-layer neural networks given Gaussian inputs.} There is a large body of work on learning two-layer neural networks over the last few years, such as~\citet{kawaguchi2016deep,soudry2016no,xie2016diversity,soltanolkotabi2017theoretical,tian2017analytical,brutzkus2017globally,boob2017theoretical,vempala2018polynomial,oymak2019towards,bakshi2018learning,yehudai2019power,zhang2018learning,li2017provable,li2020can,allen2020feature}.
Our work is particularly related to those that learn a two-layer neural network given Gaussian inputs.
~\citet{li2017convergence,zhong2017recovery} consider learning two-layer networks with ReLU activations with a warm start tensor initialization, as opposed to from a random initialization.
\citet{du2017gradient} consider learning a target function consisting of a single ReLU activation.
~\citet{brutzkus2017globally,tian2017analytical} study the case where the weight vector for each neuron has disjoint support.
Apart from the gradient descent algorithm, the method of moments has also been shown to be an effective strategy with provable guarantees~(e.g. \cite{bakshi2018learning,ge2018learning}).

The closest work to ours is \citet{ge2017learning} that consider a similar concept class.
However, their work requires designing a complicated loss function, which is different from the mean squared loss.
The learner network also uses a low-degree activation function as opposed to the ReLU activation.
These are introduced to address the challenge of analyzing non-convex optimization for tensor decomposition with multiple components as variables, because prior works mostly focus on the non-convex formulation that optimizes over a single component (e.g., see~\cite{ge2017optimization}). \citet{ge2017learning} have stated the question of analyzing the gradient descent dynamic for minimizing the sum of second and fourth order tensor decompositions as a challenging open question. 
Our analysis not only applies to this setting, but also allows for more even order tensor decompositions.
Apart from ReLU activations, quadratic activations have been studied in \citet{li2017algorithmic,oymak2019towards,soltanolkotabi2017theoretical}. 

\smallskip
\noindent{\bf Infinite-width neural networks.}
Previous work such as \citet{mmn18,cb18} show that as the hidden layer width goes to infinity, gradient descent approaches the Wasserstein gradient flow.
\citet{mmn18} use tools from partial differential equations to prove the global convergence of the gradient descent.
Both of these results do not provide explicit convergence rates. 
\citet{wei2018margin} show that under a certain regularity assumption on the activation function, the Wasserstein gradient flow converges in polynomial iterations for infinite-width neural networks..

\medskip
\noindent\textbf{Organizations.}
The rest of the paper is organized as follows.
In Section \ref{sec_prelim}, we reduce our setting to learning a sum of tensor decomposition problems.
In Section \ref{sec_overview}, we describe an overview of the analysis for the infinite-width case.
In Section \ref{sec_couple}, we show how to connect the above case to the gradient descent dynamic on the empirical loss for polynomially-wide networks.
Finally we validate our theoretical insight on simulations in Section \ref{sec_simulate}.
In Section \ref{app_inf}, we provide the proof of the infinite-width case.
In Section \ref{app_finite}, we provide an error analysis of the infinite-width case and complete the proof of Theorem \ref{thm:main}.
In Section \ref{app_lb}, we present the proof of Theorem \ref{thm_lb}.

\section{Preliminaries}\label{sec_prelim}

Recall that the ground-truth weights $\Set{w_i^*}_{i=1}^d$ forms an orthonormal basis.
Since the input distribution $x \sim \cN(0, I_{d \times d})$ and the initialization $\Set{w_i}_{i=1}^m \sim \cN\left(0, \frac 1 {d}\cdot\id_{d\times d}\right)$ are both rotation invariant, without loss of generality we can assume that $w_i^* = e_i$, for all $1\le i\le d$.

We can average out the randomness in $x$ by applying Theorem 2.1 of \citet{ge2017learning} on the loss function $\objpop(W)$, by expanding the activations function in the Hermite basis \cite{o2014analysis}.
	{\begin{align} \label{eq:Fjaifosajifa}
    \objpop(W)=& c_0\bignormFro{\frac 1 m \sum_{i=1}^m \norm{w_i}^2 - \sum_{i=1}^d a_i}^2     +  c_1\bignormFro{\frac 1 m \sum_{i=1}^m \norm{w_i} w_i }^2   +  {c}_2 \bignormFro{\frac 1 m\sum_{i=1}^m w_i^{\otimes 2} - \sum_{i=1}^d a_i e_i e_i^{\top}}^2 \nonumber \\
        &+ \sum_{j\ge 2} {c}_{2j} \bignormFro{\frac 1 m\sum_{i=1}^m {w_i}^{\otimes 2} \otimes {{}\bar{w}_i}^{\otimes (2j-2)} - \sum_{i=1}^d a_i e_i^{\otimes 2j}}^2,
  \end{align}}%
where $c_k = \frac {2 [(k-3)!!]^2} {\pi \cdot k!}$ is the Hermite coefficients of the absolute value function for any $k \ge 0$.
We remark that the population loss is a infinite sum of orthogonal tensor decomposition problems!
For example, the $0$-th order tensor decomposition concerns the $l_2$-norm of the weights.
More generally, the $k$-order tensor decomposition concerns the $k$-th moment of the weights.

\medskip
\noindent{\bf The distribution of neuron weights. }
We begin by considering an infinite-width neural network and then extend the proof to finite-width neural networks.
Following \citet{wei2018margin}, %
an infinite-width neural network specifies a distribution of neuron weights. Let $\cP$ denote a distribution over $\mathbb{R}^d$.
A learner network (cf. equation \eqref{eq_labeling}) using $\cP$ as its neuron weights gives the output for an input $x\in\real^d$ from the Gaussian distribution:
\begin{align}
f_{\cP}(x) = \exarg{w\sim \cP}{\|w\|_2 \cdot \relu\left(w^\top x\right)}.
\end{align}
Correspondingly, the population loss of $f_{\cP}(x)$ is given as
\vspace{-0.0in}
{\begin{align}\label{eq:fajsoifsajif}
	\objinf(\cP)=& c_0\bignormFro{\E_{w \sim \cP} \norm{w}^2 - \sum_i^d a_i}^2 + c_1\bignormFro{\E_{w \sim \cP} w \|w\|_2}^2
	+ {c}_2 \bignormFro{ \E_{w \sim \cP } w^{\otimes 2} - \sum_{i=1}^d a_i e_i e_i^{\top}}^2 \nonumber\\
	&+ \sum_{j\ge 2} {c}_{2j} \bignormFro{ \E_{w \sim \cP }{w}^{\otimes 2} \otimes {{}\bar{w}}^{\otimes (2j-2)} - \sum_{i=1}^d a_i e_i^{\otimes 2j}}^2.
\end{align}}%

\medskip
\noindent{\bf Gradient descent update.} It has been shown in prior works that gradient descent in the (natural) parameter space corresponds to Wasserstein gradient descent in the distributional space.
However, we found that the Wasserstein gradient perspective is not particularly helpful for us to analyze our algorithms and therefore we work with the update in the parameter space.
The distribution $\cP$ can be viewed as a collection of infinitesimal neurons. %
The gradient of each neuron $v$ is given by computing the gradient of the objective $L(W)$ w.r.t a particle $v$ assuming the rest of the particles follow the distribution $\cP$.
Let $\nabla_v\objinf(\cP)$ denote the gradient of $v$. We have that
{\begin{align}
&	\nabla_v \objinf(\cP)  \define  b_0 \left( \E_{w\sim\cP} \| w\|_2^2 - 1  \right) v + b_1 \left( \E_{w\sim\cP} \| w\|_2 w  \| v\|_2 + \|w\|_2 \langle w, v \rangle \bar{v}  \right)
		\\&+ b_2 \left( \E_{w\sim\cP} \langle w, v \rangle w -  \sum_{i=1}^d a_i \langle e_i, v \rangle e_i\right) \nonumber
	 + \sum_{j \geq 2} b_{2j}  \left( \E_{w\sim\cP} \langle w, v \rangle \langle \bar{w} , \bar{v} \rangle^{2j - 2} w -   \sum_{i=1}^d a_i \langle e_i, v \rangle \langle e_i, \bar{v} \rangle^{2j - 2} e_i \right) \nonumber
	\\
	&+ \sum_{j \geq 2} b_{2j}'\ \Pi_{v^{\bot}}\left( \E_{w\sim\cP} \langle w, v \rangle \langle \bar{w} , \bar{v} \rangle^{2j - 2} w -   \sum_{i=1}^d a_i\langle e_i, v \rangle \langle e_i, \bar{v} \rangle^{2j - 2} e_i \right),\label{eqn:3}
\end{align}}%
where $b_0 = 4c_0, b_1 =  2c_1 $, and for any $j \ge 2$, $b_{2j } =  (4j) \times c_{2 j } = \Theta \left( \frac{1}{j^2} \right)$ and $b_{2j'} = (4j - 4)\times c_{2 j }$.
We use $\nabla_v \objinf$ and $\nabla_v$ as a shorthand for $\nabla_{v}L_{\infty}(\cP)$.
Based on equation \eqref{eqn:3}, we can further decompose $\nabla_v L_{\infty}(\cP)$ into the sum of $\nabla_{2j, v} L_\infty(\cP)$ for $j \ge 0$, where the $2j$-th gradient refers to the gradient of the $2j$-th tensor decomposition.
As a result, given a neural network with neuron distribution $\cP^{(t)}$, the neuron distribution after a truncated gradient descent step, denoted by  $\cP^{(t+1)}$, satisfies that
\begin{align}\label{eq:update_p}
	v^{(t + 1)} \sim \cP^{(t+1)}  \Leftrightarrow
	v^{(t+1)} := v^{(t)} - \eta \indi{\|v^{(t)}\|_2^2 \leq \frac{1}{2 \lambda} } \nabla_{v^{(t)}}\objinf(\cP^{(t)}), \text{ for } v^{(t)}\sim \cP^{(t)}.
\end{align}

\medskip
\noindent {\bf Finite-width case}.
We briefly describe the connection between the above infinite-width case and the finite-width case.
Intuitively, we can think of the finite-width case as sampling $m$ neurons randomly from the neuron population $\cP$ in the infinite-width case.
There are two sources of sampling error that arise from the above process: (i) the error of the gradients between the finite neuron distribution and the infinite neuron distribution;
(ii) the error between the empirical loss and the population loss.
Because of gradient truncation,  the norm of every neuron is bounded by $1/\lambda$.
Therefore, the sampling error reduces as $m$ and $N$ increases, as shown in the following claim.
\begin{claim}\label{claim:911}
For every $\lambda > 0$, for every distribution $\cP$ over $\mathbb{R}^d$ supported on the ball $\{ w \in \mathbb{R}^d \mid \| w\|_2^2 \leq \frac{1}{\lambda} \}$, let $W = \{w_i\}_{i = 1}^m$ be i.i.d. random samples from $\cP$.
For any sufficiently small $\delta > 0$, with probability at least $1 - \delta$ over the randomness of $W$, we have that:
\begin{align*}
	\left| \obj\left(W \right)- \objinf (\cP)\right| \leq  \frac{\poly\left(  \frac{1}{\lambda}\right) \log \frac{1}{\delta}}{\sqrt{m}}.
\end{align*}
With probability at least $1 - \delta$ over the randomness of $\{w_i\}_{i = 1}^m$ and the training dataset $\set{Z}$, for every $w \in W$, we have that:
	\begin{align*}
		\left\| \nabla_{w} \obje(W) -  \nabla_{w} L_{\infty}(\cP) \right\|_2 \leq \poly\left(\frac{1}{\lambda}\right) \log \frac{m}{\delta} \left( \frac{1}{\sqrt{m}} + \frac{1}{\sqrt{N}} \right).
	\end{align*}
\end{claim}%
Claim \ref{claim:911} can be proved by standard concentration inequalities such as the Chernoff bound.

\medskip
\noindent{\bf Notations.}
Let $a = b \pm c$ denote a number within $[b - |c|, b+ |c|]$.
Let $[d]$ denote the set including $1,2,\dots,d$.
Let $\id_{d\times d} \in\real^{d\times d}$ denote the identity matrix in dimension $d$.
For two matrices $A, B$ with the same dimensions, we use $\inner{A}{B} = \tr[A^{\top}B]$ to denote their inner product.
For a vector $w \in\real^d$, let $\| w\|_2 $ denote its $\ell_2$ norm and $\| w\|_{\infty}$ denote its $\ell_{\infty}$ norm.
For $i \in [d]$, let $w_i$ denote the $i$-th coordinate of $w$ and $w_{-i}$ denote the vector which zeroes out the $i$-th coordinate of $w$.
We define $\bar{w} = \frac{w}{\| w\|_2}$ to be the normalized vector, and $\Pi_{w^{\bot}} = (\id - \bar{w} \bar{w}^{\top})$ to be the projection onto the orthogonal complement of $w$.
For a matrix $M$, let $\| M \|_2$ denote the spectral norm of a matrix $M$.

\section{Overview of the Infinite-Width Case}
\label{sec_overview}

We begin by studying Algorithm \ref{alg} for minimizing the population loss using an infinite-width neural network.
The infinite-width case plays a central role in our analysis.
First, the infinite-width case allows us to simplify the gradient update rule through a conditional-symmetry condition that we describe below.
Second, the finite-width case can be reduced to the infinite-width case by bounding the sampling error of the two cases --- we describe the reduction in the next section.

A natural starting point for the infinite-width case is to simply set the network width $m$ to infinity in Theorem \ref{thm:main}.
However, this will include negligible outliers such as those with large norms in the Gaussian distribution.
Therefore, we focus on a truncated probability measure $\cP^{(0)}$ of $\cN(0, \id_{d\times d})$ by enforcing a certain bounded condition.
The precise definition of $\cP^{(0)}$ is presented in Definition \ref{defn:true} of Appendix \ref{app_inf}.
For the purpose of providing an overview of the analysis, it suffices to think of $\cP^{(0)}$ as a Gaussian-like distribution that satisfies the following property.

\begin{definition}[Conditional-symmetry]\label{def_cs_prop}
We call a distribution $\cP$ over $\mathbb{R}^d$ conditionally-symmetric if for every $i \in [d]$ and every $v \in \mathbb{R}^d$, the following is true.
\begin{align}
	\Pr_{w \sim \cP}[w_i = v_i \mid w_{-j} = v_{-j}] = \Pr_{w \sim \cP}[w_i = -v_i \mid w_{-j} = v_{-j}].
\end{align}
\end{definition}
Provided with $\cP^{(0)}$ as initialization, we are ready to state the main result of the infinite-width case as follows.
\begin{theorem}[Infinite-width case]\label{thm_inf}
	In the setting of Theorem \ref{thm:main},
	let the number of samples $N$ go to infinity.
	Starting from the initialization $W^{(0)}$ as the neuron distribution $\cP^{(0)}$, let $\hat{W}$ be the final output network by Algorithm \ref{alg}.
	The population loss of $\hat{W}$ satisfies $L(\hat{W}) \le O(1 / d^{1 + Q})$.
\end{theorem}

In the rest of this section, we present an overview of the proof of Theorem \ref{thm_inf} and provide pointers to the proof details to be found in Section \ref{app_inf}.
First, we provide a simplifying formula for the gradient of $L_{\infty}(\cP)$.
We describe an overview of the two stages of Algorithm \ref{alg} in Section \ref{sec_stage1} and \ref{sec_stage2}, respectively.

\medskip

First, we show how to simplify the gradient of $L_{\infty}(\cP)$ (cf. equation \eqref{eqn:3}).
Recall from Section \ref{sec_prelim} that we can view the weights $W^{(t)}$ in the $t$-th iteration as a distribution $\cP^{(t)}$ over $\real^d$.
Our main observation is that when $\cP^{(t)}$ is {conditionally-symmetric},  $\cP^{(t+1)}$ is also conditionally-symmetric.

\begin{claim}\label{cl_ss}
	Suppose the update rule of $\cP^{(t)}$ is given in equation~\eqref{eq:update_p}. If $\cP^{(t)}$ is {\cs}, then $\cP^{(t + 1)}$ is also {\cs}.
\end{claim}
To see that Claim \ref{cl_ss} is true, we first observe that the 1st order tensor decomposition is always zero when $\cP^{(t)}$ is conditionally symmetric.
For the even order tensor decompositions, we observe that for every neuron $v$ in $\cP^{(t)}$ and every $1\le j\le d$, subject to $v_{-j}$ being fixed, $\nabla_v L_{\infty}(\cP)$ is a polynomial of $v_j$ that only involves odd degree monomials.
Therefore, as long as $\cP^{(t)}$ is {\cs}, then $\cP^{(t + 1)}$ is still {\cs}.
Since $\cP^{(0)}$ is conditionally-symmetric by definition, we conclude that the neuron distribution is conditionally-symmetric throughout Algorithm \ref{alg}.
Based on this claim, we simplify equation \eqref{eqn:3} as follows.

\begin{claim} \label{claim:symmetry}
	Suppose that $\cP = \cP^{(t)}$ is \cs.
	For any $j \ge 0$, let $\nabla_{2j, v}$ be a shorthand for the gradient of the 2j-th tensor $\nabla_{2j, v} L_{\infty}(\cP)$.
	For any $1\le i\le d$, let $[\nabla_{2j , v}]_i$ be the $i$-th coordinate of $\nabla_{2j, v}$.
	We have that $[\nabla_{2j, v}]_i$ is equal to the following for each value of $j$:
	{\begin{align}
		[\nabla_{0, v}]_i &=  b_0 \left( \E_{w\sim \cP} \| w\|_2^2 - 1  \right) \inner{e_i}{v}, \quad [\nabla_{2, v}]_i =  b_2 \left( \E_{w\sim \cP} w_i^2  -  a_i \right)  \inner{e_i}{v},  \label{eq:def02} \\
		[\nabla_{2j , v}]_i &= \left( b_{2j} + b_{2j}'\right) \left( \exarg{w\sim \cP}{\sum_{i_1, \cdots, i_{j }} \bigbrace{\prod_{r \in [j - 1]} (\bar{w}_{i_r} \bar{v}_{i_r})^2} ({w}_{i_j})^2 \inner{e_{i_j}}{v} } -   a_i  \langle e_i, \bar{v} \rangle^{2j - 2} \langle e_i, v \rangle \right) \nonumber \\
& ~~~~ - b_{2j}'\left(  \exarg{w\sim \cP}{\|w\|_2^2 \sum_{i_1, \cdots, i_{j}} \prod_{r \in [j ]} (\bar{w}_{i_r} \bar{v}_{i_r})^2 }  -   \sum_{r=1}^d a_r \langle e_r, \bar{v} \rangle^{2j }  \right) v_i, \forall\, j \ge 2. \label{eq:def04}
	\end{align}}
\end{claim}
The proof of Claim \ref{claim:symmetry} is by applying Claim \ref{cl_ss} to equation \eqref{eqn:3}, which zeroes out the coordinates in $w$ that has an odd order before taking the expectation of $w$ in $\cP$.
For the 2nd order gradient $[\nabla_{2, v}]_i$, we have that
\[ [\nabla_{2, v}]_i = b_2 \bigbrace{\E_{w\sim \cP} \inner{w}{v} w_i - a_i \inner{e_i}{v} e_i} = b_2 \bigbrace{\E_{w\sim\cP} w_i^2 - a_i} v_i. \]
Similar arguments apply to the gradient of higher order tensor decompositions.
Claim \ref{cl_ss} and \ref{claim:symmetry} together implies that for the infinite-width case, the gradient descent update is given by equation \eqref{eq:def02} and \eqref{eq:def04}.

\subsection{Dynamic during Stage 1}
\label{sec_stage1}

\paragraph{Stage 1.1: learning 0th and 2nd order tensors.}
We show that Algorithm \ref{alg} minimizes the 0th and 2nd order tensor decompositions of the objective $\objinf$ to zero first.

First, we show that the gradient of the 4th and higher order tensor decompositions is dominated by $\nabla_{0, v}$ and $\nabla_{2, v}$.
We observe that for $v \sim \cP^{(0)}$, the $i$-th coordinate of $\nabla_{0, v}$ and $\nabla_{2, v}$ satisfies that
\begin{align}
	|[\nabla_{0, v}]_i| +  |[\nabla_{2, v}]_i| = \Theta\left( \frac{1}{d^{1.5}} \right). \label{eq_02_stage1}
\end{align}
This is because $\cP^{(0)}$ is a suitable truncation of $\cN(0, \id_{d\times d} / d)$.
We further have that
	\[ \|v \|_{\infty}^2, \|\bar{v} \|_{\infty}^2 = \tilde{\Theta}\left( \frac{1}{d} \right) \text{ and } \left| \exarg{w \sim \cP^{(0)}}{w_i^2} - a_i \right| \le \bigoarg{\kappa}{\frac{1}{d}}.\footnote{We use the notation $x \le O_{\kappa}(y)$ to denote that $x \le h(\kappa) \cdot y$ for a fixed value $h(\kappa)$ that only depends on $\kappa$.}  \]
Applying the above to equation \eqref{eq:def02}, we obtain equation \eqref{eq_02_stage1}.
For higher order tensors, in Proposition \ref{lem:four_plus_upper}, we show that
 for any $j \geq 2$, $|[\nabla_{2j, v}]_i| = \tilde{O}\left( {1}/{d^{2.5}} \right)$.
Therefore, the 0th and 2nd order gradients indeed dominate the higher order gradients and Algorithm \ref{alg} is simply minimizing the 0th and 2nd order tensor decompositions of $\objinf$.

Based on the above observation, we show that the 0th and 2nd order tensor decompositions converge to zero in Lemma \ref{lem:zero_two3}. The main intuition is as follows.
By equation~\eqref{eq:def02}, both the 0th and 2nd order gradient only depend on the $i$-th coordinate of neurons in $\cP$.
Hence, the update can be viewed as $d$ independent updates over the $d$ coordinates.
In Proposition~\ref{lem:zero_two}, we show that throughout Algorithm \ref{alg}, the 0th order tensor decomposition loss given by $|\E_{\cP} \| w\|_2^2 - 1|$ is smaller than the 2nd order tensor decomposition loss given by $\max_{i \in [d]}|  \E_{\cP} w_i^2  -  a_i|$.
Thus, it suffices to show that the 2nd order loss $\max_{i\in[d]} \Abs{\E_{\cP}w_i^2 - a_i}$ converges to zero.
This problem reduces to principal component analysis and in Proposition \ref{lem:zero_two2}, we show that the 2nd order loss indeed converges by a rate of $O_{\kappa}(1/d)$ using standard techniques.

As shown in Lemma~\ref{lem:zero_two3}, Stage 1.1 finishes within $\tilde{O}_{\kappa}\left( {d}/{\eta} \right)$ iterations, when eventually $|[\nabla_{0, v}]_i| +  |[\nabla_{2, v}]_i| $ becomes $\tilde{O}( {1}/{d^{2.5}} )$ for all $1\le i\le d$, which is the same order as $ |[\nabla_{2j, v}]_i|$ for $j \geq 2$.
Thus, Algorithm \ref{alg} enters the next substage where the gradient of the higher order tensor decompositions becomes effective.

\paragraph{Stage 1.2: learning higher order tensor decompositions.}
After the 0th and 2nd order tensor decompositions are minimized to a small enough value, the gradient of higher order tensor decompositions begins to dominate the update.
In Lemma~\ref{lem:stage_1_final}, we show that for a small fraction of neurons, their norms become much larger than an average neuron --- a phenomenon that we term as ``winning the lottery ticket''.
The main intuition is as follows.

In Proposition \ref{lem:four_plus_interval}, we show that the gradient of most neurons $v$ except a small fraction can be approximated by a signal term from the 4th order gradient plus an $O(1/d^2)$ error term:
\begin{align}
	\left| [\nabla_v]_i \right| = \left( b_{4} + b_{4}'\right)  a_i \langle e_i, v \rangle \langle e_i, \bar{v} \rangle^{2}  \pm \frac{C_t(\kappa) \log d}{d^2} |v_i|, \label{eq_grad_simplify}
\end{align}
where $C_t(\kappa)$ is a function that only depends on $\kappa$ but grows slowly with $t$.
To see that equation \eqref{eq_grad_simplify} is true, except a small set of neurons with probability mass at most $1/d^{\alpha}$ where $\alpha$ will be specified later, any other neuron $w$ satisfies $\| w\|_{\infty}^2 \leq  {\alpha \log d}/{d}$.
For the small set of neurons, since we stop updating a neuron when its norm grows larger than $1 / \lambda_0$, the norm of any of these neurons is less than ${1}/{\lambda_0}$.
Thus, provided with a sufficiently large $\alpha$, the contribution of these neurons to the gradient is negligible.
Combined together, we prove equation \eqref{eq_grad_simplify} in Proposition \ref{lem:four_plus_interval}.

Next, we reduce the dynamic to tensor power method.
Based on equation \eqref{eq_grad_simplify}, we observe that the update of $v_i$ is approximately
$v_i^{(t + 1)} \approx v_i^{(t)} + \eta \cdot a_i \cdot (v_i^{(t)})^3$,
which is analogous to performing power method over a fourth order tensor decomposition problem.
Hence, for larger initializations of $v_i$, $v_i$ also grows faster.
Based on the intuition, we introduce the set of ``basis-like'' neurons $\set{S}_{i, good}$ in the population $\cP$, which are defined more precisely in Lemma \ref{lem:stage_1_final}.
Intuitively, $\set{S}_{i, good}$ includes any neuron $v$ that satisfies $[v_i^{(0)}]^2 \geq {C^2 \log d}/{d}$, which has probability measure at least ${1}/{d^{C^2}}$ by standard anti-concentration inequalities.
Following equation \eqref{eq_grad_simplify}, we show that the neurons in $\set{S}_{i, good}$ keeps growing until they become roughly equal to $e_i / (\lambda_0 \poly(d))$.

As shown in Lemma \ref{lem:stage_1_final}, Algorithm \ref{alg} goes through a long plateau %
of $O_{\kappa}({d^2}/ ({\eta} \poly\log(d))$ iterations, until the neurons of $\set{S}_{i, good}$ are sufficiently large.
Intuitively, the scaling of $d^2$ in the number of iterations arises from the $1/d^2$ increment in equation \eqref{eq_grad_simplify}.
This concludes Stage 1. The update of these basis-like neurons will be the focus of Stage 2.

\subsection{Dynamic during Stage 2}\label{sec_stage2}
In the second stage, we reduce the gradient truncation parameter in Algorithm \ref{alg} from $\lambda_0 = \Theta(1/\poly(d))$ to a smaller value $\lambda_1 = \Theta(1 /\poly_{\kappa}(d))$.
This allows the neurons that are close to basis vectors to fit the target network more accurately.

\paragraph{Stage 2.1: obtaining a warm start initialization.}
In Lemma \ref{lem:final_222}, we show that after $\Theta(d\log d/\eta)$ iterations, the population loss reduces to less than $o(1/(d \log^{0.01} d))$.
The proof of Lemma \ref{lem:final_222} involves analyzing the 0th and 2nd order tensor decompositions, similar to Stage 1.1.

At the end of Stage 2.1, the weights of the learner neural network form a ``warm start'' initialization, meaning that its population loss is less than $o(1/d)$ \cite{li2017convergence,zhong2017recovery}.
The final substage will show that the population loss can be further reduced from $o(1/(d\log^{0.01} d))$ to $O(1/d^{1 + Q})$, where $Q$ is a fixed constant defined in Theorem \ref{thm:main}.

\paragraph{Stage 2.2: the final substage.}
In Lemma \ref{lem:final_333}, we show that the population loss further reduces to $O({1}/d^{1 + Q})$ after $\Theta(d^{1 + 10Q} / \eta)$ iterations.
We describe an informal argument by contrasting the gradient update of neurons in $\set{S}_{i,good}$ and the rest of the neurons for a particular coordinate $i\in[d]$.

For any neuron $v \in \set{S}_{i, good}$, in Claim \ref{claim:grad_21}, we show that the $i$-th coordinate of $v$ approximately follows the following update (cf. equation \eqref{eq:cor_upppp}):
\begin{align}
	[\nabla_{v}]_i  \approx b_0 \left( \E_{\cP} \| w\|_2^2 - 1  \right) v_i+  b_2 \left( \E_{\cP} w_i^2  -  a_i \right)  v_i  -  \eta \frac{c_t\cdot C(\kappa)}{d^2} v_i, \label{eq_grad_update_22}
\end{align}
where $c_t$ is a function that grows  with $t$ but bounded above by $O(d^{Q})$ and $C(\kappa)$ is a function that only depends on $\kappa$.
For any neuron $v \notin \set{S}_{i,good}$, in Claim~\ref{claim:grad_21}, we show that $v_i$ follows a similar update but its corresponding value of $c_t$ is much smaller than that of neurons in $\set{S}_{i, good}$.
Thus, basis-like neurons grow faster than the rest of neurons by an additive factor that scales with $c_t / d^2$.

Based on the intuition, we analyze the dynamic following equation \eqref{eq_grad_update_22} using standard techniques for analyzing the convergence of gradient descent.
In Lemma \ref{lem:final_333}, we show for after $O(d^{1 + Q} / \eta)$ iterations, the  0th order tensor decomposition loss given by $b_0 \left( \E_{\cP} \| w\|_2^2 - 1  \right)$ and the 2nd order tensor decomposition loss given by $b_2 \left( \E_{\cP} w_i^2  -  a_i \right)$ both become less than $O(d^{1 + Q})$. %

Once Lemma \ref{lem:final_333} is finished, Algorithm \ref{alg} has learned an accurate approximation of $f^{\star}(\cdot)$ and we can conclude the proof of Theorem \ref{thm_inf}.
We show that the population loss has also become less than $O(d^{1 + Q})$ (cf. equation \eqref{claim:addition_1}).
Thus, we have finished the analysis of Algorithm \ref{alg} for $L_{\infty}(\cP)$.
We provide the proof details of Theorem \ref{thm_inf} in Section \ref{app_inf}.

\section{Overview of the Finite-Width Case}\label{sec_couple}

Based on the analysis of the infinite-width case, we reduce the finite-width case to the infinite-width case.
By applying Claim \ref{claim:911} with $\cP = \cP^{(t)}$, when $\{w_i^{(t)}\}_{i=1}^m$ are i.i.d. samples from $\cP^{(t)}$, the empirical loss and its gradient are tightly concentrated around the population loss and its  gradient.
Furthermore, as we increase the number of neurons $m$ and the number of samples $N$, the sampling error reduces.
Therefore, the goal of our reduction is to show that the sampling error remains small throughout the iterations of Algorithm \ref{alg}.
We describe our reduction informally and leave the details to Section \ref{app_finite}.

The connection between the dynamic of the finite-width case and the infinite-width case is as follows.
For a neuron $w^{(t)}$ sampled from $\cP^{(t)}$, we have analyzed the dynamic of $w^{(t)}$ in the infinite-width case starting from $w^{(0)}$.
For the finite-width case, let $\tw^{(t)}$ denote the $t$-th iterate starting from the same initialization $w^{(0)}$ using Algorithm \ref{alg}.
Our goal is to show that $\xi_w^{(t)} \define \tw^{(t)} - w^{(t)}$ does not become exponentially large before Algorithm \ref{alg} finishes.

Based on the above connection, we show that the propagation of the error $\xi_w^{(t)}$ remains polynomially small throughout Stage 1 in Lemma \ref{lem:error_final_stage_1}.
Our analysis involves a bound on the average error of all neurons $\E_{w\sim W}[\| \xi_w^{(t + 1) } \|_2^2]$ and a bound on the individual error of every neuron $\max_{w\in W} \|\xi^{(t + 1)} \|_2^2$.
First, in Proposition~\ref{prop:err_higher}, we show that it suffices to consider the first order errors in $\xi_w^{(t+1)}$, i.e. those that involve at most one of $\xi_w^{(t)}$. %
Based on this result, in Proposition~\ref{prop:error_total_stage_1} and \ref{prop:err_ind_stage_1}, we show that the average error and the individual error satisfy that:
\begin{align*}
	\E_{w\sim W}[\| \xi_w^{(t + 1) } \|_2^2] &\leq (1 \pm o(1))\left( 1 + \frac{\eta}{\poly(d)} \right) \E[\| \xi_w^{(t ) } \|_2^2], \\
	\max_{w \in W}\|\xi_w^{(t + 1)} \|_2^2 &\leq \poly(d) \E_{w\sim W} \|\xi_w^{(t+1)} \|_2^2.
\end{align*}
Combined together, we show in Lemma \ref{lem:error_final_stage_1} that $\xi_w^{(t)}$ indeed remains polynomially small.
For Stage 2, we analyze the propagation of $\xi_w^{(t)}$ in Lemma \ref{lem:error_final_stage_2} and \ref{lem:error_final_stage_22} using similar arguments.

Combining the above three lemmas on error propagation and Theorem \ref{thm_inf},
we complete the proof of Theorem \ref{thm:main} in Section \ref{app_finite}.

\section{Simulations}\label{sec_simulate}

We provide simulations to complement our theoretical result.
We consider a setting where $w_i^{\star} = e_i$ and $a_i = 1/d$, for $1 \le i \le d$.
The input is drawn from the Gaussian distribution.
For the $i$th order tensor, we measure the corresponding tensor decomposition loss from the population loss $L(W)$.

\paragraph{Stage-wise convergence.}
We validate the insight of our analysis, which shows that the convergence of gradient descent has several stages.
We use the labeling function of equation \eqref{eq_abs} and a learner network with absolute value activation functions as in Section \ref{sec_overview} and Section \ref{app_inf}.
First, the 0th and 2nd order tensor decomposition losses converge to zero quickly.
Second, the 4th and higher order tensor decomposition losses converge to zero followed by a long plateau.
Figure \ref{fig_tensor} shows the result.
Here we use $d = 30$ and $m = 100 > 2d$. The number of samples is $10^4$.

We can see that initially, the 0th and 2nd order tensor decompositions have higher loss than the 4th and higher order tensor decompositions.
Then, both the 0th and the 2nd order losses decrease significantly from the initial value and converge to below $10^{-1}$ very quickly.
Moreover, after a quick warm up period, the 0th order loss always stays smaller than the 2nd order loss, as our theory predicts.
This is followed by a long plateau, which corresponds to Stage 1.2 of our analysis.
During this stage, the 4th and higher order losses dominate dynamic, where a small fraction of neurons converge to basis-like neurons.
Eventually, the learner neural network accumulates enough basis-like neurons from the 4th and higher tensors in the network.
The 4th and higher order losses become less than $10^{-2}$.
The 0th and 2nd order losses further reduce to closer to zero.
Our theory provides an in-depth explanation of these phenomena.

\begin{figure}[!t]
	\centering
	\begin{minipage}{.45\linewidth}
		\centering
		\includegraphics[width=0.95\textwidth]{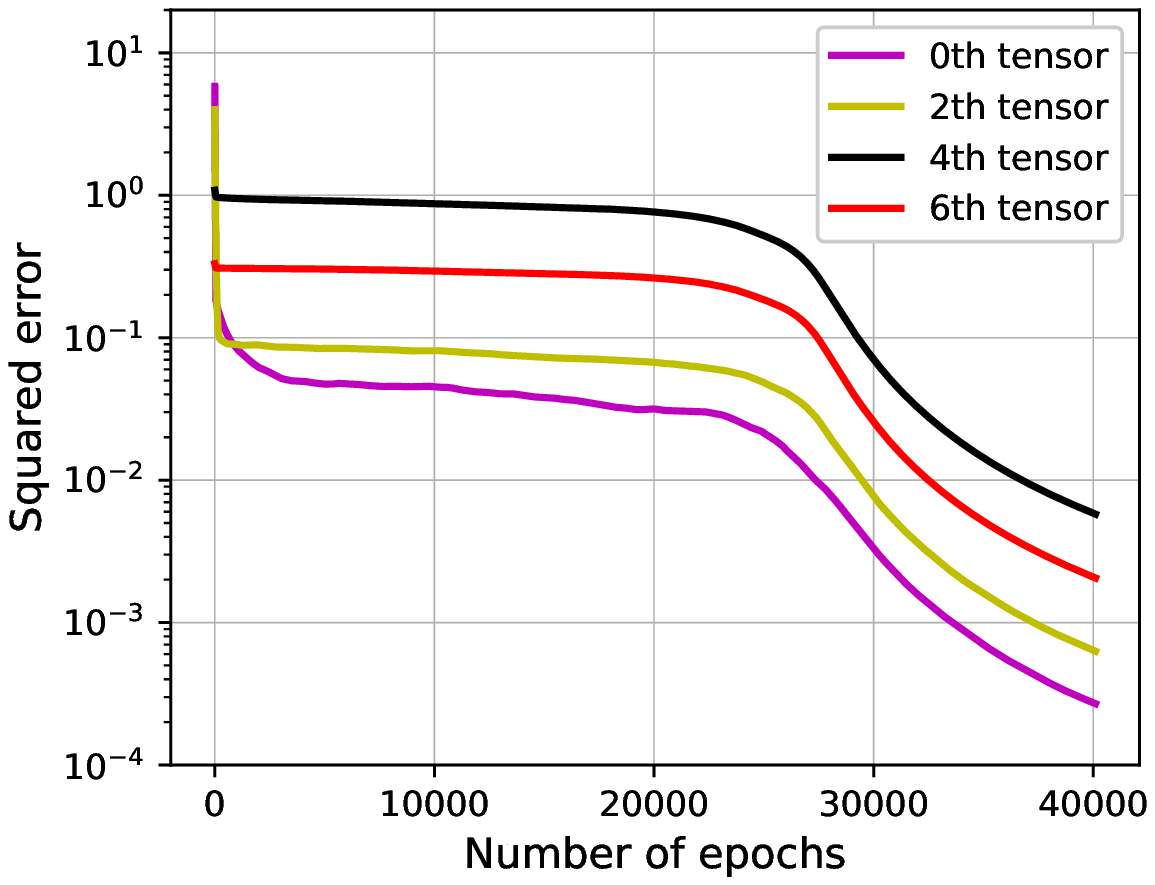}
		\vspace{-0.1in}
		\caption{Illustrating the convergence of each tensor during the gradient descent dynamic using absolute value activations.}
		\label{fig_tensor}
	\end{minipage}\hfill
	\begin{minipage}{.45\linewidth}
		\centering
		\includegraphics[width=0.95\textwidth]{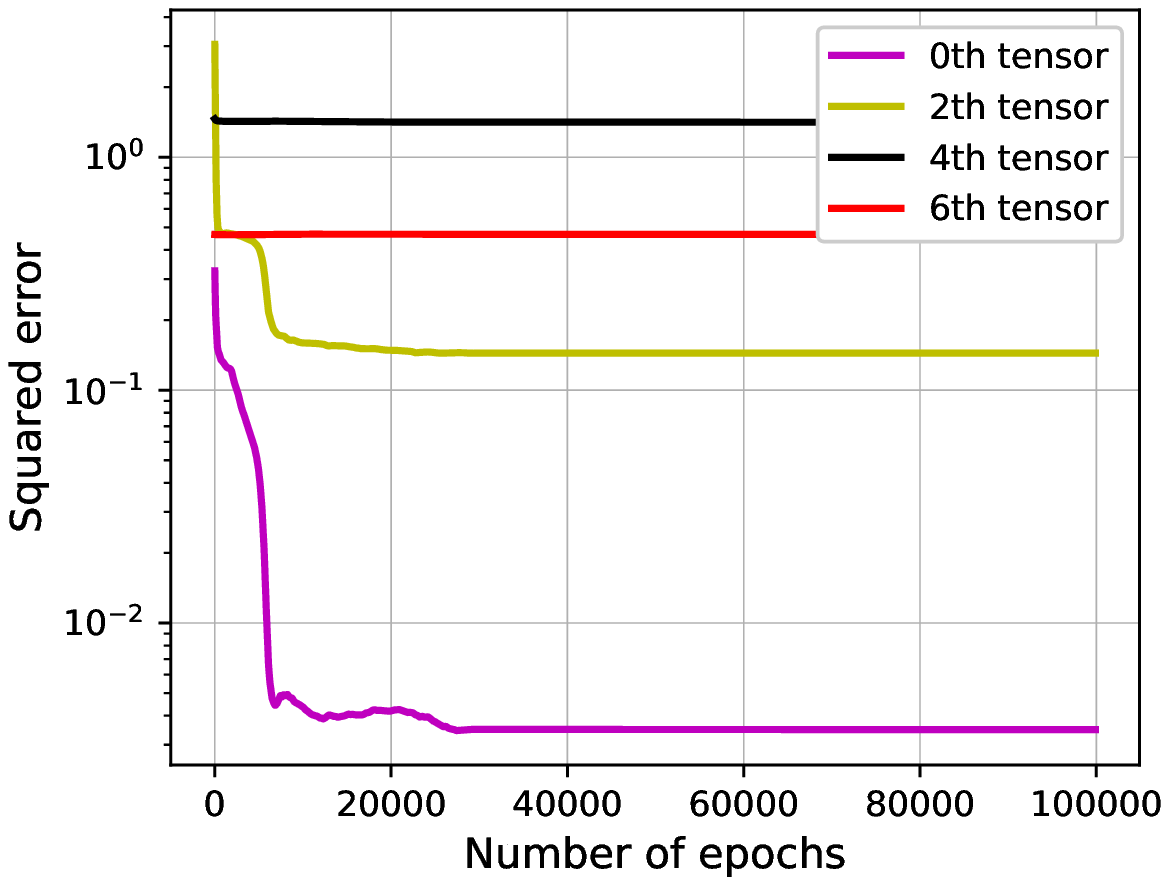}
		\vspace{-0.1in}
		\caption{For properly parametrized gradient descent, the 4th and 6th order tensors get stuck using absolute value activations.}
		\label{fig_stuck}
	\end{minipage}
	\vspace{-0.15in}
\end{figure}

\paragraph{Over-parametrization is necessary.}
It has been observed that for properly parametrized gradient descent, gradient descent can get stuck starting from a random initialization \cite{ge2017learning,dltps18}.
We show that this is because the higher order losses remain large even though the 0th order loss has become small.
We consider the same setting as the previous experiment but use \textcolor{black}{$m = 2d$}.
Figure \ref{fig_stuck} shows the result.
We can see that the 0th order loss still reduces to less than $10^{-2}$.
However, the 2nd, 4th and 6th order losses are still larger than $10^{-1}$ even after $10^5$ iterations.

\section{Conclusions and Discussions}

In this work, we have shown that for learning a certain target network with absolute value activation, a truncated gradient descent algorithm can provably converge in polynomially many iterations starting from a random initialization.
The learned network is more accurate compared to any kernel method that uses polynomially large feature mappings.

We describe several interesting questions for future work.
First, it would be interesting to extend our result to a setting where the target network uses ReLU activation, i.e. $f^{\star}(x) = a^{\top}\relu(Wx)$.
We note that there is a straightforward reduction from the above setting to our setting by simply solving a linear regression.
After applying the reduction, we could then apply our result.
The challenge of directly analyzing gradient descent for learning $f^{\star}(x) = a^{\top}\relu(Wx)$ is that the 1st order tensor decomposition in the Hermite expansion of $f^{\star}(x)$ breaks the conditionally-symmetric property.
Second, it would be interesting to extend our result to settings where $W^{\star}$ is not necessarily orthonormal.
The challenge is to analyze the gradient descent dynamic beyond orthogonal tensors.
We leave this question for future research.

\subsection*{Acknowledgment}

The work is in part supported by SDSI and SAIL. T. M is also supported in part by Lam Research and Google Faculty Award.

\bibliographystyle{plainnat}
\bibliography{main,reference}

\appendix
\newpage
\paragraph{Organizations.}
The appendix provides complete proofs to Theorem \ref{thm:main} and \ref{thm_lb}.
\begin{itemize}
	\item In Section \ref{app_inf}, we describe the proof of Theorem \ref{thm_inf} for the infinite-width case. This section comprises the bulk of the appendix.
	\item In Section \ref{app_finite}, we describe the proof of Theorem \ref{thm:main} by reducing the finite-width case to the infinite-width case.
	\item In Section \ref{app_lb}, we prove Theorem \ref{thm_lb} using ideas from the work of \citet{AL2019-resnet}.
\end{itemize}

\section{Proof of the Infinite-Width Case}\label{app_inf}

We provide the proof of Theorem \ref{thm_inf}, which shows that running truncated gradient descent on an infinite-width network can recover the target network with population loss at most $O(d^{1 + Q})$, where $Q$ is a sufficiently small constant defined in Theorem \ref{thm_inf}.
Recall from Section \ref{sec_overview} that our analysis begins by setting up the random initialization and then proceeds in two stages.
We fill in the proof details left from Section \ref{sec_overview}.
The rest of this section is organized as follows.
\begin{itemize}
	\item {\bf Initialization:} We set up the random initialization used by Algorithm \ref{alg}.
	\item {\bf Stage 1:} We fill in the proof details of the dynamic during Stage 1, which subsumes Stage 1.1 and Stage 1.2 described in Section \ref{sec_stage1}.
	This stage runs for $\Theta(\frac{d^2}{\eta C(\kappa) \log d})$ iterations.
	\item {\bf Stage 2:} We fill in the proof details of the dynamic during Stage 2, which subsumes Stage 2.1 and Stage 2.2 described in Section \ref{sec_stage2}.
	This stage runs for $\Theta(\frac{d^{1 + 10Q}}{\eta})$ iterations.
\end{itemize}

\paragraph{Initialization.}
Recall that for the infinite-width case, our initialization of the neuron distribution is a probability measure truncated from a Gaussian distribution with identity covariance.
We formally define the truncation and the initialization, denoted by $\cP^{(0)}$, as follows.

\begin{definition}[Truncated neuron space]\label{defn:true}
Let $\set{S}_g\subseteq\real^d$ be the set of all $w\in\real^d$ that satisfies the following properties:
\begin{itemize}
\item The maximum entry of $w$ is bounded: $\| w \|_{\infty} \leq \frac{\poly\log(d)}{\sqrt{d}}$.
\item Both $\|w \|_2^2$ \text{ and } $\sum_{i=1}^d a_i  d \cdot w_i^2$ \text{ are in the range }
	\begin{align}\label{eq:nkvajbkfjaf}
		  \left[ 1 -\frac{\poly\log(d)}{\sqrt{d}}, 1 + \frac{\poly\log(d)}{\sqrt{d}} \right].
	\end{align}
\item There are at most $O(\log^{0.01} (d))$ coordinates $i\in[d]$ of $w$ such that $w_i^2 \geq \frac{\log d}{d}$.
\end{itemize}
We define $\cP^{(0)}$ as the probability measure of $\cN(0, \id_{d\times d} /d)$ conditional on the support set $\set{S}_g$.
\end{definition}

\noindent{\it Remark.}
For our purpose of proving the finite-width case later in Section \ref{app_finite}, it suffices to consider $\cP^{(0)}$ as the initialization as opposed to $\cN(0, \id_{d\times d} / d)$.
This is because when Algorithm \ref{alg} samples $m = \poly_{\kappa}(d)$ neurons from $\cN(0, \id_{d\times d} / d)$, with high probability all the $m$ samples are in the set $\set{S}_g$.
To see this, by standard concentration inequalities for the Gaussian distribution, we can show that the set $\set{S}_g$ has probability measure at least $\mu(\set{S}_g) \geq 1 - \frac{1}{d^{\Omega(1)}}$. %
Thus by union bound, with high probability all $m$ samples are in $\set{S}_g$.

As stated in Section \ref{sec_overview}, we are going to heavily use the  conditionally-symmetric property (cf. Definition \ref{def_cs_prop}).
We observe that the initialization $\cP^{(0)}$ is indeed conditionally-symmetric.
This is because $\cN\left(0, \id_{d\times d} / d\right)$ satisfies the conditionally-symmetric property and our truncation in Definition \ref{defn:true} only involves conditions on the square of the coordinates of $w$.
Hence the truncation of $\cN(0, \id_{d\times d} / d)$ to $\cS_g$ preserves the conditionally-symmetric condition.

\smallskip
\noindent{\it Notations for gradients.}
Before describing the analysis, we introduce several notations first.
Recall from Claim \ref{claim:symmetry} that the gradient of a neuron $v$ in the distribution $\cP$ can be simplified given the conditionally-symmetric property.
For each coordinate $1\le i\le d$, the gradient of neuron $v$ satisfies that
$[\nabla_{v}]_i = \sum_{j \geq 0} \left[ \nabla_{2j , v} \right]_i$,
where $\nabla_v = \nabla_v L_{\infty}(\cP)$, $\nabla_{2j, v} = \nabla_{2j, v} L_{\infty}(\cP)$ denotes the gradient of $v$ for the $2j$-th loss, and $[\nabla_v]_i$ denotes the $i$-th coordinate of $\nabla_v$.
Let $B_{1, 2j} =  b_{2j} + b_{2j}'$ and $B_{2, 2j} =  b_{2j}'$, where $b_{2j}$ and $b_{2j}'$ are the Hermite coefficients of the $2j$-th loss given in Section \ref{sec_prelim}.
For a vector $w \in \set{S}_g$, let $w^{(0)}$ denote a neuron with initialization $w$ in the initialization $\cP^{(0)}$.
Let $\cP^{(t)}$ denote the $t$-th iterate of $\cP^{(0)}$ following the update rule of equation \eqref{eq:update_p}.

\paragraph{Stage 1.}
Recall from Section \ref{sec_stage1} that the goal of Stage 1 is to show that a small fraction of neurons becomes basis-like, i.e. close to a basis $e_i$ times a scaling factor of $\poly(d)$ at the end of $\Theta_{\kappa}(d^2/\eta \log d)$ iterations for some $i\in[d]$.
To facilitate the analysis, we maintain an inductive hypothesis throughout Stage 1 that provides an upper bound on the norm of a typical neuron during the update.
We first introduce the set of neurons that will not become basis-like by the end of Stage 1.

\begin{definition}\label{defn:no_win}
	Let $C_0$ be a large enough constant.
	Let $c_0 = C_0 \log d$ and $\set{S}$ be the set of all vectors $w$ in $\set{S}_g$ such that
\begin{align*}
		\| w \|_{\infty}^2 \le \frac{c_0}{d} \text{ and } \| \bar{w} \|_{\infty}^2 \leq \frac{c_0}{d},
\end{align*}
where $\bar{w} = w / \norm{w}$ denotes $w$ being normalized to norm $1$.
\end{definition}

Based on the above definition, we introduce the following inductive hypothesis that shows the neurons in $\set{S}$ remain ``small and dense'' (i.e. not basis-like) throughout Stage 1.
This stage runs for $ \Theta(\frac{d^2}{\eta \log d})$ iterations.
We use $\Kappa$ to denote a value that is less than $O(\exp(\poly(\kappa)))$.

\begin{proposition}[Inductive hypothesis $\set{H}_1$ for Stage 1]\label{def_H0}
In the setting of Theorem \ref{thm_inf}, let  $T_2 = \Theta(\frac{d^2}{\eta c_0\exp(\poly(\kappa))})$.
There exists an increasing sequence $\Set{c_t}_{t=1}^{T_2}$ where $c_t \leq \exp(\poly(\kappa)) \log d$ such that for every  $w \in \set{S}$ and every $t \le T_2$, the $t$-th iterate of the neuron $w^{(t)}$ with initialization $w^{(0)} = w$ satisfies that
	\begin{align}
		\| w^{(t)} \|_{\infty}^2 \le \frac{c_t}d \text{ and } \| \bar{w}^{(t)} \|_{\infty}^2 \leq \frac{c_t}{d}. \label{eq:fajosifsajfasjif}
	\end{align}
Furthermore, for every coordinate $i \in [d]$, we have that in expectation,
	\begin{align}
		\E_{w^{(t)} \sim \cP^{(t)}}[{w_i^{(t)}}^2]\leq   \frac{2 \kappa}{d} \text{ and } \E_{w^{(t)}\sim\cP^{(t)}}[\bar{w^{(t)}}_i^2]  \leq \frac{4 \kappa^2}{d}. \label{eq_H0_2}
	\end{align}
\end{proposition}
Equation \eqref{eq:fajosifsajfasjif} and \eqref{eq_H0_2}, which we also refer to as inductive hypothesis $\set{H}_1$, show that the norm of any neuron in $\set{S}$ will not grow beyond $O_{\kappa}(\log d / d)$.
Hence they will not become basis-like during Stage 1.

The set $\set{S}$ contains most neurons in $\cP^{(0)}$ because by standard anti-concentration inequalities, the measure of the set $\set{S}$ is at least $1 - d^{-O(C_0)}$.
Hence, $1 - \mu(\set{S})$ is at most $d^{-O(C_0)}$.
Based on this fact, we state a simple claim on the norm of neurons that are not in $\set{S}$ that will be used later:
\begin{align} \label{lem:norm_bound}
		\E_{w^{(t)} \sim \cP^{(t)}, w^{(0)} \notin \set{S}} \| w^{(t)} \|_2^2 \leq  \Lambda  :=  O\left(\frac{1}{\lambda_0} (1 - \mu(\set{S}))   \right) \le \frac{1}{\poly(d)}.
\end{align}
To see that equation \eqref{lem:norm_bound} is true, recall that the truncation of Algorithm \ref{alg} ensures that $\norm{w}^2 \le 1/ \lambda_0$.
Combined with the fact that $1 - \mu(\set{S}) \le d^{-O(C)}$ and $\lambda_0 = \Theta(1 / \poly(d))$, we have that equation \eqref{lem:norm_bound} holds for a sufficiently large constant $C_0$.
This finishes our introduction of the inductive hypothesis $\set{H}_1$.
The proof of Proposition \ref{def_H0} can be found in Section \ref{app_proof_H0}.

\bigskip
Given the inductive hypothesis $\set{H}_1$, we can state the formal result that corresponds to Stage 1.1 in Section \ref{sec_stage1}.
For a neuron distribution $\cP$, let us first introduce the following notations, which corresponds to the population loss of the 0th and 2nd order tensor decompositions.
\begin{align*}
  \Delta_+ &\define b_0 \sum_{i=1}^d [ \E_{w \sim \cP} w_i^2 - a_i]^+,
  \Delta_- \define b_0 \sum_{i=1}^d [  a_i - \E_{w \sim \cP} w_i^2]^+,\\
  \delta_+ &\define  b_2 \max_{i \in [d]} [ \E_{w \sim \cP} w_i^2 - a_i]^+,
    \delta_- \define b_2 \max_{i \in [d]} [ a_i - \E_{w \sim \cP} w_i^2 ]^+.
\end{align*}
Let $\Delta \define \Delta_+ - \Delta_-$ denote an upper bound on the 0th order loss.
Let $\delta_i \define  b_2(\E_{\cP} w_i^2 - a_i)$ for every $1\le i\le d$.
At the $t$-th iteration, we use $\delta_i^{(t)}$ to denote the value of $\delta_i$ given the neuron distribution $\cP^{(t)}$, as well as $\Delta^{(t)}$ for $\Delta$, $\delta_+^{(t)}$ for $\delta_+$, and $\delta_-^{(t)}$ for $\delta_-$.

Based on the above notations, we show the following convergence result at the end of Stage 1.1.
\begin{lemma}[Stage 1.1: learning 0th and 2nd order tensors]\label{lem:zero_two3}
  In the setting of Theorem \ref{thm_inf}, suppose that Proposition \ref{def_H0} holds.
	Let $T_1 = \Theta\left( \frac{\poly(\Kappa) d \log d}{\eta } \right)$.
	Then, for every $t \ge T_1$, we have that
		$\Delta^{(t)}, \delta_+^{(t)}, \delta_-^{(t)}$ are all less than $\frac{c_t \poly(\Kappa)}{d^2}$, where $c_t$ is given in Proposition \ref{def_H0}.
\end{lemma}%
The above result implies that after $T_1$ iterations, the 0th and 2nd order losses  remain smaller than $c_t \poly(\Kappa) / d^2$.
The proof of Lemma \ref{lem:zero_two3} can be found in Section \ref{sec_proof_lemma_A2}.

\bigskip
Once Stage 1.1 is finished, recall from Section \ref{sec_overview} that the higher order gradients begin to dominate the dynamic.
Hence Algorithm \ref{alg} enters Stage 1.2.
We introduce the following notations in order to state the formal result.
Let $T_2' =  T_2 -  \frac{d^2}{\eta \poly\log(d)}$.
For every $1\le i\le d$, let $\Gamma_i = \frac{1}{2 B_{1, 4} (a_i^2 d) (\eta T_2')}$.
Let $\rho = \frac{ \poly(\Kappa) \cdot \log d}{d}$.
Here, by our assumption, we know that $a_i^2 = \Theta({1}/{d^2})$.
Since $T_2' = \Theta({d^2}/(\eta \log d))$, we can see that $\Gamma_i = \Theta({\log d}/{d})$.
Consider a coordinate $i \in [d]$.
We define the set of good neurons whose $i$-th coordinate is larger than $\Gamma_i + \rho$ as
\[ \set{S}_{i, good} \define \left\{ v \in \set{S}_g \mid [v^{(0)}]_i^2 \geq \Gamma_i  + \rho \text{ and for all other } j\neq i: [v^{(0)}]_j^2 < \Gamma_j - \rho  \right\}. \]
Then we define the set of bad neurons that have two large coordinates as
\[ \set{S}_{i, bad} = \left\{ v \in \set{S}_g \mid [v^{(0)}]_i^2 \geq \Gamma_i  - \rho \text{ and there exists } r \neq i: [v^{(0)}]_r^2 \ge \Gamma_r - \rho  \right\}. \]
The following lemma shows that, among other statements, the neurons in $\set{S}_{i, good}$ will win the lottery and become basis-like at the end of Stage 1.2 in the sense described below.
\begin{lemma}[Stage 1.2: learning higher order tensors] \label{lem:stage_1_final}%
	In the setting of Theorem \ref{thm_inf}, suppose that Proposition \ref{def_H0} holds.
	At iteration $T_2$ (recall that $T_2$ is defined in Proposition \ref{def_H0}), the following holds for $\set{S}_{i, good}$ and $\set{S}_{i, bad}$:
	\begin{itemize}
		\item For every $i\in[d]$ and every $v \in \set{S}_{i, good}$, we have that
			\begin{align*}
				|v_i^{(T_2)}|^2  \geq \frac{1}{\lambda_0 \poly(d)} \ge \poly(d),
				\text{ and for every $j \neq i$, } |v_j^{(T_2)}|  \leq \frac{2(\log d)^2}{\sqrt{d}}.
			\end{align*}
		\item For every $1\le i\le d$ and every $v \in \set{S}_g$, if there exists $j \neq i$ such that $|v_i^{(T_2)}|$ and $|v_j^{(T_2)}|$ are both greater than $\frac{2(\log d)^2}{\sqrt{d}}$, then the neuron $v$ is in the union of $ \set{S}_{i, bad} $ and $ \set{S}_{j, bad}$.
		\item For every $i \in [d]$, the probability measure of $\set{S}_{i, good}$ and $\set{S}_{i, bad}$ satisfies that
				\[ \mu(\set{S}_{i, good}) \ge d^{-\exp(\poly(\Kappa))} ~\text{ and }~ %
			\mu(\set{S}_{i, good}) \ge \mu(\set{S}_{i, bad}) \cdot d^{\exp(\poly(\Kappa))}.\]%
\end{itemize}
\end{lemma}

In the above result, the set $\set{S}_{i, good}$ contains neurons that become approximately a large scaling of the basis $e_i$ after $T_2$ iterations, a phenomenon that we term as winning the lottery ticket.
The norm of these neurons become much larger than those in $\set{S}$, whose norm is bounded by $O_{\kappa}(\log d / d)$.
The set $\set{S}_{i, bad}$ contains neurons whose coordinate $i$ might be large in the end, but not close to a basis.
The final statement in this lemma shows that the probability measure of bad neurons is small compared to good neurons.
Lemma \ref{lem:stage_1_final} is proved in Section \ref{app_stage_21}.
This concludes Stage 1.

\paragraph{Stage 2.}
The second stage begins by reducing the gradient truncation parameter from $\lambda_0 = \Theta(\frac{1} {\poly(d)})$ to $\lambda_1 = \Theta(\frac{1}{\poly_{\kappa}(d)})$.%
\footnote{As a remark, the rational for this technical twist is that the neurons do not grow too large Stage 1. This is useful for the error analysis later in the finite-width case.}
Recall from Section \ref{sec_stage2} that the goal of Stage 2 is to allow basis-like neurons to grow until they fit the target network with population loss at most $O(d^{1 + Q})$.
\begin{itemize}
	\item The first substage of the analysis shows that the population loss reduces below $O(\frac {1} {d \log^{0.01} d})$, after $T_3 = \Theta({d \log d}/{\eta})$ many iterations.
	\item The second substage of the analysis shows that the population loss further reduces below $O\left({1}/{d^{1 + \cz}} \right)$, after $T_4 = \Theta({d^{1 + 10 \cz}} / {\eta})$ many iterations.
\end{itemize}

To facilitate the analysis, we introduce an inductive hypothesis throughout Stage 2 that describes the behavior of the good and bad neurons.
Let us introduce several notations first.
Let the union of the bad neurons for all coordinates be given by
	\[ \set{S}_{bad} \define \{ v  \in \set{S}_g \mid \exists i\neq j \text{ such that } [v^{(0)}]_i^2 \geq \Gamma_i  - \rho \text{ and } [v^{(0)}]_j^2 \geq  \Gamma_j - \rho   \}. \]
The set of potential neurons for coordinate $i\in[d]$ is given by
	\[ \set{S}_{i, pot} = \left\{ v \in \set{S}_g \mid [v^{(0)}]_i^2 \geq \Gamma_i  - \rho \right\}. \]
We remark that these are the set of neurons whose coordinate $i$ can become larger than $O(\frac{\poly\log(d)}{\sqrt{d}})$ at the end of Stage 1 (cf. Section \ref{sec_proof_lemma_A2}).
The set of good neurons $\set{S}_{i, good}$ is a subset of $\set{S}_{i, pot}$.
Let the union of the potential neurons for all coordinates be given by
	\[ \set{S}_{pot} \define \{ v   \in \set{S}_g \mid \exists i \in [d] \text{ such that } [v^{(0)}]_i^2 \geq \Gamma_i  - \rho  \}. \]
We maintain the following running hypothesis that, among other things, specifies the behavior of the potential, good, and bad neurons in detail.
\begin{proposition}[Inductive hypothesis $\set{H}_2$ for Stage 2]\label{def_H1}
	In the setting of Theorem \ref{thm_inf}, there exists a monotonically increasing sequence $\set{c_t}_{t=T_2}^{T_4}$ such that $c_{T_2} = \poly(\log d) \le c_t \leq d^{O(Q)} \le d^{1/10}$ and for every $T_2 < t \le T_4$, the following list of properties holds for the neuron distribution $\cP^{(t)}$:
	\begin{enumerate}
		\item For every $v\in\set{S}_g$, we have that $\norm{v^{(t)}}_2^2 \le 1 / \lambda_1$. As a result, gradient truncation never happens during this stage.
		\item For every $v \notin \set{S}_{pot}$, we have that%
				\begin{align}\label{eq:fsajfoiofasjcjuiefh}
					\|\bar{v}^{(t)}\|_{\infty}^2 \le \frac{c_t}d \text{ and } \|{v^{(t)}}\|_{\infty}^2  \leq \frac{c_t}{d}.
				\end{align}
		For every $i \in [d]$, every $v \in \set{S}_{i, pot} \backslash   \set{S}_{bad}$, and $j \not = i$, we have that
			\begin{align}\label{eq:bnodsifahsoia}
				\| v_j^{(t)} \|_2^2 \leq \frac{c_t}{d}.
			\end{align}
		\item The probability mass of the set of bad neurons satisfies that
			\begin{align}\label{eq:vajoisajfoiasfjisajf}
				\E_{v^{(t)} \sim \cP^{(t)}, v \in \set{S}_{bad}} \| v^{(t)}\|_2^2  \leq \frac{1}{\poly(d)}.
			\end{align}
		\item For every $i \in [d]$ and every $v \in \set{S}_{i, good}$, we have that $\| v_i^{(t)} \|_2^2 \geq \frac{1}{\lambda_0 \poly(d)}$.
		\item For every $i\in [d]$, the following claims regarding the set of potential neurons and bad neurons hold:
			\begin{align}
				\gamma_i^{(t)}  &\define \E_{v^{(t)} \sim \cP^{(t)}, v \in  \set{S}_{i, pot} \backslash \set{S}_{bad}} {v_i^{(t)}}^2 \le \frac{\poly(\Kapppa)}{d}, \label{eq_gamma}\\
				\beta_i^{(t)}  &\define \E_{v^{(t)} \sim \cP^{(t)}, v \notin \set{S}_{pot}} {v_i^{(t)}}^2
				\leq \frac{\poly(\Kapppa)}{d}. \label{eq_beta}
			\end{align}
		where $\Kapppa$ denotes $\exp(\poly(\Kappa))$ and $\Kappa$ denotes $\exp(\poly(\kappa))$.
\end{enumerate}
\end{proposition}

We remark that in the above inductive hypothesis, equation \eqref{eq:fsajfoiofasjcjuiefh} and \eqref{eq:bnodsifahsoia} show similar conditions as equation \eqref{eq:fajosifsajfasjif}  provided in Proposition \ref{def_H0}.
For the rest of the section, we refer to the conclusion of Proposition \ref{def_H1} as inductive hypothesis $\set{H}_2$.
The proof of Proposition \ref{def_H1} can be found in Section \ref{app_proof_H1}.

\bigskip
Given the inductive hypothesis, we can state the formal result that corresponds to Stage 2.1 in Section \ref{sec_stage2}.
We introduce the notation	$\Delta^{(t)} = 2 b_0 \left( \sum_{i=1}^d ( \gamma_i^{(t)} + \beta_i^{(t)})  - \sum_{i=1}^d a_i \right)$ that measures the average error of the neurons across all coordinates at iteration $t$.
We show that by the end of $T_3 = T_2 + \Theta(d\log d /\eta)$ iterations, we have obtained a warm start neuron distribution for $\Delta^{(t)}$, $\Set{\beta_1^{(t)}, \dots, \beta_d^{(t)}}$, and $\Set{\gamma_1^{(t)}, \dots, \gamma_d^{(t)}}$.
We state the result below.
\begin{lemma}[Stage 2.1: Obtaining a warm start initialization]\label{lem:final_222}
	In the setting of Theorem \ref{thm_inf}, suppose Proposition \ref{def_H1} holds.
	There exists an iteration $T_3 = T_2 + \Theta(d\log d/\eta)$ such that at iteration $T_3$, the following holds:
	\begin{align*}
		\text{For any}~ i \in [d],~~ \beta_i^{(T_3)} \leq \frac{1}{d \log^{0.01} d}, \quad |a_i - \gamma_i^{(T_3)}| \leq  \frac{1}{d \log^{0.01} d};  ~\text{ Furthermore, } |\Delta^{(T_3)}| \leq  \frac{1}{d \log^{0.01} d}.
	\end{align*}
\end{lemma}
The above result implies that the set of potential neurons has fit the $i$-th coordinate of the target network with error less than $o(1/d)$.
The 0th order loss has also been reduced below $o(1/d)$.
The proof of Lemma \ref{lem:final_222} can be found in Appendix \ref{app_stage_21}.

\bigskip
In the end, we describe the formal result that corresponds to Stage 2.2 in Section \ref{sec_stage2}.
We construct a potential function to show that $\beta_i^{(t)} + \gamma_i^{(t)}$ converges to $a_i$ when $t \ge T_3$.
After running for $T_4 = T_3 + \Theta(\frac{d^{1 + 10 \cz}}{\eta})$ many iterations, we show that a certain set of potential neurons has converged to $a_i$ with error at most $O(1/ d^{2 + Q})$, for every $1\le i\le d$.

The result is shown in Lemma \ref{lem:final_333} below.
We introduce the following notations for defining the potential function at iteration $t$:
\begin{align*}
	\delta_-^{(t)} &= \max\left\{\max_{i \in [d]}\left\{ C_1( a_i - \beta_i^{(t)}  - \gamma_i^{(t)} ) + \frac{C_2 \gamma_i^{(t)} }{\beta_i^{(t)}  + \gamma_i^{(t)} } \left( a_i - \gamma_i^{(t)}  \right) \right\}, 0 \right\}, \\
	\delta_+^{(t)} &= \max\left\{ \max_{i \in [d]}\left\{  C_1(\beta_i^{(t)}  + \gamma_i^{(t)}  - a_i)  + \frac{C_2 \gamma_i^{(t)} }{\beta_i^{(t)}  + \gamma_i^{(t)} } \left( \gamma_i^{(t)}   -  a_i \right) \right\}, 0\right\}.
\end{align*}
where $C_1, C_2$ denote two sufficiently large constants.
Consider the following functions (recall that $\Delta_+$ and $\Delta_-$ have been defined in Stage 1):
\begin{align*}
	\Phi_+^{(t)}  = \max \left\{ \delta_+^{(t)} , \left( 1 + \frac{1}{\poly(\kappa ) } \right)\Delta_-^{(t)}   \right\} \text{ and }
	\Phi_-^{(t)}  = \max \left\{ \delta_-^{(t)} ,  \left( 1 + \frac{1}{\poly(\kappa ) }  \right)\Delta_+^{(t)}  \right\}.
\end{align*}
Let $\beta_+^{(t)}  = \frac{1}{C} \max_{i \in [d]} \{ \beta_i^{(t)}  \}$.
Let $\Phi^{(t)}  = \max \{ \Phi_+^{(t)} , \Phi_-^{(t)} , \beta_+^{(t)} \}$ be our potential function.
Lemma \ref{lem:final_222} implies that by the end of $t = T_3$ iterations, we have that $\delta_-^{(t)} , \delta_+^{(t)} , \beta_+^{(t)} , \Delta_+^{(t)} , \Delta_-^{(t)}  $ are all less than $ O\left({1}/{d \log^{0.01} d} \right)$.
Hence $\Phi^{(T_3)} \le O(1 / (d \log^{0.01} d))$.
The result below shows that after iteration $T_3$, $\Phi^{(t)}$ further decreases whenever $\Phi^{(t)}$ is at least $O(\frac{\poly(\kappa_2)c_t)}{d^2})$.
\begin{lemma}[Stage 2.2: the final substage] \label{lem:final_333}
	In the setting of Theorem \ref{thm_inf}, suppose that Proposition \ref{def_H1} holds.
	Let $C_1$ be a fixed constant.
	For any $T_3 < t \le T_4$, as long as $\Phi^{(t )} \geq \frac{\poly(\Kapppa ) c_t}{d^2}$ (recalling that $c_t$ is defined in Proposition \ref{def_H1}) we have that
	\begin{align*}
		\Phi^{(t + 1)} \leq \Phi^{(t )} \left( 1 - \eta \frac{\min\{C_1, 1\}}{8} \Phi^{(t )} \right).
	\end{align*}
\end{lemma}

By combining the results of Stage 1 and Stage 2, we are ready to prove Theorem \ref{thm_inf}.
\begin{proof}[Proof of Theorem \ref{thm_inf}]
	When Proposition \ref{def_H0} and \ref{def_H1} hold,
	using the induction hypothesis in equation~\eqref{eq:vajoisajfoiasfjisajf}, we have that for the infinite-width case, the population loss $\objinf(\cP^{(t )}) $ satisfies:
	\begin{align}\label{claim:addition_1}
		\objinf(\cP^{(t)}) =  O\left( \left( \sum_{i \in [d]}  \left[\gamma_i^{(t)} + \beta_i^{(t)} \right] - \sum_{i \in [d]} a_i \right)^2\right) +  O\left( \sum_{i \in [d]} \left[ (a_i - \gamma_i^{(t)})^2  + \beta_i^2 \right]\right) + \frac{1}{\poly(d)},
	\end{align}
	where the first term comes from the 0th order loss and the second term comes from 2nd and higher order losses.
	This claim also implies that
	\begin{align}
		\objinf(\cP^{(t)})  = O\left( d  [\Phi^{(t)}]^2 \right) +  \frac{1}{\poly(d)}. \label{eq_phi}
	\end{align}

	At the beginning of Stage 2.2, by Lemma \ref{lem:final_222}, we know that $\Phi^{(T_3)} \le 1 / (d\log^{0.01} d)$.
	During Stage 2.2, by Lemma \ref{lem:final_333}, as long as $\Phi^{(t)} \ge O_{\kappa}(c_t /d^2)$, $\Phi^{t+1} \le \Phi^{(t)} \le \Phi^{(t)} (1 - O(\Phi^{(t)}))$.
	Hence, after at most $d^{1 + O(Q)} / \eta$ iterations (or $T_4-T_3$ more precisely), $\Phi^{(T_4)}$ reduces to below $O(d^{1 + Q})$.
	Applying this result to equation \eqref{eq_phi}, we conclude that $L_{\infty}(\cP^{(T_4)}) \le O(1/d^{1 + Q})$.

\end{proof}

\subsection{Stage 1.1: Proof of Convergence for 0th and 2nd Order Tensors}

This section provides the proof of Lemma \ref{lem:zero_two3} is organized as follows.
\begin{itemize}
	\item In Proposition \ref{lem:four_plus_upper}, we first show that the gradients from 4th and higher order tensor decompositions are small compared to that of the 0th and 2nd order tensor decompositions.
	\item The above shows that the dynamic is mainly dominated by the 0th and 2nd losses initially.
	In Proposition \ref{lem:zero_two} and Proposition \ref{lem:zero_two2}, we show the gradient update of the 0th and 2nd order.
	Based on these, we show the proof Lemma \ref{lem:zero_two3} at the end of this subsection.
\end{itemize}

\paragraph{Upper bound on the gradient of 4th and higher order losses.}
We first show that the 4th and higher order tensor gradients do not have much contribution to the gradient, for all the neurons in $\set{S}$.
We introduce the following notations for convenience.
For a neuron distribution $\cP$, let the following denote the gradient of $v$ involving only other neurons $w$.
\begin{align} \label{eq:2jvn}
	\nabla_{2j, v, n}  &\define \left( b_{2j} + b_{2j}'\right) \left( \E_{w \sim \cP} \langle w, v \rangle \langle \bar{w} , \bar{v} \rangle^{2j - 2} w \right)  - b_{2j}'\left( \E_{w \sim \cP} \langle w, v \rangle \langle \bar{w} , \bar{v} \rangle^{2j - 2} \langle w, \bar{v} \rangle \right)  \bar{v}.
\end{align}
Recall that $\nabla_{2j, v}$ is the gradient of $v$ for the 2j-th tensor (cf.  equation \eqref{eq:def04}).
Let
\begin{align*}
	\nabla_{\ge 4, v} = \sum_{j \ge 2} \nabla_{2j, v}, \text{ and }
	\nabla_{\ge 4, v, n} = \sum_{j \ge 2} \nabla_{2j, v, n}.
\end{align*}
The following result provides an upper bound on the higher order gradients.

\begin{proposition}[Upper bound for 4th or higher order gradients] \label{lem:four_plus_upper}
In the setting of Lemma \ref{lem:zero_two3},  suppose Proposition \ref{def_H0} holds.
Then there exists an absolute constant $C > 0$ such that for every $i \in [d]$ and $v\in\set{S}_g$, at the $t$-th iteration for $t \le T_2$, the neuron $v$ from distribution $\cP^{(t)}$ satisfies that
\begin{align*}
	\left| \left[ \nabla_{\geq 4 , v} \right]_i  \right| \leq  \frac{C}{4} \left( \frac{c_t \kappa }{d^2} + \frac{\kappa}{d}  \|\bar{v}^{(t)} \|_{\infty}^2  \right) |v_i^{(t)}|.
\end{align*}
Moreover, the gradient from the network satisfies
\begin{align*}
\left| \left[ \nabla_{\geq 4 , v, n} \right]_i  \right| \leq  \frac{C}{4} \cdot \frac{c_t \kappa }{d^2}   |v_i^{(t)}|.
\end{align*}
As a corollary, for every $v \in \set{S}\subseteq \set{S}_g$, we have that
\begin{align*}
\left| \left[ \nabla_{\geq 4 , v} \right]_i  \right| \leq \frac{C}{2} \cdot \frac{c_t \kappa }{d^2}  |v_i^{(t)}|.
\end{align*}
\end{proposition}

\begin{proof}%
Let us focus on $\nabla_{4, v}$ first.
We now bound each term in $[\nabla_{4, v}]_i$ in equation~\eqref{eq:def04} separately.
\begin{align*}
	[\nabla_{4 , v}]_i &= \left( b_{4} + b_{4}'\right) \left( \exarg{w\sim \cP}{\sum_{i_1, i_{2}} \bigbrace{\prod_{r =1} (\bar{w}_{i_r} \bar{v}_{i_r})^2} ({w}_{i_j})^2 {v}_{i_j} } -   a_i \langle e_i, v \rangle \langle e_i, \bar{v} \rangle^{2}  \right)  \\
& ~~~~ - b_{4}'\left(  \exarg{w\sim \cP}{\|w\|_2^2 \sum_{i_1, i_2} \prod_{r = 1}^2 (\bar{w}_{i_r} \bar{v}_{i_r})^2 }  -   \sum_{r \in [d]} a_r \langle e_r, \bar{v} \rangle^{4}  \right) v_i.
\end{align*}
For the second two line of the above,
\begin{align}
 \left| \E_{\cP} \left[ w_i^2 \sum_{j = 1}^d  \left( \bar{w}_j  \bar{v}_j  \right)^{2} \right] v_i \right| & \leq \left| \E_{w \sim \cP, w^{(0)} \in \set{S}} \left[ w_i^2 \sum_{j = 1}^d  \left( \bar{w}_j  \bar{v}_j  \right)^{2} \right] v_i \right|  + \left| \E_{w \sim \cP, w^{(0)} \notin \set{S}} \left[ w_i^2 \sum_{j = 1}^d  \left( \bar{w}_j  \bar{v}_j  \right)^{2} \right] v_i \right| \nonumber
\\
& \leq |v_i|  \left| \E_{w \sim \cP, w^{(0)} \in \set{S}} \left[ w_i^2 \sum_{j = 1}^d  \left( \bar{w}_j  \bar{v}_j  \right)^{2} \right]  \right| +\Lambda |v_i| \nonumber
\\
& \leq |v_i| \frac{c_t}{d}  \left| \E_{w \sim \cP, w^{(0)} \in \set{S}} \left[ w_i^2 \sum_{j = 1}^d  \left( \bar{v}_j  \right)^{2} \right]  \right| +\Lambda |v_i| \nonumber
\\
& \leq |v_i|  \left( \frac{2 c_t \kappa}{d^2}  + \frac{1}{\poly(d)} \right) \label{eq:vjaoisfjiai1}
\end{align}
where the second inequality uses inequality~\eqref{lem:norm_bound} so $\E_{w \sim \cP^{(t)}, w^{(0)} \notin \set{S}} \| w \|_2^2 \leq  \Lambda $, and the second last inequality uses $\E_{\cP}[w_i^2] \leq \frac{2 \kappa}{d}$ as in Eq~\eqref{eq:fajosifsajfasjif}.

For the signal term in the gradient, $a_i\inner{e_i}{v}\inner{e_i}{\bar{v}}^2$, because $a_i \le \kappa /d $, we have
\begin{align}
\left|  a_i \langle e_i, v \rangle \langle e_i, \bar{v} \rangle^{ 2}  \right| \leq \frac{\kappa}{d} |v_i| |\bar{v}_i|^2 \label{eq:vjaoisfjiai2}
\end{align}
Another term  in the gradient  is (again, using the fact that for $w$ with $w^{(0)} \in \set{S}$, $\| w \|_{\infty}^2 \leq \frac{c_t}{d}$):
\begin{align} 
\left |   \sum_{r, r'} \E_{\cP} \frac{w_r^2 v_r^2 w_{r'}^2 v_{r'}^2}{\| w\|_2^2 \| v\|_2^4}  \right| |v_i|  &\leq \left |   \sum_{r, r'} \E_{ w \sim \cP, w^{(0)} \in \set{S}} \frac{w_r^2 v_r^2 w_{r'}^2 v_{r'}^2}{\| w\|_2^2 \| v\|_2^4}  \right| |v_i|   + \left |   \sum_{r, r'} \E_{ w \sim \cP, w^{(0)} \not \set{S}} \frac{w_r^2 v_r^2 w_{r'}^2 v_{r'}^2}{\| w\|_2^2 \| v\|_2^4}  \right| |v_i| \nonumber
\\
& \leq  \frac{c_t}{d } |v_i|   \sum_{r, r'}  \E_{\cP} \frac{w_{r}^2}{\| w\|_2^2}  \frac{v_r^2  v_{r'}^2}{ \| v\|_2^4}+\Lambda\poly(d) |v_i| \nonumber
\\
& \leq |v_i| \left(  \frac{2 \kappa c_t}{d^2}  + \frac{1}{\poly(d)} \right) \label{eq:vjaoisfjiai3}
\end{align}
The last term in the gradient  is given by:
\begin{align}
 \left|  \left( \sum_{r} a_r\langle e_r, v \rangle \langle e_r, \bar{v} \rangle^{3}  \right)  \bar{v}_i \right| &\leq \frac{\kappa}{d} \left( \sum_{r} v_r \bar{v}_r^3 \right) |\bar{v}_i| \nonumber
 \\
 & \leq \frac{\kappa}{d} \frac{\| v \|_4^4}{\| v\|_2^4} |v_i| \leq  \frac{\kappa}{d}\|\bar{v} \|_{\infty}^2| v_i | \label{eq:vjaoisfjiai4}
\end{align}
Combining Eq~\eqref{eq:vjaoisfjiai1}, Eq~\eqref{eq:vjaoisfjiai2}, Eq~\eqref{eq:vjaoisfjiai3} and Eq~\eqref{eq:vjaoisfjiai4}, we obtain that
\begin{align*}
\left| \left[ \nabla_{\geq 4 , v} \right]_i  \right| \leq  O(1) \times \left( \frac{c_t \kappa }{d^2} + \frac{\kappa}{d}  \|\bar{v} \|_{\infty}^2  \right) |v_i|.
\end{align*}
For  $\Delta_{2j , v}$, with $j \geq 3$, we can apply the same calculation as above, and show that 
\begin{align*}
\left| \left[ \nabla_{2j , v} \right]_i  \right| \leq  O(b_{2j} ) \times \left( \frac{c_t \kappa }{d^2} + \frac{\kappa}{d}  \|\bar{v} \|_{\infty}^2  \right) |v_i|.
\end{align*}
Since $\nabla_{\ge 4, v} = \sum_{j \ge 2} \nabla_{2j, v}$ and $\sum_j b_{2j } = O(1)$ we complete the proof.
\end{proof}

Based on the above result, we describe the dynamic of the 0th order tensor in the following proposition.
\begin{proposition}[Learning 0th order tensor] \label{lem:zero_two}
	In the setting of Lemma \ref{lem:zero_two3}, suppose Proposition \ref{def_H0} holds.
	Assume that for every $1\le i\le d$,
	$\E_{w^{(t)} \sim \cP^{(t)}}[{w^{(t)}}_i^2] \ge \frac{1}{\Kappa d}$.
	Then for every $t\le T_2$, at least \emph{one of} the following holds:
 \begin{enumerate}
 \item $|\Delta^{(t)}  | \leq \frac{c_t \poly(\Kappa)}{d^2}$.
 \item If $\Delta^{(t)} > 0$, then $\Delta^{(t)} \leq \delta_-^{(t)} \left(1 - \frac{1}{16 \kappa^2 d} \right)$. If $\Delta^{(t)} < 0$, then $| \Delta^{(t )} | \leq \left( 1 -   \frac{1}{5 \Kappa^5 d} \right) \delta_+^{(t )}$.
 \end{enumerate}
 Moreover, when $ \Delta^{(t)}  \geq \max \left\{8 \kappa^2  \delta_+^{(t)},  \frac{c_t \poly(\Kappa)}{d^2} \right\}$ , it holds that
\begin{align}\label{eq:fsafjaoifjasfoiajsoiaj}
\Delta^{(t + 1)} \leq  \Delta^{(t)} - \eta \frac{1}{4 \kappa d} \Delta^{(t)} | \set{S}^{(t)}_+|,
\end{align}
where $\set{S}^{(t)}_+$ is the set of all $i \in [d]$ with $\delta_i^{(t)} \geq 0$ and $\Abs{\set{S}^{(t)}_+}$ denotes its cardinality.
 \end{proposition}

\begin{proof}%
 Consider the iteration $t$, we have that for every $i$ and every $v \in \set{S}_g$, the update of $v_i^{(t)}$ is given as:
 \begin{align*}
 v_i^{(t + 1)} = v_i^{(t)} - \eta (\Delta^{(t)} + \delta_i^{(t)} ) v_i^{(t)} - \eta [\nabla_{\geq 4, v}^{(t)}]_i.
 \end{align*}
Hence, using Proposition~\ref{lem:four_plus_upper} and inequality~\eqref{lem:norm_bound}, it holds that
\begin{align} \label{eq:aofahfoiashoafhaso}
\E_{\cP^{(t + 1)}} [w_i^2] &= (1 -  2\eta (\Delta^{(t)} + \delta_i^{(t)} ) )\E_{\cP^{(t)}} [w_i^2] \pm \left( 2\eta C \frac{c_t \kappa^2}{d^3} +  \eta^2 \frac{1}{\lambda_0}  + \eta \Lambda \frac{1}{\lambda_0}\right).
\end{align}
This implies that for every sufficiently small $\eta \leq \lambda_0^2$ and $\Lambda \leq \lambda_0^2$, it holds:
 \begin{align}\label{eq:vnsklfafkajf}
\delta_i^{(t + 1)} &= \delta_i^{(t )} -2 \eta \left( \Delta^{(t)} + \delta_i^{(t)}  \right)  \E_{\cP^{(t)}} [w_i^2] \pm 3 \eta C \frac{c_t \kappa^2}{d^3}
\end{align}
Let us consider two cases when $\abs{\Delta^{(t)}} = \Omega(\kappa_1^8 c_t / d^2)$.
\paragraph{Case 1.} $ \Delta^{(t)} = \Omega\left(\frac{\Kappa^8 c_t}{d^2} \right)$, we have that  for every $i$ with $\delta_i^{(t)} \geq 0$, it holds that $\E_{\cP^{(t)}}[ w_i^2] \geq a_i \geq \frac{1}{\kappa d}$. Hence,
\begin{align*}
\delta_i^{(t + 1)} &\leq \delta_i^{(t)} \left( 1 - \eta \frac{1}{\kappa d}\right) +  3 \eta C \frac{c_t \kappa^2}{d^3} - \eta \Delta^{(t)} \frac{1}{\kappa d}
\\
& \leq  \delta_i^{(t)} \left( 1 - \eta \frac{1}{\kappa d}\right)  -\eta \frac{1}{2 \kappa d} \Delta^{(t)}.
\end{align*}
Summing up over all those $i$ gives us
\begin{align} \label{eq:Baojisfajasfsoifjas}
\Delta_+^{(t + 1)} \leq \Delta_+^{(t)}\left( 1 - \eta \frac{1}{\kappa d}\right)  - \eta \frac{1}{2 \kappa d} \Delta^{(t)} | \set{S}^{(t)}_+|.
\end{align}
On the other hand, for every $i$ with $\delta_i^{(t)} \leq 0$, it holds that $\E_{\cP^{(t)}}[ w_i^2] \leq a_i \leq \frac{\kappa}{ d}$. Hence,
\begin{align} \label{eq:vnskbdshjafjha}
|\delta_i^{(t + 1)} |  \geq  |\delta_i^{(t )} |  - \eta \left( |\delta_i^{(t )} |  -  \Delta^{(t)}  \right) \E_{\cP^{(t)}}[ w_i^2] -  3 \eta C \frac{c_t \kappa^2}{d^3}.
\end{align}
Now, consider a value $\rho =  \frac{1}{4 \kappa^2 d}$, when $\Delta^{(t)} \geq \delta_-^{(t)} (1 - \frac{\rho}{2})$, it holds that $(1 - \frac{\rho}{2}) \Delta^{(t)} \geq \delta_-^{(t)}(1 - \rho)$. Therefore, when $\Delta^{(t)} \geq \delta_-^{(t)} (1 - \frac{\rho}{2})$, Eq~\eqref{eq:vnskbdshjafjha} implies that
\begin{align} \label{eq:vsdhfkajhahjf}
|\delta_i^{(t + 1)} |  \geq  |\delta_i^{(t )} |  - \eta \rho \left( |\delta_i^{(t )} |    \right) \frac{\kappa}{d} -  3 \eta C \frac{c_t \kappa^2}{d^3} +\eta \frac{\rho}{2} \Delta^{(t)}  \E_{\cP^{(t)}}[ w_i^2].
\end{align}
Hence, using the assumption that $\E_{\cP^{(t)}}[ w_i^2] \geq \frac{1}{\Kappa d}$, it holds that
\begin{align}\label{eq:fajsiofasifjaji}
|\delta_i^{(t + 1)} |  \geq  |\delta_i^{(t )} |  - \eta \rho \left( |\delta_i^{(t )} |    \right) \frac{\kappa}{d}  =  |\delta_i^{(t )} | \left( 1 - \eta \rho \frac{\kappa}{d} \right).
\end{align}
Summing up all $\delta_i^{(t)}$ with $\delta_i^{(t)} \leq 0$, this implies that
\begin{align*}
\Delta_-^{(t  + 1)} &\geq \Delta_-^{(t)} \left( 1 -  \eta \frac{\rho \kappa}{d} \right).
\end{align*}
Combine the above inequality with inequality~\eqref{eq:Baojisfajasfsoifjas}, we have that (using $\Delta^{(t)} \geq 0$ so that $ \Delta_+^{(t )} \geq  \Delta_-^{(t )}$):
\begin{align*}
	\Delta^{(t + 1)} =  \Delta_+^{(t  + 1)}  - \Delta_-^{(t  + 1)} &\leq \Delta_+^{(t  )}  - \Delta_-^{(t  )} - \eta \frac{1}{\kappa d} \Delta_+^{(t )}  + \eta \frac{\rho \kappa}{d } \Delta_-^{(t )}   - \eta \frac{1}{2 \kappa d} \Delta^{(t)} | \set{S}^{(t)}_+| \\
		& \leq  \Delta^{(t)}- \eta \frac{1}{2 \kappa d}  \Delta_+^{(t )}   -  \eta\frac{1}{2 \kappa d} \Delta^{(t)} | \set{S}^{(t)}_+|.
\end{align*}
Therefore we conclude that when $ \Delta^{(t)} \geq \Omega\left( \frac{c_t \kappa^8}{d^2} \right)$ and  $\Delta^{(t)} \geq \delta_-^{(t)} (1 - \frac{\rho}{2})$, it must holds that
\begin{align*}
	\Delta^{(t + 1)}&\leq \Delta^{(t)} \left( 1 - \eta \frac{1}{2 \kappa d} \right) \\
  \delta_-^{(t + 1)} &\geq \delta_-^{(t + 1)}  \left( 1 - \eta \frac{1}{4 \kappa d} \right).
\end{align*}
Here the second inequality comes from Eq~\eqref{eq:fajsiofasifjaji}.
This implies that $\Delta^{(t + 1)} $ will decrease faster than $ \delta_-^{(t + 1)} $ at the next iteration. Hence,   when  $ \Delta^{(t)} \geq \Omega\left( \frac{c_t \kappa^8}{d^2} \right)$, then $\Delta^{(t)} \geq \delta_-^{(t)} (1 - \frac{\rho}{2})$ can never happen. Hence, by our choice of $\rho$, we conclude that  as long as  $ \Delta^{(t)} \geq \Omega\left( \frac{c_t \kappa^8}{d^2} \right)$, then
\begin{align*}
	\Delta^{(t)} \leq \delta_-^{(t)} \left(1 -\eta  \frac{1}{16 \kappa^2 d} \right).
\end{align*}
On the other hand, even when $\Delta^{(t)} \leq \delta_-^{(t)} (1 - \frac{\rho}{2})$ but $ \Delta^{(t)} = \Omega\left(\frac{\Kappa^8 c_t}{d^2} \right)$ , we still have that for every $i $ with $\delta_i^{(t)} \leq 0$, by Eq~\eqref{eq:vsdhfkajhahjf}:
\begin{align*}
	|\delta_i^{(t + 1)} |  &\geq  |\delta_i^{(t )} |  - \eta \left( |\delta_i^{(t )} |  -  \Delta^{(t)}  \right) \E_{\cP^{(t)}}[ w_i^2] -  3 \eta C \frac{c_t \kappa^2}{d^3}
	\\
	& \geq |\delta_i^{(t )} |  \left( 1 - \eta \frac{2 \kappa}{d } \right).
\end{align*}
Hence, as long as $ \Delta^{(t)} = \Omega\left(\frac{\Kappa^8 c_t}{d^2} \right)$, we will always have
\begin{align*}
	\Delta_-^{(t  + 1)} &\geq \Delta_-^{(t)} \left( 1 -  \eta \frac{2 \kappa}{d} \right).
\end{align*}
Combining the above with equation~\eqref{eq:Baojisfajasfsoifjas},  we have that
\begin{align*}
	\Delta^{(t + 1)} =  \Delta_+^{(t  + 1)}  - \Delta_-^{(t  + 1)} &\leq \Delta_+^{(t  )}  - \Delta_-^{(t  )} - \eta \frac{1}{\kappa d} \Delta_+^{(t )}  + \eta \frac{2 \kappa}{d } \Delta_-^{(t )}   -\eta \frac{1}{2 \kappa d} \Delta^{(t)} | \set{S}^{(t)}_+| \\
	& \leq  \Delta^{(t)} + \eta \frac{2 \kappa}{ d} \delta_+^{(t)}   | \set{S}^{(t)}_+| - \eta \frac{1}{2 \kappa d} \Delta^{(t)} | \set{S}^{(t)}_+|.
\end{align*}
Here, we are using the fact that $\Delta^{(t)}_- \leq \Delta^{(t)}_+ \leq  \delta_+^{(t)}   | \set{S}^{(t)}_+|$.
Now, this implies that when $ \Delta^{(t)}  \geq  \max \left\{ 8 \kappa^2  \delta_+^{(t)},  \Omega\left(\frac{\Kappa^8 c_t}{d^2} \right) \right\}$ , it also holds that
\begin{align*}
	\Delta^{(t + 1)} \leq  \Delta^{(t)} - \eta \frac{1}{4 \kappa d} \Delta^{(t)} | \set{S}^{(t)}_+|.
\end{align*}

\paragraph{Case 2.} In the second case $ \Delta^{(t)} \leq - \Omega \left(  \frac{c_t \Kappa^8}{d^2} \right)$, we shall use the similar proof, but with the assumption that $\E_{\cP^{(t)}}[w_i^2] \in \left[  \frac{1}{\Kappa d}, \frac{2 \kappa }{d} \right]$, it holds that when $|\Delta^{(t)} |\geq \delta_+^{(t)} (1 - \rho)$, for $\rho = \frac{1}{10 \Kappa^5 d} $. Therefore, we can also conclude:
	\begin{align*}
		- \Delta^{(t + 1)}  \leq  - \Delta^{(t )}  - \eta \frac{1}{5 \Kappa^5 d}  \Delta_-^{(t )}.
	\end{align*}
	The proof follows by a similar argument to Case 1.
\end{proof}

Based on the above result, next we describe the dynamic of the 2nd order tensor.
\begin{proposition}[Learning 2nd order tensor] \label{lem:zero_two2}
	In the setting of Lemma \ref{lem:zero_two3}, suppose Proposition \ref{def_H0} holds.
	Assume that for every $1\le i\le d$,
	$\E_{w^{(t)} \sim \cP^{(t)}}[{w^{(t)}}_i^2] \ge \frac{1}{\Kappa d}$.
 For every $t \le T_2$, we have that
 \begin{align*}
 \delta_+^{(t + 1)} &\leq \delta_+^{(t)} \left(1 -  \eta \frac{1}{ \poly(\Kappa) d} \right) +  \eta\frac{c_t \poly(\Kappa)}{d^3},
 \\
 \delta_-^{(t + 1)}  & \leq \delta_-^{(t)}  \left(1 - \eta \frac{1}{ \poly(\Kappa) d} \right) + \eta \frac{c_t \poly(\Kappa)}{d^3}.
 \end{align*}
 Moreover, when $\Delta^{(t)} > 0$, we have the following improved bound for $\delta_+^{(t+1)}$:
 \begin{align*}
 \delta_+^{(t + 1)} \leq \delta_+^{(t)} \left(1 - \eta \frac{1}{4 \kappa d} \right) +\eta  \frac{c_t \poly(\Kappa)}{d^3}.
 \end{align*}
\end{proposition}
\begin{proof}%
 By the update rule, we can obtain (in Eq~\eqref{eq:vnsklfafkajf}) that
 \begin{align*}
\delta_i^{(t + 1)} &= \delta_i^{(t )} - 2\eta \left( \Delta^{(t)} + \delta_i^{(t)}  \right)  \E_{\cP^{(t)}} [w_i^2] \pm 3 \eta C \frac{c_t \kappa^2}{d^3}.
\end{align*}
Let us first consider the case when $\Delta^{(t)} > 0$, then, for $\delta_i^{(t )} \geq 0$, it holds that $ \E_{\cP^{(t)}} [w_i^2] \geq a_i \geq \frac{1}{\kappa}$, thus,
 \begin{align*}
\delta_+^{(t + 1)} & \leq  \delta_+^{(t )} \left( 1 - \eta \frac{1}{\kappa d} \right) +  3 \eta C \frac{c_t \kappa^2}{d^3}.
\end{align*}
Now, for $\delta_-^{(t + 1)}$, consider two cases: $\delta_-^{(t )} \geq \frac{c_t \poly(\Kappa)}{d^2}$ or $\delta_-^{(t )} \leq \frac{c_t \poly(\Kappa)}{d^2}$.
In the first case, we have that when  $| \Delta^{(t )} | > \left( 1 -   \frac{1}{5 \Kappa^5 d} \right) \delta_-^{(t )}$, it must be that $| \Delta^{(t )} | \geq  \frac{c_t \poly(\Kappa)}{d^2}$. By Proposition~\ref{lem:zero_two}, this can not happen. Therefore, we must have  $| \Delta^{(t )} | \leq \left( 1 -   \frac{1}{5 \Kappa^5 d} \right) \delta_-^{(t )}$. Hence, using $ \E_{\cP}[w_i^2]  \geq \frac{1}{\Kappa d}$, we have that
\begin{align*}
 \delta_-^{(t + 1)} &   \leq  \delta_-^{(t )}  - \eta\left( \delta_-^{(t)}  - |\Delta^{(t)}  |\right) \frac{1}{\Kappa d} +   3 \eta C \frac{c_t \kappa^2}{d^3}
 \\
 & \leq   \delta_-^{(t )} - \eta   \delta_-^{(t )} \frac{1}{5 \Kappa^9 d} +   3 \eta C \frac{c_t \kappa^2}{d^3},
\end{align*}
which proves the condition.  On the other hand when $\delta_-^{(t )} \leq \frac{c_t \poly(\Kappa)}{d^2}$, we directly completes the proof by choosing a larger poly in $ \frac{c_t \poly(\Kappa)}{d^3}$.
We can apply the same argument for $\delta_+$, and the improved bound for the case when $\Delta^{(t)} > 0$.
\end{proof}

\subsubsection{Proof of the Main Lemma}

Now we are ready to show the final convergence lemma.
We first provide the following claim that shows on average, each coordinate of the neuron distribution lies in a bounded range.
This also proves the first equation of \eqref{eq_H0_2} in the inductive hypothesis $\cH_1$.
\begin{claim}\label{lem:zero_two4}
  In the setting of Lemma \ref{lem:zero_two3},
	for every $t \leq T_2' = \Theta(\frac{d^2}{\eta c_{T_2} \exp(poly(\kappa))})$  and every $i \in [d]$, we have that
	 $\E_{w^{(t)} \sim \cP^{(t)}}[{w^{(t)}}_i^2] \in \left[  \frac{1}{\Kappa d}, \frac{2 \kappa }{d} \right]$.
\end{claim}
\begin{proof}%
  The upper bound follows from $\delta_+^{(t)} \leq \delta_+^{(0)} \leq \frac{\kappa}{d}$, so $\E_{\cP^{(t)}}[w_i^2] \leq a_i +\delta_+^{(t)} \leq  \frac{2 \kappa}{d}$.
	For the lower bound,  we will prove it by induction. Let us assume that the lower bounds hold for $t \leq T_0$ for some $T_0 < T_2$.
  For $t = T_0 + 1$, denote $\set{T}$ the iterations $t \le T_0 + 1$ where $\Delta^{(t)} > 0$, and $\set{T}^c$ be the other iterations.
  To show the lower bound, for $i \in [d]$, we know that when $\E_{\cP^{(t)}}[w_i^2]  \leq \frac{1}{2 \kappa d}$, by the update rule in Eq~\eqref{eq:aofahfoiashoafhaso}, we can conclude that:
	\begin{align}
		t \in \set{T}^c &\implies \E_{\cP^{(t + 1)}}[w_i^2]  \geq \E_{\cP^{(t)}}[w_i^2], \\
		t \in \set{T} &\implies \E_{\cP^{(t + 1)}}[w_i^2]  \geq \E_{\cP^{(t)}}[w_i^2] \left( 1  - \eta \Delta^{(t)} \right). \label{eq_A8_1}
	\end{align}
  Hence, we only need to consider  $t \in \set{T}$, for these iterations, by Proposition~\ref{lem:zero_two2} we know that
   \begin{align*}
 \delta_+^{(t + 1)} \leq \delta_+^{(t)} \left(1 - \eta  \frac{1}{4 \kappa d} \right) + \eta \frac{c_t \poly(\Kappa)}{d^3}.
 \end{align*}
	On the other hand by Proposition~\ref{lem:zero_two}, we have that  when $ \Delta^{(t)}  \geq \max \left\{8 \kappa^2  \delta_+^{(t)},  \frac{c_t \poly(\Kappa)}{d^2} \right\}$ , it holds that:
\begin{align*}
\Delta^{(t + 1)} \leq  \Delta^{(t)} - \eta \frac{1}{4 \kappa d} \Delta^{(t)} | \set{S}^{(t)}_+|   \leq \Delta^{(t)} - \eta \frac{1 }{4 \kappa d} \Delta^{(t)}.
\end{align*}
  Now, let us define $\gamma_0 = \delta_+^{0}  \leq \frac{2 \kappa}{d}$, with $\gamma_{t + 1} = \gamma_t  \left(1 - \eta  \frac{1}{4 \kappa d} \right) + \eta \frac{c_t \poly(\Kappa)}{d^3} $ for every $t \in \set{T}$ and $\gamma_{t + 1} = \gamma_t  + \eta \frac{c_t \poly(\Kappa)}{d^3}$ otherwise.  We know that as long as  $ \Delta^{(t)}  \geq \frac{c_t \poly(\Kappa)}{d^2} $, we have:
 \begin{align*}
 \Delta^{(t)}  \leq 8 \kappa^2 \gamma_t.
 \end{align*}
 This implies that
 \begin{align}
		\eta  \sum_{t \in \set{T}} \Delta^{(t)} & \leq \eta \sum_{t \in \set{T}} \left( 8 \kappa^2 \gamma_t + \frac{c_t \poly(\kappa_1)}{d^2} \right) \leq  \eta \frac{c_{T_2} \poly(\Kappa)}{d^2} \times T_2 +   \eta 8 \kappa^2 \sum_{t \in \set{T}}  \gamma_t \nonumber \\
		& \leq 1 +    8 \kappa^2 \left( 4 \kappa d \right) \left( \gamma_0  +  \eta \frac{c_{T_2} \poly(\Kappa)}{d^3} \times T_2   \right) \leq 2 + 64 \kappa^4 \label{eq:fajisofjasfiasjasfo}
 \end{align}
 Using equation~\eqref{eq_A8_1},  we have that
 \begin{align*}
	  \E_{\cP^{(T_0 + 1 )}}[w_i^2]  \geq   \E_{\cP^{(0 )}}[w_i^2]  \exp \left\{ - \eta \sum_{s \in \set{T}}\Delta^{(s)} - \eta^2 \poly(1/\lambda_0) T_2 \right\}.
 \end{align*}
 Together with Eq~\eqref{eq:fajisofjasfiasjasfo} we conclude that $\E_{\cP^{(T_0 + 1)}}[w_i^2] \geq \frac{1}{\Kappa d}$. By induction we complete the proof.
\end{proof}

Based on Claim \ref{lem:zero_two4}, we prove Lemma \ref{lem:zero_two3}.
\begin{proof}[Proof of Lemma~\ref{lem:zero_two3}]
	Clearly, by Proposition~\ref{lem:zero_two2} once $ \delta_+^{(t)}$ or $ \delta_-^{(t)} \leq  \frac{c_t \poly(\Kappa)}{d^2} $, they will stay within the interval for the next iterations. By Proposition~\ref{lem:zero_two}, after both $ \delta_+^{(t)}$ and $ \delta_-^{(t)} \leq  \frac{c_t \poly(\Kappa)}{d^2} $, we know $\Delta^{(t)}$ will be within the interval as well.

	Hence, we just need to consider the first time that $ \delta_+^{(t)}$ and $ \delta_-^{(t)}$ goes outside the interval.
	Following Proposition~\ref{lem:zero_two2}, we know that when $ \delta_+^{(t)} \geq \frac{c_t \poly(\Kappa)}{d^2} $, it holds:
	\begin{align*}
    \delta_+^{(t + 1)} &\leq \delta_+^{(t)} \left(1 -\eta  \frac{1}{ \poly(\Kappa) d} \right),
  \end{align*}
  which gives the convergence error rate of $\delta_+$ after $T_1$ iterations. The same holds for $  \delta_-^{(t)}$.
  \end{proof}

Finally, we have an estimate of how big each coordinate is for the neurons at the end of Stage 1.1, which can be given by the output layer weights $\{a_i\}_{i=1}^d$.
We show the following claim, which will be used in the proof of Stage 2.
\begin{claim}[The end of Stage 1.1]\label{lem:zero_two5}
	In the setting of Lemma~\ref{lem:zero_two3}, at iteration $T_1$ (recalling that $T_1 = \Theta(\frac{\poly(\kappa_1)d\log d}{\eta})$), for every  $v \in \set{S}_g$ and every $i \in [d]$, we have that
 \begin{align*}
		v_i^{(T_1)} &= (a_id) v_i^{(0)}  \pm \frac{\poly(\log d)}{d^{3/2}}.
 \end{align*}
 \end{claim}
\begin{proof}%
 Let us first show the upper bound. For every $v^{(0)} \in \set{S}_g$. By the update rule, we have that
  \begin{align*}
 v_i^{(t + 1)} = v_i^{(t)} - \eta (\Delta^{(t)} + \delta_i^{(t)} ) v_i^{(t)} - \eta [\nabla_{\geq 4, v}^{(t)}]_i.
 \end{align*}
Using Proposition~\ref{lem:four_plus_upper}, we have that:
\begin{align*}
 \left(v_i^{(t + 1)} \right)^2 & \leq  \left(v_i^{(t)} \right)^2 - 2 \eta (\Delta^{(t)} + \delta_i^{(t)} ) \left( v_i^{(t)} \right)^2 + \eta [\nabla_{\geq 4, v}^{(t)}]_i v_i^{(t)} + \eta^2 O\left( \frac{1}{\lambda_0^2} \right)
 \\
 & \leq \left(v_i^{(t)} \right)^2 \left( 1 - 2 \eta   (\Delta^{(t)} + \delta_i^{(t)} ) \right) + \eta  \frac{C}{4} \left( \frac{c_t \kappa }{d^2} + \frac{\kappa}{d} \| v^{(t)}\|_{\infty} \|\bar{v}^{(t)} \|_{\infty}  \right) |v_i^{(t)}|^2  + \eta^2 O\left( \frac{1}{\lambda_0^2} \right).
 \end{align*}
 On the other hand, we have that by Eq~\eqref{eq:aofahfoiashoafhaso}, it holds:
	\begin{align*}%
		\E_{\cP^{(t + 1)}} [w_i^2] \geq (1 -  2\eta (\Delta^{(t)} + \delta_i^{(t)} ) )\E_{\cP^{(t)}} [w_i^2] - \left( 3\eta C \frac{c_t \kappa^2}{d^3} \right)
	\end{align*}
Using Claim~\ref{lem:zero_two4}, we have that $ \E_{\cP^{(t)}}[w_i^2] \geq \frac{1}{\Kappa d}$, which implies
 \begin{align*}
		\frac{ (v^{(t + 1)}_i )^2 }{\E_{\cP^{(t + 1)}}[ w_i^2]} \leq \frac{ (v^{(t )}_i )^2 }{\E_{\cP^{(t)}}[ w_i^2]} + \eta \left( 10   C \frac{c_t \kappa^2 \Kappa }{d^2} +  C\left( \frac{c_t \kappa \Kappa }{d} + \kappa \Kappa \| v^{(t)}\|_{\infty} \|\bar{v}^{(t)} \|_{\infty}  \right) |v_i^{(t)}|^2    \right).
 \end{align*}
 Similarly, we also have that
\begin{align*}
	\frac{ (v^{(t + 1)}_i )^2 }{\E_{\cP^{(t + 1)}}[ w_i^2]} \geq \frac{ (v^{(t )}_i )^2 }{\E_{\cP^{(t)}}[ w_i^2]} - \eta \left( 10   C \frac{c_t \kappa^2 \Kappa }{d^2} +  C\left( \frac{c_t \kappa \Kappa }{d} + \kappa \Kappa \| v^{(t)}\|_{\infty} \|\bar{v}^{(t)} \|_{\infty}  \right) |v_i^{(t)}|^2    \right).
\end{align*}
Using these two inequalities, with the conclude $\E_{\cP^{(t)}}[w_i^2] \in \left[  \frac{1}{\Kappa d}, \frac{2 \kappa }{d} \right]$ in Claim~\ref{lem:zero_two4} and the assumption about initialization of $\set{S}_g$, which says that for every $v \in \set{S}_g$, $\| \bar{v}^{(0)} \|_{\infty}, \| v^{(0)} \|_{\infty} \leq \frac{\poly(\log d)}{d}$, we can conclude that for every $v^{(0)} \in \set{S}_g$, we have for every $t \in [T_1]$,
\begin{align}\label{eq:Bnoaidhfaoiuhas}
	\frac{ (v^{(t )}_i )^2 }{\E_{\cP^{(t)}}[ w_i^2]} = \frac{ (v^{(0)}_i )^2 }{\E_{\cP_{0 }}[ w_i^2]} \pm \frac{\poly(\log d)}{d}.
 \end{align}
Note that $v^{(t )}_i$ do not change sign during the gradient process (otherwise $v^{(t )}_i$ will be close to zero, violating the above inequality). Since $\E_{\cP_{0 }}[ w_i^2] = \frac{1}{d}$ and by Lemma~\ref{lem:zero_two3}, $\left| \E_{\cP^{(T_1)}}[ w_i^2] - a_i \right| \leq  \frac{c_t \poly(\Kappa)}{d^2}$.
This implies that
\begin{align*}
	v^{(T_1)}_i  = (a_id) v_i^{(0)} \pm \frac{\poly(\log d)}{d^{3/2}},
\end{align*}
which completes the proof.
\end{proof}

\subsection{Stage 1.2: Proof of Convergence for Higher Order Tensors}

In this section, we prove Lemma \ref{lem:stage_1_final}, which shows that by the end of Stage 1, a small fraction of neurons have won the lottery ticket by growing much larger than a typical neuron.
This stage runs for approximately $\Theta_{\kappa}(\frac{d^2}{\eta \log d})$ many iterations (or $T_2 - T_1$ more precisely).
The proof of Lemma \ref{lem:stage_1_final} is organized as follows.
\begin{itemize}
	\item First, in Proposition \ref{lem:four_plus_interval}, we show that the dynamic is mainly determined by the 4th order gradients by bounding the gradients contributed by the 0th and 2nd order losses so that, as described in Section \ref{sec_stage1}. Based on this result, we can relate the dynamic of this substage to tensor power method.
	\item Second, we provide a lower bound on the norm of every neuron in Claim \ref{lem:grow_4}. Based on this result, we prove Claim \ref{claim:grow_2} that shows the growth of good neurons. This leads to the proof of Lemma \ref{lem:stage_1_final} in Section \ref{sec_proof_lemma_A2}.
	\item Finally, we prove the inductive hypothesis $\cH_1$ in Section \ref{app_proof_H0}.
\end{itemize}

We describe the following proposition to bound the gradients of 4th or higher tensors.
\begin{proposition}[Gradient bound for Stage 1.2]\label{lem:four_plus_interval}
	In the setting of Lemma \ref{lem:stage_1_final}, suppose that Proposition \ref{def_H0} holds.
	Consider any iteration $t \in [T_1 + 1, T_2]$ and any neuron $v\in\set{S}_g$.
	Suppose that for every $s \le t$, $\| v^{(s)} \|_{\infty}\leq \frac{\poly(\log d)}{\sqrt{d}} $.
	Then for every $i \in [d]$, the gradient of $v$ at iteration $t$ satisfies
	\begin{align*}
		-[\nabla_{v}]_i &=  B_{1,4}  a_i \langle e_i, v^{(t)} \rangle \langle e_i, \bar{v}^{(t)} \rangle^{ 2}   \pm \frac{c_t \poly(\Kappa)}{d^2} |v_i^{(t)}|.
	\end{align*}
\end{proposition}

\begin{proof}%
The result mainly follows from combining  Proposition~\ref{lem:four_plus_upper} for the gradient coming from 4th and higher order losses with Proposition~\ref{lem:zero_two3} for the gradient of 0th and 2nd order losses.
The only remaining term is
\begin{align*}
 \left|  \left( \sum_{r} a_r\langle e_r, v^{(t)} \rangle \langle e_r, \bar{v}^{(t)} \rangle^{3}  \right)  \bar{v}_i^{(t)} \right| &\leq \frac{\kappa}{d} \left( \sum_{r} v_r^{(t)}( \bar{v}_r^{(t)} )^3 \right) |\bar{v}_i^{(t)}| \leq \frac{\kappa}{d} \frac{\| v^{(t)} \|_4^4}{\| v^{(t)}\|_2^4} |v_i^{(t)}|.
\end{align*}
Hence, we have
\begin{align}\label{eq:fjaoisfasjfia}
- [\nabla_{v^{(t)}}]_i &=  B_{1,4}  a_i \langle e_i, v^{(t)} \rangle \langle e_i, \bar{v}^{(t)} \rangle^{ 2}   \pm \frac{c_t \poly(\Kappa)}{d^2} |v_i^{(t)}| \pm \frac{\kappa}{d} \frac{\| v^{(t)} \|_4^4}{\| v^{(t)}\|_2^4} |v_i^{(t)}|.
\end{align}
By the definition of $\set{S}_g$, we know that at $T_1$ every $v \in \set{S}_g$ satisfies
\begin{align*}
 \| v^{(T_1)}\|_2^2 \geq \frac{1}{2 \kappa} \text{, and for at most $O(\log d)$ many  $i \in [d]$, } |v_i^{(T_1)} |^2 \geq \frac{\kappa \log d}{d}.
\end{align*}
We will maintain the following condition by induction.
\begin{align}\label{eq:Vahoifashfoiashf}
 \| v^{(t)}\|_2^2 \geq \frac{1}{4 \kappa} \text{, and for at most $O(\log d)$ many  $i \in [d]$, } |v_i^{(t)} |^2 \geq \frac{c_t \poly(\Kappa)}{d}.
\end{align}
Now suppose the following is true  at some iteration $t \geq T_1$, then we have that
\begin{align*}
 \frac{\| v^{(t)} \|_4^4}{\| v^{(t)}\|_2^4} \leq \frac{c_t \poly(\Kappa)}{d}.
\end{align*}
Thus,  for iteration $t + 1$, using Eq~\eqref{eq:fjaoisfasjfia} we know that
\begin{align*}
( v_{i}^{(t  + 1)})^2 = ( v_{i}^{(t  )})^2 + \eta B_{1,4}  a_i  |v_i^{(t)}|^2 \langle e_i, \bar{v}^{(t)} \rangle^{ 2}   \pm \frac{c_t \poly(\Kappa)}{d^2} |v_i^{(t)}|^2 \pm \eta^2 O\left( \frac{1}{\lambda_0^2} \right).
\end{align*}
Hence, we have that for every $i$ with $|v^{(t)}_i|^2 \leq \frac{\poly(\Kappa ) c_t}{d}$, it holds that
\begin{align}\label{eq:fajoiashfaifhas}
( v_{i}^{(t  + 1)})^2  =  ( v_{i}^{(t  )})^2 \left( 1 \pm \eta \frac{c_t \poly(\Kappa)}{d^2} \right).
\end{align}
Hence for every $t \leq T_2$,  as long as $|v^{(T_1)}_i|^2 \leq \frac{\poly(\Kappa ) c_t}{d}$, we have:
\begin{align*}
( v_{i}^{(t  + 1)})^2  \in  \left[ \frac{2}{3}( v_{i}^{(T_1 )})^2 ,  \frac{3}{2}( v_{i}^{(T_1 )})^2 \right].
\end{align*}
This proves inequality~\eqref{eq:Vahoifashfoiashf} for $t + 1$.
\end{proof}

Next, we use the following claim to maintain a lower bound on the norm of each neuron.
\begin{claim}[Norm lower bound for Stage 1.2] \label{lem:grow_4}
	In the setting of Lemma \ref{lem:stage_1_final}, suppose Proposition \ref{def_H0} holds.
	For every $v \in \set{S}_g$, the norm of $v$ at any iteration $t \in [T_1 + 1, T_2]$ satisfies $\| v^{(t)} \|_2^2 \ge \Omega \left( \frac{1}{\kappa} \right)$.
\end{claim}
\begin{proof}%
By the update rule, using Proposition~\ref{lem:four_plus_upper} we know that for every $p \in [d]$:
\begin{align}
- [\nabla_{v^{(t)}}]_p  &= \sum_{j \geq 2} \left( B_{1,2j}  a_p \langle e_p, v^{(t)} \rangle \langle e_p, \bar{v}^{(t)} \rangle^{ 2j - 2} - B_{2, 2j} \left( \sum_{r} a_r\langle e_r, v^{(t)} \rangle \langle e_r, \bar{v}^{(t)} \rangle^{2 j - 1}  \right)  \bar{v}_p^{(t)}  \right)
\nonumber\\
& \pm  \frac{c_t \poly(\Kappa)}{d^2}v^{(t)}_p
 \nonumber\\
& =  \sum_{j \geq 2} \left( B_{1,2j}  a_p \frac{(v_p^{(t)})^{2j - 1}}{\| v^{(t)}\|_2^{2j - 2}} - B_{2, 2j} \frac{ \sum_{r} a_r (v_r^{(t)})^{2j} }{\| v^{(t)} \|_2^{2j}} v_p^{(t)} \right) \pm  \frac{c_t \poly(\Kappa)}{d^2} v^{(t)}_p
 \nonumber\\
& = v_p^{(t)} Q_p^{(t)} \pm  \frac{c_t \poly(\Kappa)}{d^2}v^{(t)}_p, \label{eq:nbosdhofaiaho}
\end{align}
where
\begin{align} \label{eq:bhosihfaoashfiah}
Q_p^{(t)} &:=  \sum_{j \geq 2} \left( B_{1,2j}  a_p \frac{(v_p^{(t)})^{2j - 2}}{\| v^{(t)}\|_2^{2j - 2}} - B_{2, 2j} \frac{ \sum_{r} a_r (v_r^{(t)})^{2j} }{\| v^{(t)} \|_2^{2j}}  \right)
 \nonumber\\
&= \sum_{j \geq 2} \frac{1}{\| v^{(t)}\|_2^{2j - 2}} \left( B_{1,2j}  a_p(v_p^{(t)})^{2j - 2} - B_{2, 2j} \frac{ \sum_{r} a_r (v_r^{(t)})^{2j} }{\| v^{(t)} \|_2^{2}}  \right).
\end{align}
We have that
\begin{align}
\sum_{p}  \left(v_p^{(t)} \right)^2 Q_p^{(t)} &= \sum_{p }\sum_{j \geq 2} \frac{1}{\| v^{(t)}\|_2^{2j - 2}} \left( B_{1,2j}  a_p(v_p^{(t)})^{2j - 2} - B_{2, 2j} \frac{ \sum_{r} a_r (v_r^{(t)})^{2j} }{\| v^{(t)} \|_2^{2}}  \right)   \left(v_p^{(t)} \right)^2
 \nonumber\\
&=  \sum_{p }\sum_{j \geq 2} \frac{1}{\| v^{(t)}\|_2^{2j - 2}} \left( B_{1,2j}  a_p(v_p^{(t)})^{2j } - B_{2, 2j} \left( \sum_{r} a_r (v_r^{(t)})^{2j} \right)   \right)
 \nonumber\\
&= \sum_{j \geq 2}\frac{1}{\| v^{(t)}\|_2^{2j - 2}}  (B_{1,2j} - B_{2, 2j}) \left( \sum_{r} a_r (v_r^{(t)})^{2j} \right)  \geq 0 \label{eq:bhoaifhaosifhasif}
\end{align}
This implies the following
\begin{align}\label{eq:vjapvdajfapsfjasp}
\|v^{(t + 1)} \|_2^2 \geq \|v^{(t)} \|_2^2 \left( 1 - \eta \frac{c_t \poly(\Kappa)}{d^2}  \right).
\end{align}
Combined with Proposition~\ref{lem:four_plus_interval} , we have that for every neuron $v$, $\| v^{(t)} \|_2^2 = \Omega\left( \frac{1}{\kappa} \right)$  for every $t \in [T_1, T_2]$.
\end{proof}

\subsubsection{Proof of the Main Lemma}\label{sec_proof_lemma_A2}

Provided with the gradient bound and norm lower bound, we are now ready to prove the main result of Stage 1.2.
Towards showing Lemma \ref{lem:stage_1_final}, we prove the following claim, which shows that if a neuron has grown beyond $\frac{\poly\log(d)} {d}$ at a certain iteration $T_2'$, then this neuron will become basis-like at iteartion $T_2$.

\begin{claim}[Growth of good neurons] \label{claim:grow_2}
	In the setting of Lemma \ref{lem:stage_1_final}, suppose that Proposition \ref{def_H0} holds.
	For every $v\in\set{S}_g$, suppose at iteration $T_2'$ (recalling that $T_2' = T_2 - \frac{d^2}{\eta\poly\log(d)}$), only one coordinate $i \in [d]$ satisfies $|v_i^{(T_2')}| \geq \frac{\log^{10} d}{\sqrt{d}}$ and all the other coordinates satisfies  $|v_j^{(T_2')}| \leq \frac{(\log d)^2}{\sqrt{d}}$, then at iteration $T_2$, we have that %
\begin{align*}
	|v_i^{(T_2)}|^2  &= \Omega \left( \frac{1}{\lambda_0 \poly(d)} \right) \geq \poly(d), \text{ and for any other } j \neq i,  |v_j^{(T_2)}|  \leq \frac{2(\log d)^2}{\sqrt{d}}.
\end{align*}
\end{claim}

In other words, the claim says that for neuron $v$, its $i$-th coordinate at iteration $T_2$, denoted by $|v_i^{(T_2)}|$, will be as large as $\poly(d)$, which implies that this neuron has won the lottery.
We describe the proof of Claim \ref{claim:grow_2}.
\begin{proof}%
We shall prove the claim by doing an induction.
Consider the condition $|v_i^{(t)}| \geq \frac{\poly(\log d)}{\sqrt{d}}$ and all the other coordinates satisfies  $|v_j^{(t)}| \leq \frac{2 (\log d)^2}{\sqrt{d}}$ for  $t  \in [T_2',  T]$.
Suppose it is true up to iteration $t$, consider iteration $t+1$.
When $p = i$, we have that   $a_i(v_i^{(t)})^{2j - 2}  \geq   a_r(v_r^{(t)})^{2j - 2} $ for every $r \not= i$, hence  this implies that (using the fact that $B_{1, 2j} > B_{2, 2j}$ and $B_{1, 4}$ is greater than $B_{2, 4}$ plus a fixed constant):
\begin{align*}
Q_i^{(t)} = \Omega \left( \frac{a_i(v_i^{(t)})^2}{\| v^{(t)}\|_2^{ 2}} \right) &= \Omega \left( \frac{\poly(\log d)}{d^2}  +  \frac{a_i(v_i^{(t)})^2}{\| v^{(t)}\|_2^{ 2}}  \right),
\end{align*}
where $Q_i^{(t)}$ is defined in the proof of Claim~\ref{lem:grow_4}.
With equation~\eqref{eq:bhoaifhaosifhasif}, this implies that
\begin{align} %
(v^{(t + 1)}_i )^2  \geq (v^{(t )}_i )^2 \left( 1  + \eta \Omega \left( \frac{\poly(\log d)}{d^2}  +   \frac{a_i(v_i^{(t)})^2}{\| v^{(t)}\|_2^{ 2}} \right) \right),
\end{align}
which provides a direct the lower bound on $(v^{(t + 1)}_i )^2 $. Now, to show the upper bound of the other coordinates, recall that we have shown $\| v^{(t)} \|_2^2 = \Omega\left( \frac{1}{\kappa} \right)$ for every $t \in [T_1, T_2]$,
\begin{align*}
Q_p^{(t)} &\leq \sum_{j \geq 2} \frac{1}{\| v^{(t)}\|_2^{2j - 2}} \left( B_{1,2j}  a_p(v_p^{(t)})^{2j - 2} \right) = O \left(  \kappa^2 \frac{\log^8 d}{d^2}\right).
\end{align*}
This implies that
\begin{align*}
(v^{(t + 1)}_p )^2 & \leq (v^{(t )}_p )^2 \left( 1  + \eta C^2 O \left( \frac{\log^8 d}{d^2} \right) \right).
\end{align*}
Hence we prove all the other $p \not= i$ satisfies $|v_p^{(t)}| \leq \frac{2 (\log d)^2}{\sqrt{d}}$ as long as $T_2 - T_2' \leq \frac{d^2}{ \eta  \log^9 (d)}$, which complete the induction.
In the end, since $|v_i^{(t)}| \geq \frac{\poly(\log d)}{\sqrt{d}}$ and all the other coordinates satisfies  $|v_j^{(t)}| \leq \frac{2 (\log d)^2}{\sqrt{d}}$ for  every $t  \in [T_2',  T]$, we can further simplify Eq~\eqref{eq:upsahaxjhfiauh} as:
\begin{align} \label{eq:upsahaxjhfiauh}
(v^{(t + 1)}_i )^2  \geq (v^{(t )}_i )^2 \left( 1 +   \frac{\kappa (v_i^{(t)})^2}{d \log^5 d}  \right),
\end{align}
which directly gives us the bound $|v_i^{(T_2)}|^2  = \Omega \left( \frac{1}{\lambda_0 \poly(d)} \right) $ at iteration $T_2$.
\end{proof}

Now we are ready to prove Lemma \ref{lem:stage_1_final}.
We define the union of good neurons as
\begin{align*}
	\set{S}_{good} &= \left\{ v \in \set{S}_g \mid \exists i \in [d], [v^{(0)}]_i^2 \geq \Gamma_i  + \rho \text{ and all other } j: [v^{(0)}]_j^2 < \Gamma_j - \rho  \right\},
\end{align*}
where we recall that $\Gamma_i$ and $\rho$ have been defined before the statement of Lemma \ref{lem:stage_1_final}.
In the proof, we focus on the dynamic of a neuron $v$ until the point that $\| v \|_{\infty} \geq \frac{\poly(\log d)}{\sqrt{d}}$.
The key step is to track the dynamic via a tensor gradient update.

\begin{proof}[Proof of Lemma~\ref{lem:stage_1_final}]
	We focus on proving the following three statements.
	\begin{enumerate}
		\item For every $v \notin \set{S}_{pot} $, $\| v^{(t)} \|_{\infty} \geq \frac{\poly(\log d)}{\sqrt{d}}$ \textcolor{black}{never} happen for any $t \leq T_2$.
		\item  For every $v \in \set{S}_{good} $, $\| v^{(t)} \|_{\infty} \geq \frac{\poly(\log d)}{\sqrt{d}}$  \textcolor{black}{must} happen for some $t \leq T_2'$ and when it happens, the condition in Claim~\ref{claim:grow_2} meets for $i = \argmax_{j \in [d]} \{ v^{(0)}_j \}$.
		\item  For every $v \in  \set{S}_{pot} \backslash \set{S}_{bad} $,  $\| v^{(t)} \|_{\infty} \geq \frac{\poly(\log d)}{\sqrt{d}}$  \textcolor{black}{might} happen for some $t \leq T_2'$. If $\| v^{(t)} \|_{\infty} \geq \frac{\poly(\log d)}{\sqrt{d}}$ happens for some $t \leq T_2$, then the condition in Claim~\ref{claim:grow_2} meets for $i = \argmax_{j \in [d]} \{ v^{(0)}_j \}$.
	\end{enumerate}
	The first and second statement of Lemma \ref{lem:stage_1_final} follow by combining the above three statements and Claim \ref{claim:grow_2}.
	The third statement can be proved by standard anti-concentration inequalities for the Gaussian distribution.
	For the rest of the proof, we focus on proving the above three statements.
	We know by Proposition~\ref{lem:four_plus_interval} that when $\| v^{(t)} \|_{\infty} \leq \frac{\poly(\log d)}{\sqrt{d}}$ the update of $v^{(t)}$ at every iteration $t \in [T_1, T_2]$ is given by
	\begin{align*}
		- [\nabla_{v^{(t)}}]_i &=  B_{1,4}  a_i \langle e_i, v^{(t)} \rangle \langle e_i, \bar{v}^{(t)} \rangle^{ 2}   \pm \frac{c_t \poly(\Kappa)}{d^2} |v_i^{(t)}| \\
		&= B_{1,4} a_i \frac{(v^{(t)}_i)^3 }{ \| v^{(t)} \|_2^2}  \pm \frac{c_t \poly(\Kappa)}{d^2} |v_i^{(t)}|.
	\end{align*}
	For every $i$, consider a process where $p^{(T_1)}, q^{(T_1)} = v_i^{(T_1)}$, with
	\begin{align*}
		p^{(t + 1)} &= p^{(t)} +  \eta p^{(t)} \left(  B_{1,4}  a_i (p^{(t)})^2 +   \frac{c_t \poly(\Kappa)}{d^2}  \right), \\
		q^{(t + 1)} &= q^{(t)} +  \eta q^{(t)} \left(  B_{1,4}  a_i (q^{(t)})^2 -   \frac{c_t \poly(\Kappa)}{d^2}  \right).
\end{align*}
	Along with Eq~\eqref{eq:fjaoifajfaiasja}, we can see that for every $t$ where $\| v^{(t)} \|_{\infty} \leq \frac{\poly(\log d)}{\sqrt{d}}$,
	\begin{align*}
		|p^{(t )}|  &\geq   |v^{(t  )}_i| \times \max \left\{ 1, \frac{1}{\|  v^{(t  )} \|_2} \right\}, \\
		|q^{(t )}|  & \leq   |v^{(t  )}_i| \times \min \left\{ 1, \frac{1}{\|  v^{(t  )} \|_2} \right\}.
	\end{align*}
	To analyze this process, we introduce the following differential equation
	\begin{align*}
		\frac{dx(t)}{dt} = \tau_1 x^3, \quad x(0)^2 = \tau_2.
	\end{align*}
	The solution is given as $x^2(t) = \frac{1}{\frac{1}{\tau_2} - 2 \tau_1 t }$. Therefore, we can easily obtain that as long as $\rho = \Omega\left( \frac{c_t \poly(\Kappa)}{d^2} \right) $, when $\tau_1 = B_{1,4}  a_i$, $\eta 2 \tau_1 T_2' = \frac{1}{\tau_2}$ which implies that $\tau_2 = \frac{1}{\eta 2 \tau_1 T_2'} = \frac{1}{\eta 2(b_4 + b_4') a_i T_2'} $, we have that
	\begin{align*}
		|  v^{(T_1)}_i |^2 \geq \tau_2 + \rho  \implies |q^{(T_2')}| = + \infty.
	\end{align*}
	On the other hand,
	\begin{align*}
		| v^{(T_1)}_i |^2 \leq \tau_2 - \rho  \implies |p^{(T_2')}|^2  = O\left( \frac{\tau_2^2}{\rho} \right) = O\left( \frac{\log^{3} d}{d} \right).
	\end{align*}
	In the end, by Proposition~\ref{lem:zero_two4} and the definition of $\set{S}_g$ (Eq~\eqref{eq:nkvajbkfjaf}), we know that for every $v \in \set{S}_g$ and every $i \in [d]$, we have that
 \begin{align*}
		v_i^{(T_1)} &= (a_id) v_i^{(0)}  \pm \frac{\poly(\log d)}{d^{3/2}}.
 \end{align*}
	Putting into the definition of $\tau_2$ we complete the proof.
\end{proof}

In addition, we state the following claim that will be used in Appendix \ref{app_finite} for the error analysis.
\begin{claim}[Upper bound on gradient norm at the end of Stage 1] \label{lem:grow_5}
	In the setting of Theorem \ref{thm_inf}, at the first iteration $t$ where $\| v^{(t)} \|_2^2  > \frac{1}{\lambda_0}$, i.e. the threshold where gradients are truncated, we have that
	\begin{align*}
		\sum_{s=1}^{t-1} \| \bar{v}^{(s)} \|_{\infty}^2 \leq O \left( \frac{\poly(\Kappa) d \log \frac{1}{\lambda_0}}{\eta} \right).
	\end{align*}
\end{claim}
\begin{proof}%
When $ \| \bar{v}^{(t)} \|_{\infty}^2 \leq \frac{1}{\poly(\Kappa)}$, we have that for $p = \argmax_{r \in [d]} \{  a_r( v_r^{(t)})^2 \}$, the following holds
\begin{align*}
j = 2: \ &\frac{1}{\| v^{(t)}\|_2^{2j - 2}} \left( B_{1,2j}  a_p(v_p^{(t)})^{2j - 2} - B_{2, 2j} \frac{ \sum_{r} a_r (v_r^{(t)})^{2j} }{\| v^{(t)} \|_2^{2}}  \right)
 \\
& = \Omega \left( \frac{a_p(v_p^{(t)})^2}{\| v^{(t)}\|_2^{ 2}} \right) = \Omega \left(\frac{1}{\poly(\Kappa) d} \| \bar{v}^{(t)} \|_{\infty}^2  \right).
 \end{align*}
 \begin{align*}
j \geq 2: \ &\frac{1}{\| v^{(t)}\|_2^{2j - 2}} \left( B_{1,2j}  a_p(v_p^{(t)})^{2j - 2} - B_{2, 2j} \frac{ \sum_{r} a_r (v_r^{(t)})^{2j} }{\| v^{(t)} \|_2^{2}}  \right)
 \\
& = O \left( \frac{a_p(v_p^{(t)})^2}{\| v^{(t)}\|_2^{ 2}} \right) \times \left( \kappa^2 \| \bar{v}^{(t)} \|_{\infty}^{2j - 4} \right).
 \end{align*}
The above implies that as long as $ \| \bar{v}^{(t)} \|_{\infty}^2 \leq \frac{1}{\poly(\Kappa)}$, we have:
\begin{align} \label{eq:bnasfajfoasjoas}
\max\{ a_r( v_r^{(t + 1)})^2 \} \geq \max\{ a_r( v_r^{(t )})^2 \}  \left( 1 +  \eta \Omega \left(  \frac{1}{\poly(\Kappa)d}\| \bar{v}^{(t)} \|_{\infty}^2  \right) \right).
\end{align}
After that, when $ \| \bar{v}^{(t)} \|_{\infty}^2 \geq \frac{1}{\poly(\Kappa)}$, we have that $ \| \bar{v}^{(t)} \|_{\infty}^4 \geq  \frac{1}{\poly(\Kappa)}$ as well, which implies
\begin{align*}
\frac{\sum_r a_r (v^{(t)}_r)^{4}}{\| v^{(t)}\|_2^{2}} \geq \frac{1}{\kappa d} \| \bar{v}^{(t)} \|_{\infty}^4 \geq \frac{1}{\poly(\Kappa) d}.
\end{align*}
Hence, as long as $ \| \bar{v}^{(t)} \|_{\infty}^2 \geq \frac{1}{\poly(\Kappa)}$, Eq~\eqref{eq:bhoaifhaosifhasif} implies that
\begin{align*}
(*): \|v^{(t + 1)} \|_2^2 \geq \|v^{(t)} \|_2^2 \left(  1 +  \eta\frac{1}{\poly(\Kappa) d} \right).
\end{align*}
On the other hand, we also have for every iteration, by Eq~\eqref{eq:vjapvdajfapsfjasp}:
\begin{align} \label{eq:fjaoifajfaiasja}
\|v^{(t + 1)} \|_2^2 \geq \|v^{(t)} \|_2^2 \left( 1 - \eta \frac{c_t \poly(\Kappa)}{d^2}  \right).
\end{align}
The above implies that $(*)$ can only happen for $\frac{\poly(\Kappa) d}{\eta} \log \frac{1}{\lambda_0}$ iterations until the norm of $v$ is too large and gradient clipping happens. For these iterations when  $ \| \bar{v}^{(t)} \|_{\infty}^2 \geq \frac{1}{\poly(\Kappa)}$, we can also easily see that
\begin{align*}
\max\{ a_r( v_r^{(t + 1)})^2 \} \geq \max\{ a_r( v_r^{(t )})^2 \}  \left( 1 - \eta O\left( \frac{ \kappa}{d} \right) \right).
\end{align*}
For all the other iterations when  $ \| \bar{v}^{(t)} \|_{\infty}^2 \leq \frac{1}{\poly(\Kappa)}$, we have Eq~\eqref{eq:bnasfajfoasjoas} holds, which implies that as long as $\| v^{(t - 1)} \|_2^2 \leq \frac{1}{2\lambda_0}$:
\begin{align*}
\eta \sum_{s =1}^{t  - 1} \| \bar{v}^{(s)} \|_{\infty}^2 &\leq \poly(\Kappa) d \log \frac{1}{\lambda_0} + \eta \frac{ \kappa}{d} \times \frac{\poly(\Kappa) d}{\eta} \log \frac{1}{\lambda_0}
\\
& \leq  \poly(\Kappa) d \log \frac{1}{\lambda_0}.
\end{align*}
\end{proof}

\subsubsection{Proof of the Inductive Hypothesis}\label{app_proof_H0}

\paragraph{Verifying the inductive hypothesis $\cH_0$ during Stage 1.}

\begin{proof}[Proof of Proposition~\ref{def_H0}]
	Note that the first part of equation \eqref{eq_H0_2} has been shown in Claim \ref{lem:zero_two4} ---
	the second part can be shown via a similar proof of Claim \ref{lem:zero_two4}.
	For the rest of the proof, we focus on proving equation \eqref{eq:fajosifsajfasjif}.
	The construction of the sequence $\Set{c_t}_{t=1}^{T_2}$ will be shown below.

By inequality~\eqref{eq:Bnoaidhfaoiuhas} in the proof of Claim~\ref{lem:zero_two4}, we know that for every $v^{(0)} \in \set{S}_g$ and $t \leq T_1$, it holds that
\begin{align} \label{eq:fanovidaofaisfh}
	\frac{ (v^{(t )}_i )^2 }{\E_{\cP^{(t)}}[ w_i^2]} = \frac{ (v^{(0)}_i )^2 }{\E_{\cP_{0 }}[ w_i^2]} \pm \frac{\poly(\log d)}{d},
\end{align}
which implies that for every $t \in [T_1]$, $c_t \leq  2\Kappa \kappa c_0$. Now, we focus on $t \in [T_1, T_2]$.
By Lemma~\ref{lem:zero_two3} ,we know that for every  $t \geq T_1$ , we have that
 \begin{align*}
 \Delta^{(t)}, \delta_+^{(t)}, \delta_-^{(t)} \leq  \frac{c_t \poly(\Kappa)}{d^2}
 \end{align*}
By Proposition~\ref{lem:four_plus_upper}, we have that  for every $v^{(0)} \in \set{S}$, $\left| \left[ \nabla_{\geq 4 , v^{(t)}} \right]_i  \right| \leq \frac{C}{2}  \frac{c_t \kappa }{d^2}  |v_i^{(t)}|   $.
Hence,
\begin{align*}
	[ v^{(t + 1)}_i]^2 = [ v^{(t)}_i]^2 \pm  \eta \frac{c_t \poly(\Kappa)}{d^2}  [ v^{(t)}_i]^2.
\end{align*}
This also implies that
\begin{align} \label{eq:fajoifasfjaif}
	\| v^{(t + 1)} \|^2_2 = \| v^{(t)} \|^2_2 \left( 1  \pm  \eta \frac{c_t \poly(\Kappa)}{d^2}  \right).
\end{align}
Hence the above implies that
\begin{align*}
	c_{t + 1} \leq \left(1 +  \eta \frac{c_t \poly(\Kappa)}{d^2}  \right) c_t.
\end{align*}
Iterating the above equation over $t$ gives us the sequence $\Set{c_t}_{t=1}^{T_2}$.
By maintaining that for every $v \in \set{S}$, the norm of $v$ at iteration $t$ satisfies $c_{t} \leq \poly(\Kappa) c_0$ and the fact that $T_2 \leq \frac{d^2}{\eta c_0 \poly(\Kappa)}$,
we have verified the running hypothesis $\set{H}_1$.
\end{proof}

\subsection{Stage 2.1: Obtaining a Warm Start Initialization}\label{app_stage_21}

At the beginning of Stage 2, we reduce the gradient truncation parameter.
This allows the basis-like neurons to continue to grow and we can obtain a warm start initialization at the end of Stage 2.1 in the sense described in Lemma \ref{lem:final_222}.
The proof of Lemma \ref{lem:final_222} consists of the following steps.
\begin{itemize}
	\item First, we analyze the 0th order loss in Claim \ref{claim:zero_order_update_new} and \ref{claim:zero_order_update_new_2}.
	\item Second, We analyze the 2nd order loss in Proposition \ref{prop:lb_beta_gamma}.
	Combined together, we prove Lemma \ref{lem:final_222} in Section \ref{sec_proof_21}.
\end{itemize}

\noindent\textit{Notations for gradients.}
To facilitate the analysis, we introduce several notations on the gradients of a neuron $v$.
We separate the gradient of $v$ into several components at the $t$-th iteration as
$\nabla_{v, 2j} = \nabla_{v, 2j, sig} +  \nabla_{v, 2j , \neg pot} + \nabla_{v, 2j, bad} + \nabla_{v, 2j , pot \backslash bad}$, where each term is given by
{\small\begin{align*}
	\nabla_{2j , v, sig} =& - B_{1, 2j} \left(   \sum_{i} a_i \langle e_i, v \rangle \langle e_i, \bar{v} \rangle^{2j - 2} e_i \right) + b_{2j}' \left(   \sum_{i} a_i\langle e_i, v \rangle \langle e_i, \bar{v} \rangle^{2j - 1}  \right)  \bar{v},
	\\
  \nabla_{v, 2j , \neg pot} =&  B_{1,2j} \left( \E_{w^{(t)} \sim \cP^{(t)}, w \notin \set{S}_{pot}} \langle w^{(t)}, v \rangle \langle \bar{w}^{(t)} , \bar{v} \rangle^{2j - 2} w^{(t)} \right) \\
		&- b_{2j}'\left( \E_{w^{(t)} \sim \cP^{(t)}, w \notin \set{S}_{pot}} \langle w^{(t)}, v \rangle \langle \bar{w}^{(t)} , \bar{v} \rangle^{2j - 2} \langle w^{(t)}, \bar{v} \rangle \right)\bar{v},
 \\
  \nabla_{v, 2j , bad} =&  B_{1,2j} \left( \E_{w^{(t)} \sim \cP^{(t)}, w \in \set{S}_{bad}} \langle w^{(t)}, v \rangle \langle \bar{w}^{(t)} , \bar{v} \rangle^{2j - 2} w^{(t)} \right) \\
   &- b_{2j}'\left( \E_{w^{(t)} \sim \cP^{(t)}, w \in \set{S}_{bad}} \langle w^{(t)}, v \rangle \langle \bar{w}^{(t)} , \bar{v} \rangle^{2j - 2} \langle w^{(t)}, \bar{v} \rangle \right)\bar{v},
  \\
    \nabla_{v, 2j , pot \backslash bad} =&  B_{1,2j} \left( \E_{w^{(t)} \sim \cP^{(t)}, w \in \set{S}_{pot} \backslash \set{S}_{bad}} \langle w^{(t)}, v \rangle \langle \bar{w}^{(t)} , \bar{v} \rangle^{2j - 2} w^{(t)} \right)
    \\
    &- b_{2j}'\left( \E_{w^{(t)} \sim \cP, w \in \set{S}_{pot} \backslash \set{S}_{bad}} \langle w^{(t)}, v \rangle \langle \bar{w}^{(t)} , \bar{v} \rangle^{2j - 2} \langle w^{(t)}, \bar{v} \rangle \right)\bar{v}.
\end{align*}}

\paragraph{Dynamic of 0th order tensor.}
Recall that this substage runs for $T_3 \leq \frac{d \log^{1.01} d}{\eta }$ iterations.
We first focus on the update of the 0th order term $\Delta^{(t)}$.
Let $\Kapppa$ denote $e^{\poly(\Kappa)}$.
We show the following claim.
\begin{claim}[Dynamic of 0th order tensor I] \label{claim:zero_order_update_new}
	In the setting of Lemma \ref{lem:final_222}, suppose that Proposition \ref{def_H1} holds.
	Let $\delta$ be any value in the range $[\frac{\poly(\Kapppa) c_t}{d^2}, \frac{1}{\kappa d}]$.
	When $\Delta^{(t)} \geq  \delta$, for any iteration $t \in [T_2 + 1, T_3]$, we have that
	\begin{align*}
		\E_{w \sim \cP^{(t)}}\|\nabla_{w}L_\infty(\cP^{(t)}) \|_2^2 \geq \Omega \left( {d^2}\delta^4 / {\kappa}  \right).
	\end{align*}
\end{claim}

\begin{proof}
Let us denote $\delta' = \min\{ \delta, \frac{\max\{C_1, C_2\}}{ 10 \kappa d } \}$.
We shall see that when $\Delta^{(t)} \geq \delta$, then for every $i$ with $\beta_i^{(t)} + \gamma_i^{(t)} \geq a_i - \frac{\delta'}{4 \max\{C_1, C_2\}}$, we have that
\begin{align*}
  -\Delta^{(t)} + C_1 (a_i - \beta_i^{(t)} - \gamma_i^{(t)}  ) \leq - \frac{\delta}{2}
\end{align*}
Therefore, using equation~\ref{eq_gradient_approx}, we have that
\begin{align*}
	\E_{w \sim \cP^{(t)}, w \notin \set{S}_{pot}} [\nabla_{w}]_i^2 = \Omega( \beta_i^{(t)} \delta^2).
\end{align*}
On the other hand, when $\gamma_i^{(t)} \geq a_i -  \frac{\delta'}{3 \max\{C_1, C_2\}}$, we have that
\begin{align*}
&  -\Delta^{(t)} + C_1 (a_i - \beta_i^{(t)} - \gamma_i^{(t)}  ) + C_2(a_i - \gamma_i^{(t)}) \leq - \frac{\delta}{6}.
\end{align*}
This implies that
\begin{align*}
	\E_{w \sim \cP^{(t)}, w \in \set{S}_{i, pot}} [\nabla_{w}]_i^2 = \Omega\left(  \gamma_i^{(t)}  \delta^2 \right).
\end{align*}
In either case, we have that as long as  $\beta_i^{(t)} + \gamma_i^{(t)} \geq a_i - \frac{\delta'}{4 \max\{C_1, C_2\}}$, it holds that
\begin{align*}
	\E_{w \sim \cP^{(t)}, w} [\nabla_{w}]_i^2 \geq \Omega\left( \frac{d}{\kappa}  ( \gamma_i^{(t)}  + \beta_i^{(t)})\delta^2 \delta' \right).
\end{align*}
Notice that
\begin{align*}
	\sum_{i \in [d]} [ \beta_i^{(t)} + \gamma_i^{(t)} ] 1_{\beta_i^{(t)} + \gamma_i^{(t)} \leq a_i - \frac{\delta'}{4 \max\{C_1, C_2\}}} \leq 1 - d  \frac{\delta'}{4 \max\{C_1, C_2\}}.
\end{align*}
Using $\Delta^{(t)} \geq 0$, we obtain that
\begin{align*}
	\sum_{i \in [d]} [ \beta_i^{(t)} + \gamma_i^{(t)} ] 1_{\beta_i^{(t)} + \gamma_i^{(t)} \geq a_i - \frac{\delta'}{4 \max\{C_1, C_2\}}} \geq d  \frac{\delta'}{4 \max\{C_1, C_2\}},
\end{align*}
which implies that
\begin{align*}
	\sum_{i \in [d]}\E_{w \sim \cP^{(t)}, w} [\nabla_{w}]_i^2 \geq \Omega\left( \frac{d}{\kappa} (\delta')^2 \delta^2 \right).
\end{align*}
\end{proof}

Next, we focus on the other side when $\Delta^{(t)}$ is negative.
We first show the first lower bound on the neuron mass.
\begin{claim}[Lower bound]~\label{claim:lb_23}
	In the setting of Lemma \ref{lem:final_222}, suppose that Proposition \ref{def_H1} holds.
	Then we have that for any $t \in [T_2+1, T_3]$, the following holds:
	\begin{align*}
		\beta_i^{(t)} + \gamma_i^{(t)} \geq \frac{1}{\poly(\Kappa) d}.
	\end{align*}
\end{claim}
\begin{proof}
	Initially at $t = 0$, we have that $\objinf(\cP^{(0)}) = O\left( \frac{1}{d} \right)$).
	Now, for every $\delta \leq \frac{\min\{ C_1, 1 \}}{100 \kappa d} $, when $ \Delta^{(t)} \leq \delta$, we know that as long as $\beta_i^{(t)} + \gamma_i^{(t)} \leq a_i - \frac{2 \delta}{C_1}$ and $\beta_i^{(t)} + \gamma_i^{(t)} \geq \frac{1}{\poly(d)}$, we also have that
	\begin{align*}
		\beta_i^{(t + 1)} + \gamma_i^{(t + 1)} \geq \beta_i^{(t)} + \gamma_i^{(t)}
	\end{align*}
	Thus, when $\beta_i^{(t )} + \gamma_i^{(t )} \leq \frac{a_i}{2}$, it can  decrease at next iteration $t + 1$ only when $\delta = \Omega \left( \frac{1}{\kappa d} \right)$, in which case, the total decrement is bounded by $\exp \{ - \eta \sum_{t \leq T} |\Delta^{(t)}| 1_{ \Delta^{(t)} \geq \delta} \} $.
	Therefore, taking $\delta =\Theta\left( \frac{1}{\kappa d} \right)$, with the fact that $\beta_i^{(0)} \geq \frac{1}{\kappa d}$, we obtain the result by combining equation~\ref{eq:c11}.
\end{proof}

Based on the above claim, we move on to the case when $\Delta^{(t)}$ is negative.
We show the following proposition.
\begin{claim}[Dynamic of the 0th order update II]\label{claim:zero_order_update_new_2}
	In the setting of Lemma \ref{lem:final_222}, suppose that Proposition \ref{def_H1} holds.
	Let $\delta$ be any value in the range $[\frac{\poly(\Kapppa ) c_t}{d^2}, \frac{1}{\kappa d}]$.
	When $\Delta^{(t)} \leq - \delta$, we have
	\begin{align*}
		\E_{w \sim \cP^{(t)}}\|\nabla_{w}L_{\infty}(\cP^{(t)})\|_2^2 \geq \Omega\left( \frac{\delta^3  d}{\poly(\Kappa) } \right).
	\end{align*}
\end{claim}

\begin{proof}
We shall see that when $\Delta^{(t)} \leq - \delta$, then for every $i$ with $\beta_i^{(t)} + \gamma_i^{(t)} \leq  a_i +  \frac{\delta}{12 \max\{C_1, C_2\}}$, we have that
\begin{align*}
  -\Delta^{(t)} + C_1 (a_i - \beta_i^{(t)} - \gamma_i^{(t)}  ) \geq  \frac{2}{3}  \delta.
\end{align*}
Therefore, we have that
\begin{align*}
	\E_{w \sim \cP^{(t)}, w \notin \set{S}_{pot}} [\nabla_{w}]_i^2 = \Omega( \beta_i^{(t)} \delta^2).
\end{align*}
On the other hand, $ \gamma_i^{(t)} \leq  a_i + \frac{\delta'}{12 \max\{C_1, C_2\}}$ as well, this implies that
\begin{align*}
	- \Delta^{(t)} + C_1 (a_i - \beta_i^{(t)} - \gamma_i^{(t)}  ) + C_2(a_i - \gamma_i^{(t)}) \geq \frac{\delta}{3}.
\end{align*}
This implies that
\begin{align*}
	\E_{w \sim \cP^{(t)}, w \in \set{S}_{i, pot}} [\nabla_{w}]_i^2 = \Omega\left(  \gamma_i^{(t)}  \delta^2 \right).
\end{align*}
Combining both cases, we have that
\begin{align*}
	\E_{w \sim \cP^{(t)}, w} [\nabla_{w}]_i^2 \geq \Omega\left(  ( \gamma_i^{(t)}  + \beta_i^{(t)})\delta^2 \right).
\end{align*}
Notice that $\Delta^{(t)} \leq 0$. This implies that
\begin{align*}
	\sum_{i \in [d]}1_{\beta_i^{(t)} + \gamma_i^{(t)} \geq a_i + \frac{\delta'}{12 \max\{C_1, C_2\}}} = d - \Omega \left( \frac{\min \left\{ \frac{1}{\kappa d} , \delta \right\} d^2}{\kappa } \right).
\end{align*}
Using the Claim~\ref{claim:lb_23}, we have that
\begin{align*}
	\sum_{i \in [d]}1_{\beta_i^{(t)} + \gamma_i^{(t)} \leq a_i + \frac{\delta'}{12 \max\{C_1, C_2\}}} [\beta_i^{(t)} + \gamma_i^{(t)}] \geq \frac{\delta d}{\poly(\Kappa)},
\end{align*}
which implies that
\begin{align*}
	\E_{w \sim \cP^{(t)}, w} [\nabla_{w}]_i^2 \geq \Omega\left( \frac{\delta^2 \min \left\{ \frac{1}{\kappa d} , \delta \right\} d}{\poly(\Kappa) } \right).
\end{align*}
\end{proof}

\begin{claim}\label{claim:delta_21}
	In the setting of Claim \ref{claim:zero_order_update_new} and \ref{claim:zero_order_update_new_2}, for every $T \in [T_2 + 1, T_3]$, the following holds:
	\begin{align*}%
		\eta \sum_{t = T_2 + 1}^{T} |\Delta^{(t)} | = O\left( \left( \frac{\eta T \poly(\Kappa)}{d} \right)^{3/4}  + \left( \frac{\eta T \poly(\Kappa)}{d} \right)^{1/2} \right)
	\end{align*}
	Furthermore, we have that:
	\begin{align}\label{eq:fakvjbdabfaj}
		\eta \sum_{t = T_2 + 1}^{T_3} |\Delta^{(t)} | \leq  (\log d)^{0.8}.
	\end{align}
\end{claim}
\begin{proof}
	To prove the above equation, we consider two scenarios.
	Using Claim \ref{claim:zero_order_update_new}, for every $\delta \in \left[\frac{1}{d^{1.5}}, \frac{1}{\kappa d}\right]$, we have:
	\begin{align}
		\eta \sum_{t = T_2 + 1}^{T} |\Delta^{(t)}| 1_{ \Delta^{(t)} \geq \delta}  = O\left( \frac{\kappa^3}{d^3 \delta^3 } \right). \label{eq:c11}
	\end{align}
	Using Claim \ref{claim:zero_order_update_new_2}, we have:
	\begin{align}
		&\eta \sum_{t = T_2 + 1}^{T} |\Delta^{(t)}| 1_{ \Delta^{(t)} < -\delta}  = O\left( \frac{ \poly(\Kappa)}{d^3 \delta^3 } \right). \label{eq_claim_0th_II_cor1}
	\end{align}
	Combined together, using the fact that $T_3 \leq \frac{d \log^{1.01} d}{\eta }$, we obtain equation \eqref{eq:fakvjbdabfaj}.
\end{proof}

\paragraph{The masses of potential neurons and bad neurons.}
We now focus on the update of $\gamma_i^{(t)}$ and $\beta_i^{(t)}$.
We first prove the following claim on the dynamic of the potential, good, and bad neurons.
Based on the claim, we can obtain the update of $\gamma_i^{(t)}$ and $\beta_i^{(t)}$.
Let $\Kapppa$ denote $e^{\poly(\Kappa)}$.
For any $i \in [d]$, let $\hat{\gamma}_i^{(t)} = \E_{w^{(t)} \sim \cP^{(t)}, w \in  \set{S}_{i, good} } {w_i^{(t)}}^2$.

\begin{claim}\label{claim:grad_21}
	In the setting of Lemma \ref{lem:final_222}, suppose Proposition \ref{def_H1} holds.
  There exists fixed constants $C_{1}, C_2 > 0$ such that for any $t \in [T_2 +1, T_3]$ and any $i\in [d]$, the update of $\hat{\gamma_i}^{(t)}, \beta_i^{(t)}$ satisfies that
	\begin{align*}
			\hat{\gamma}_i^{(t + 1)} &= \left( 1 - \eta \Delta^{(t)} + \eta C_1 (a_i - \beta_i^{(t)} - \gamma_i^{(t)}  ) + \eta C_2 (a_i - \gamma_i^{(t)})    \pm \eta \frac{\poly(\Kapppa ) c_t}{d^2}\right) \hat{\gamma}_i^{(t)}, \\
			\beta_i^{(t + 1)} &= \left( 1 - \eta \Delta^{(t)} + \eta C_1 (a_i - \beta_i^{(t)} - \gamma_i^{(t)}  )   \pm \eta \frac{\poly(\Kapppa ) c_t}{d^2} \right) \beta_i^{(t)}.
	\end{align*}
		Moreover, when $\gamma_i^{(t)} \geq \frac{1}{\poly(d)}$, we have that
	\begin{align*}
			{\gamma}_i^{(t + 1)} &= \left( 1 - \eta \Delta^{(t)} + \eta C_1 (a_i - \beta_i^{(t)} - \gamma_i^{(t)}  ) + \eta C_2 (a_i - \gamma_i^{(t)})    \pm \eta \frac{\poly(\Kapppa ) c_t}{d^2}\right) {\gamma}_i^{(t)}.
	\end{align*}
\end{claim}

The above claim implies that the update between the potential neurons and those not in the potential set differs by a multiplicative factor of $a_i - \gamma_i^{(t)}$.
Intuitively, this gap allows us to show that the mass of potential neurons will converge and reduce the value of $a_i - \gamma_i^{(t)}$.
On the other hand, the mass of bad neurons $\beta_i^{(t+1)}$ will remain polynomially small throughout the update, since its increment only scales with $\poly(\Kapppa)c_t / d^2$ every iteration.
We now describe the proof of the above proposition, which is based on a simple claim that bounds the gradient from irrelevant neurons in equation \eqref{eq_gradient_approx}.
\begin{proof}
	We first show the following claim.
	For every $v \in \set{S}_g$, every $i \in [d]$:
	\begin{align} \label{eq_gradient_approx}
		\left| \left[  \nabla_{v, 2j, sig} +  \nabla_{v, 2j , \neg pot} + \nabla_{v, 2j, bad} + \nabla_{v, 2j , pot \backslash bad} - \nabla_{2j , v, sd} \right]_i \right| \leq \left( \frac{c_t \poly(\Kapppa) }{d^2}  \right) |v_i|
	\end{align}

	To see that the above claim is true, for $v \notin \set{S}_{pot}$, we can bound $ \nabla_{v, 2j , \neg pot} $ as in Lemma~\ref{lem:four_plus_upper}.
	For $v \in \set{S}_{bad}$, we can bound $\nabla_{v, 2j , bad}$ directly using equation~\eqref{eq:vajoisajfoiasfjisajf}.
	For $v \in \set{S}_{i, pot}$, we notice
	\begin{align*}
		\E_{w \sim \cP, w \in \set{S}_{pot}, \| w\|_2 \leq d^6} \| w\|_2^2 \leq \frac{1}{\poly(d)}
	\end{align*}
	On the other hand, when $v \in \set{S}_{i, pot}$ and $\| v\|_2 \geq d^6$, we have that $\|\bar{v} - e_i\|_2 \leq \frac{1}{d^4}$ by Eq~\eqref{eq:bnodsifahsoia}.
	This implies that for
	\begin{align*}
		\nabla_{2j , v, sd} :=  - \left( b_{2j} + b_{2j}'\right) \left(   \sum_{i} (a_i - \gamma_i) \langle e_i, v \rangle \langle e_i, \bar{v} \rangle^{2j - 2} e_i \right) + b_{2j}' \left(   \sum_{i} (a_i - \gamma_i)\langle e_i, v \rangle \langle e_i, \bar{v} \rangle^{2j - 1}  \right)  \bar{v}.
	\end{align*}

	By plugging in the claim in the beginning of the proof into the gradient update rule, we can prove the update rules for each set of neurons.
		For every $v \notin \set{S}_{pot}$ and every $i \in [d]$, we have that
		\begin{align*}
			v_i^{(t + 1)} = \left( 1 - \frac{\eta}{2} \Delta^{(t)} + \eta \frac{C_1}{2}(a_i - \beta_i^{(t)} - \gamma_i^{(t)}) \pm \eta \frac{\poly(\Kapppa ) c_t}{d^2}   \right) v_i^{(t)}.
		\end{align*}
	For every $v \in \set{S}_{i, pot}$ with $| v_i | \geq d^6$, we have that
		\begin{align}
			v_i^{(t + 1)} = \left( 1 - \frac{\eta}{2} \Delta^{(t)} + \eta \frac{C_1}{2}(a_i - \beta_i^{(t)} - \gamma_i^{(t)})   + \eta C_2 (a_i - \gamma_i^{(t)}) \pm \eta \frac{\poly(\Kapppa ) c_t}{d^2}   \right) v_i^{(t)}.
			\label{eq:cor_upppp}
		\end{align}
		For all other coordinates $j \not= i$,
		\begin{align*}
			v_j^{(t + 1)} = \left( 1 - \frac{\eta}{2} \Delta^{(t)} + \eta \frac{C_1}{2}(a_j - \beta_j^{(t)} - \gamma_j^{(t)}) \pm \eta \frac{\poly(\Kapppa ) c_t}{d^2}   \right) v_j^{(t)}.
		\end{align*}
	By applying the above results on each set of neurons, we obtain the result of this claim.
\end{proof}

\subsubsection{Proof of the Main Lemma}\label{sec_proof_21}

We are now ready to prove Lemma \ref{lem:final_222}.
Based on the dynamic of 0th order tensor and the update of the 2nd order losses shown above, we prove the following proposition that shows $\beta_i^{(t)} + \gamma_i^{(t)}$ cannot be too far away from $a_i$ for too many iterations.
\begin{proposition}[Bounds on $\beta_i^{(t)} + \gamma_i^{(t)}$ during Stage 2.1] \label{prop:lb_beta_gamma}
	In the setting of Lemma \ref{lem:final_222}, suppose that Proposition \ref{def_H1} holds, then for every $t \in [T_2+1, T_3]$, we have that
	\begin{align}\label{eq:fjaijasofsajffjs}
		\eta \sum_{s = T_2+1}^t |a_i - \beta_i^{(s)} - \gamma_i^{(s)} | \leq  (\log d)^{0.9}.
	\end{align}
	Moreover, for every $\delta \leq \frac{1}{100 \kappa d}$, we have:
	\begin{align} \label{eq:fajfoiasjfaiofjasj}
		\eta \sum_{s = T_2+1}^{t} \indi{\beta_i^{(s)} + \gamma_i^{(s)}  \geq a_i + \delta} |a_i - \beta_i^{(s)} - \gamma_i^{(s)} |  = O\left( \frac{ \poly(\Kapppa)}{d^3 \delta^3 } \right).
	\end{align}
\end{proposition}

\begin{proof}%
Let us construct an auxiliary function
\begin{align*}
	\Phi^{(t)} &= C_1 \left( a_i - \beta_i^{(t)} - \gamma_i^{(t)} \right)^2 + C_2 \left( a_i - \gamma_i^{(t)} \right)^2.
\end{align*}
We consider an update step, then it holds that as long as $\beta_i^{(t)}+ \gamma_i^{(t)} \leq \frac{\poly(\Kapppa)}{d}$, using Claim~\ref{claim:grad_21}, the update of $\Phi^{(t)}$ is given as: %
\begin{align*}
	\Phi^{(t + 1)} =& \Phi^{(t)} - 2 \eta \left(  C_1^2\left( a_i - \beta_i^{(t)} - \gamma_i^{(t)} \right)^2 \beta^{(t)} + \left( C_1(a_i - \beta_i^{(t)} - \gamma_i^{(t)}) + C_2(a_i - \gamma_i^{(t)} ) \right)^2 \gamma^{(t)}\right) \\
	& + 2 \eta \Delta^{(t)} \left( C_1 \left( a_i - \beta_i^{(t)} - \gamma_i^{(t)} \right) (\beta_i^{(t)} + \gamma_i^{(t)}) +  C_2  \left( a_i  - \gamma_i^{(t)} \right) ( \gamma_i^{(t)})\right) \pm \eta \frac{c_t \poly(\Kapppa)}{d^4}.
\end{align*}
This also implies that
\begin{align*}
	\Phi^{(t + 1)} \leq&  \Phi^{(t)}- 2 \eta \left(  C_1^2\left( a_i - \beta_i^{(t)} - \gamma_i^{(t)} \right)^2 \beta^{(t)} + \left( C_1(a_i - \beta_i^{(t)} - \gamma_i^{(t)}) + C_2(a_i - \gamma_i^{(t)} ) \right)^2 \gamma^{(t)}\right) \\
	&+ \eta | \Delta^{(t)}| \frac{\poly(\Kapppa)}{d^2} + \eta \frac{c_t \poly(\Kapppa)}{d^4},
\end{align*}
which implies that for every $t \geq 0$:
\begin{align} \label{eq:asflajfoasijfios}
	\Phi^{(t + 1)} &\leq  \Phi^{(t)} + \eta | \Delta^{(t)}| \frac{\poly(\Kapppa)}{d^2} +\eta \frac{c_t \poly(\Kapppa)}{d^4}.
\end{align}
Hence, consider the case that $\beta_i^{(t)} + \gamma_i^{(t)} = a_i - \rho^{(t)} $ for $\rho^{(t)} \geq 0$, we have that $\gamma_i^{(t)} \leq a_i - \rho^{(t)}$. Hence in addition to Eq~\eqref{eq:asflajfoasijfios}, we also have (using Claim~\ref{claim:lb_23}):
\begin{align*}
	\Phi^{(t + 1)} &\leq  \Phi^{(t)} - \eta \Omega \left( \frac{1}{\poly(\Kappa) d}(\rho^{(t)})^2 \right)+ \eta | \Delta^{(t)}| \frac{\poly(\Kapppa)}{d^2} +\eta \frac{c_t \poly(\Kapppa)}{d^4}.
\end{align*}
Note that originally $ \Phi^{(0)} = O\left( \frac{1}{d^2} \right)$ using the fact that $\beta^{(0)}_i \leq 2 a_i$ and $\gamma_i^{(0)} \leq \frac{1}{\poly(d)}$, with Claim~\ref{claim:delta_21}, we have that for $T \leq \frac{d}{\eta} \log^{1.01} d$:
\begin{align*}
	\eta \sum_{t \leq T} 1_{\beta_i^{(t)} + \gamma_i^{(t)}  \leq a_i} (a_i - \beta_i^{(t)} - \gamma_i^{(t)} )^2 \leq  \frac{1}{d} (\log d)^{0.81}.
\end{align*}
This implies that
\begin{align*}
	\eta \sum_{t \leq T} 1_{\beta_i^{(t)} + \gamma_i^{(t)}  \leq a_i} (a_i - \beta_i^{(t)} - \gamma_i^{(t)} ) \leq  \sqrt{ \eta \frac{1}{d} (\log d)^{0.81} \times T} \leq  \frac{1}{2} (\log d)^{0.9}.
\end{align*}
Similarly, we can see that when $\beta_i^{(t)} + \gamma_i^{(t)} = a_i + \rho^{(t)} $ for $\rho^{(t)} \geq 0$, then either $\beta_i^{(t)} \geq \rho^{(t)} /2 $ or  $\gamma_i^{(t)} \geq a_i - \rho^{(t)}/ 2$. In either case, we have that
\begin{align*}
	\Phi^{(t + 1)} &\leq  \Phi^{(t)} - \eta \Omega \left(  \frac{1}{\poly(\Kapppa) } (\rho^{(t)})^3 \right)+ \eta | \Delta^{(t)}| \frac{\poly(\Kapppa)}{d^2} +\eta\frac{c_t \poly(\Kapppa)}{d^4}.
\end{align*}
Hence we can also show that
\begin{align*}
	\eta \sum_{t \leq T} 1_{\beta_i^{(t)} + \gamma_i^{(t)}  \geq a_i} (a_i - \beta_i^{(t)} - \gamma_i^{(t)} ) \leq \frac{1}{2} (\log d)^{0.9}.
\end{align*}
Eventually, consider for every $\delta \leq \frac{1}{100 \kappa d}$, when $\beta_i^{(t)} + \gamma_i^{(t)} \geq  a_i + \delta$ and $|\Delta^{(t)}| \leq \frac{d^2}{\poly(\Kapppa)} \delta^3$, then we also have
\begin{align*}
	\Phi^{(t + 1)} &\leq  \Phi^{(t)} - \eta \Omega \left(  \frac{1}{\poly(\Kapppa) }\delta^3 \right) +\eta\frac{c_t \poly(\Kapppa)}{d^4}.
\end{align*}
Using equation~\ref{eq_claim_0th_II_cor1}, we obtain that when $T \leq T_3$,
\begin{align*}
	\eta \sum_{t \leq T} 1_{\beta_i^{(t)} + \gamma_i^{(t)}  \geq a_i + \delta} |a_i - \beta_i^{(t)} - \gamma_i^{(t)} |  = O\left( \frac{ \poly(\Kapppa)}{d^3 \delta^3 } \right).
\end{align*}
\end{proof}

Based on the above proposition, we are ready to prove the main Lemma of Stage 2.1, which provides a warm start initialization at a certain iteration $T_3 = \Theta(d\log d / \eta)$.

\begin{proof}[Proof of Lemma \ref{lem:final_222}]
	We first define $T_3$ more precisely.
	We note that initially, for any $i\in[d]$, $\hat{\gamma}_i^{(0)} \leq {1}/{\poly(d)}$ by construction.
	Using equation \eqref{eq:fakvjbdabfaj} and equation \eqref{eq:fjaijasofsajffjs}, by working on $\hat{\gamma}_i^{(t)}$ and noticing that $\hat{\gamma}_i^{(t)} \le \gamma_i^{(t)}$, we have that there exists an iteration $T^{(i)} = O( {d \kappa \log ( \frac{1}{\hat{\gamma}_i^{(0)}} ) }/{\eta} )$ such that at this iteration,
		$\gamma_i^{ T^{(i)} } \geq \frac{1}{10 \kappa d}$.
	We shall fix $T_3$ to be the maximum of $T^{(i)}$ over $i\in[d]$, which is on the order of $\Theta(d\log d /\eta)$.

	Next, similar to the proof of Proposition~\ref{prop:lb_beta_gamma}, we consider the function
\begin{align*}
	\Phi^{(t)} = \max_{i \in [d] } \left\{  C_1 \left( a_i - \beta_i^{(t)} - \gamma_i^{(t)} \right)^2 + C_2 \left( a_i - \gamma_i^{(t)} \right)^2 \right\}.
\end{align*}
	Let $i$ be the coordinate that achieves the maximum for the function above.
	We show that
\begin{align*}
	\Phi^{(t + 1)} &\leq  \Phi^{(t)}- 2 \eta \left(  C_1^2\left( a_i - \beta_i^{(t)} - \gamma_i^{(t)} \right)^2 \beta^{(t)} + \left( C_1(a_i - \beta_i^{(t)} - \gamma_i^{(t)}) + C_2(a_i - \gamma_i^{(t)} ) \right)^2 \gamma^{(t)}\right) \\
	&+ \eta | \Delta^{(t)}| \frac{\poly(\Kapppa)}{d^2} + \eta \frac{c_t \poly(\Kapppa)}{d^4}.
\end{align*}
Let $ \mu =   C_1(a_i - \beta_i^{(t)} - \gamma_i^{(t)}) + C_2(a_i - \gamma_i^{(t)} ) $, $\nu = C_1(a_i - \beta_i^{(t)} - \gamma_i^{(t)} )$, we have that
\begin{align*}
\Phi^{(t)} = (\mu - \nu)^2 + \nu^2,
\end{align*}
with $\beta_i^{(t)} = \frac{\mu - \nu}{C_2} - \frac{\nu}{C_1}$.
So we have when  $\gamma_i^{(t)} \geq \frac{1}{\poly(\Kappppa) d}$,
\begin{align*}
	\Phi^{(t + 1)} &\leq  \Phi^{(t)} - 2 \eta \nu^2 \left( \frac{\mu - \nu}{C_2} - \frac{\nu}{C_1} \right)^2 - \eta  \frac{1}{\poly(\Kappppa) d} \mu^2 + \eta | \Delta^{(t)}| \frac{\poly(\Kapppa)}{d^2} + \eta \frac{c_t \poly(\Kapppa)}{d^4}.
\end{align*}
When $\Phi^{(t)} \geq \delta$, we have that either $\mu^2 \geq \frac{\delta}{100}$, or $\mu^2 \leq \frac{\delta}{100}$ and $\nu^2 \geq \frac{\delta}{2}$.
In the first case, we have that
\begin{align*}
	\Phi^{(t + 1)} &\leq  \Phi^{(t)}  - \eta  \frac{1}{\poly(\Kappppa) d} \delta + \eta | \Delta^{(t)}| \frac{\poly(\Kapppa)}{d^2} + \eta \frac{c_t \poly(\Kapppa)}{d^4}.
\end{align*}
In the second case, we have that
\begin{align*}
	\Phi^{(t + 1)} &\leq  \Phi^{(t)} - \eta \Omega(\delta^2) + \eta | \Delta^{(t)}| \frac{\poly(\Kapppa)}{d^2} + \eta \frac{c_t \poly(\Kapppa)}{d^4}.
\end{align*}
Combining this equation with the bound in equation~\eqref{eq:fakvjbdabfaj}, we know that for $\delta = \frac{1}{d \log^{0.01} d}$, we have that
$\Phi^{(t)} \geq \delta$ can only happen for at most $\frac{d \log^{0.5} d}{\eta}$ many of the iterations within $t \in [T_2 + 1, T_3]$.
Combining the above with equation~\eqref{eq_claim_0th_II_cor1}, we obtain the desired result.
\end{proof}

\subsection{Stage 2.2: The Final Substage}

In this section, we present the proof of Lemma \ref{lem:final_333} for the final substage.
In the end, we prove the running inductive hypothesis $\cH_1$ in Proposition \ref{def_H1}.

\begin{proof}[Proof of Lemma \ref{lem:final_333}]
Suppose the lemma holds at iteration $t$, then using the condition at iteration $t$, together with $\Phi^{(0)} = O\left( \frac{1}{d \log^{0.01} d} \right)$, we have that
\begin{align}\label{eq:fajoifsjfasifjasoji}
\forall i \in [d]:  \gamma_i^{(t)} & \geq a_i - \frac{1}{ d \log^{0.001} d} \quad \text{ and } \beta_i^{(t)} \leq \frac{1}{  d \log^{0.001} d}
\end{align}
Define the following function
\begin{align*}
\tau_i :=  C_1( a_i - \beta_i - \gamma_i) + \frac{C_2 \gamma_i}{\beta_i + \gamma_i} \left( a_i - \gamma_i \right),
\end{align*}
Then we have the following
\begin{align*}
\beta_i^{(t + 1)} + \gamma_i^{(t + 1)} &= \left(\beta_i^{(t )} + \gamma_i^{(t )}  \right) \left( 1 - \eta \Delta^{(t)} + \eta \tau_i^{(t)}  \pm \eta \frac{\poly(\Kapppa ) c_t}{d^2} \right).
\end{align*}
Similar to the proof of Lemma~\ref{lem:zero_two}, we have that as long as $\Delta_+^{(t)} \geq  \frac{\poly(\Kapppa ) c_t}{d^2}$ and
\begin{align*}
\Delta_+^{(t)} \geq  \left( 1 - \frac{1}{\poly(\kappa )} \right)\delta_-^{(t)}.
\end{align*}
Then it must satisfy that $\Delta_+^{(t + 1)} \leq \Delta_+^{(t )} \left( 1 - \eta \frac{1}{ d \poly(\kappa ) } \right)$.
Hence, if the maximizer of $\Phi$ is $\Delta_+$, Then it must be the case that
\begin{align}\label{eq:fajiosfsafjasiasj}
\Delta_+^{(t + 1)} \leq \Delta_+^{(t )} \left( 1 - \eta \frac{1}{ d \poly(\kappa ) } \right).
\end{align}
Now, consider another case when $\Delta_-^{(t)} \leq  \left( 1 - \frac{1}{\poly(\kappa ) } \right)\delta_+^{(t)}$, let $i$ be the argmax of $\{\tau_j\}_{j \in [d]}$, then we must have that
\begin{align*}
	\beta_i^{(t + 1)} + \gamma_i^{(t + 1)} &= \left(\beta_i^{(t )} + \gamma_i^{(t )}  \right) \left( 1 - \eta \Delta^{(t)} + \eta \delta_+^{(t)}  \pm \eta \frac{\poly(\Kapppa ) c_t}{d^2} \right)\\
	& \geq  \left(\beta_i^{(t )} + \gamma_i^{(t )}  \right)  \left( 1 + \eta \frac{1}{\poly(\kappa) } \delta_-^{(t)}  - \eta \frac{\poly(\Kapppa ) c_t}{d^2}  \right).
\end{align*}
Hence as long as  $\delta_-^{(t)} \geq  \frac{\poly(\Kapppa ) c_t}{d^2}$, we have that
\begin{align} \label{eq:fajoifajsofias}
	\beta_i^{(t + 1)} + \gamma_i^{(t + 1)}  \geq  \left(\beta_i^{(t )} + \gamma_i^{(t )}  \right)  \left( 1 + \eta \frac{1}{\poly(\kappa) } \delta_-^{(t)}  \right).
\end{align}
On the other hand, since
\begin{align}
	\gamma_i^{(t + 1)} &=  \gamma_i^{(t )}\left( 1 - \eta \Delta^{(t)} + \eta   C_1( a_i - \beta_i - \gamma_i) + \eta C_2 \left( a_i - \gamma_i \right)  \pm \eta \frac{\poly(\Kapppa ) c_t}{d^2} \right) \nonumber \\
	& \geq  \gamma_i^{(t )}   \left( 1 - \eta \Delta^{(t)} + \eta  \delta_-^{(t)}  - \eta \frac{\poly(\Kapppa ) c_t}{d^2} \right) \nonumber \\
	& \geq  \gamma_i^{(t )}  \left( 1 + \eta \frac{1}{\poly(\kappa) } \delta_-^{(t)}  \right). \label{eq:bsofjaofiaj}
\end{align}
Notice that
\begin{align*}
f(\gamma_i) = \gamma_i \left( a_i - \gamma_i \right)
\end{align*}
is a decreasing function of $\gamma_i$ with slop at least $0.5 \gamma_i$ when $\gamma_i \geq \frac{a_i}{2}$, which holds true using Eq~\eqref{eq:fajoifsjfasifjasoji}.
Combining Eq~\eqref{eq:fajoifajsofias} and Eq~\eqref{eq:bsofjaofiaj}, we have that if the maximizer of $\Phi$ is $\delta_-$, the following is true
\begin{align}\label{eq:fajfiosajasoifasjf}
	\delta_-^{(t + 1)} \leq \delta_-^{(t + 1)} \left( 1 -  \eta \frac{1}{\poly(\kappa) d } \right).
\end{align}
Consider another case when the maximizer is $\Delta_-$.
Similar to the proof of Lemma~\ref{lem:zero_two}, as long as $\Delta_-^{(t)} \geq  \frac{\poly(\Kapppa ) c_t}{d^2}$, we have that
\begin{align*}
	\Delta_-^{(t)} \geq  \left( 1 - \frac{1}{\poly(\kappa )} \right)\delta_+^{(t)}.
\end{align*}
Hence, if the maximizer of $\Phi$ is $\Delta_-$, then it must be the case that
\begin{align}\label{eq:fjaoisfjiafajsif}
	\Delta_-^{(t + 1)} \leq \Delta_-^{(t )} \left( 1 - \eta \frac{1}{\poly(\kappa ) d } \right).
\end{align}
Moreover, using the fact that when $\Delta_-^{(t)} \leq  \left( 1 - \frac{1}{\poly(\kappa ) } \right)\delta_+^{(t)}$, let $i$ be the argmax of $\{- \tau_j\}_{j \in [d]}$, then we must have that  as long as  $\delta_+^{(t)} \geq  \frac{\poly(\Kapppa ) c_t}{d^2}$, we have that
\begin{align*}
\beta_i^{(t + 1)} + \gamma_i^{(t + 1)}  \leq  \left(\beta_i^{(t )} + \gamma_i^{(t )}  \right)  \left( 1 - \eta \frac{1}{\poly(\kappa) } \delta_+^{(t)}  \right)
\end{align*}
Consider two cases.
\begin{enumerate}
\item The maximizer is $\delta_+$. Then we must have that for every $i \in [d]$, $ \beta_i^{(t)} \leq C \delta_+^{(t)}$, then we must have that $a_i - \gamma^{(t)} \leq C \delta_+^{(t)}$ as well. Hence, it holds that %
\begin{align*}
 \gamma_i^{(t + 1)}  &\leq  \gamma_i^{(t )}  \left( 1 - \eta \frac{1}{\poly(\kappa) } \delta_+^{(t)}  + \eta\frac{\beta_i}{\gamma_i + \beta_i} (a_i - \gamma_i^{(t)}) \right)
 \\
 & \leq \gamma_i^{(t )} \left( 1 - \eta \frac{1}{\poly(\kappa) } \delta_+^{(t)}  +  \eta 2\kappa d (C \delta_+^{(t)})^2  \right)
 \\
 & \leq \gamma_i^{(t )}  \left( 1 - \eta \frac{1}{\poly(\kappa) } \delta_+^{(t)}  \right).
\end{align*}
Hence if the maximizer of $\Phi$ is $\delta_+$, it must be the case:
\begin{align}\label{eq:jfoasijasf}
	\delta_+^{(t + 1)} \leq \delta_+^{(t + 1)} \left( 1 -  \eta \frac{1}{\poly(\kappa) d } \right).
\end{align}
\item The maximizer is $\beta_+$. Then there is a $j \in [d]$ such that $ \beta_j^{(t)} \geq C \delta_+^{(t)}$ , $  \beta_j^{(t)} \geq C \delta_-^{(t)}$ and $\beta_j^{(t)} \geq C |\Delta^{(t)}|$, we have that for this $j$, it holds: let $S = C_1\left(\beta_j^{(t )} + \gamma_j^{(t )} - a_j  \right)$ and $\rho = (a_j - \gamma_j^{(t)})$, we have: if $\rho \leq \frac{1}{4} \beta_j^{(t)}$, then
\begin{align*}
S \geq C_1 \beta_j^{(t)}  - 2 C_1  \rho \geq \frac{C_1}{2} \beta_j^{(t)} \geq 2 |\Delta^{(t)}|.
\end{align*}
On the other hand if $\rho > \frac{1}{4} \beta_j^{(t)}$, then using $\delta_-^{(t)}$, we have:
\begin{align*}
- S + \rho &\leq \delta_-^{(t)}.
\end{align*}
This implies that
\begin{align*}
	S \geq \rho -  \delta_-^{(t)} \geq \frac{1}{4} \beta_j^{(t)} -  \delta_-^{(t)} \geq |\Delta^{(t)}| + \frac{1}{8} \beta_j.
\end{align*}
Hence if the maximizer of $\Phi$ is $\beta_+$, it must be the case:
\begin{align}
 \beta_j^{(t + 1)}  &=  \beta_j^{(t)} \left( 1- \eta \Delta^{(t)}  - \eta S \pm  \eta \frac{\poly(\Kapppa ) c_t}{d^2}\right) \nonumber \\
 & \leq  \beta_j^{(t)}\left( 1 - \eta \frac{1}{8}  \beta_j^{(t)}  \right). \label{eq:fajosifjaoijsio}
\end{align}
\end{enumerate}
To sum up, the result follows by combining Eq~\eqref{eq:fajosifjaoijsio}, ~\eqref{eq:jfoasijasf}, ~\eqref{eq:fjaoisfjiafajsif}, ~\eqref{eq:fajiosfsafjasiasj} and~\eqref{eq:fajfiosajasoifasjf}.
\end{proof}

\subsubsection{Proof of the Inductive Hypothesis}\label{app_proof_H1}

\paragraph{Verifying the inductive hypothesis $\cH_1$ during Stage 2.}

\begin{proof}[Proof of Proposition \ref{def_H1}]
We first verify the inductive hypothesis for $t \le T_3$.
The bound for $v\notin \set{S}_{pot}$ follows from Claim~\ref{claim:delta_21} and Proposition~\ref{prop:lb_beta_gamma}.
We prove the bound for $v\in \set{S}_{pot}$ by tracking the gradient descent dynamic.
Following Eq~\eqref{eq:nbosdhofaiaho}, for every neuron $v$, and every $p \in [d]$, define
\begin{align*}
	&Q_p^{(t) } := 2 \sum_{j \geq 2} \left( B_{1,2j}  (a_p - \gamma_p^{(t)} ) \frac{(v_p^{(t)})^{2j - 2}}{\| v^{(t)}\|_2^{2j - 2}} - B_{2, 2j} \frac{ \sum_{r}  (a_r - \gamma_r^{(t)} ) (v_r^{(t)})^{2j} }{\| v^{(t)} \|_2^{2j}}  \right), \\
	& R_p^{(t)}:=  2 \sum_{j \geq 2} \left( B_{1,2j}  (a_p - \gamma_p^{(t)} ) \frac{(v_p^{(t)})^{2j - 2}}{\| v^{(t)}\|_2^{2j - 2}}  \right).
\end{align*}
Hence, using Eq~\eqref{eq_gradient_approx}, we have that for every $i \in [d]$
{\small\begin{align} \label{eq:fasjoasifjaosifj}
	\left[v_i^{(t + 1)} \right]^2 = \left[ v_i^{(t )} \right]^2 \left( 1  - \eta \Delta^{(t)} + \eta C_1(a_i - \beta_i^{(t)} - \gamma_i^{(t)})  + \eta Q_p^{(t) } \pm \eta  \frac{c_t \poly(\Kappa)}{d^2}  \right).
\end{align}}
Hence, we have that for every $i, j \in [d]$,  %
{\small\begin{align*}
	\frac{ \left[v_i^{(t + 1)} \right]^2}{ \left[v_j^{(t + 1)} \right]^2 } = \frac{\left[ v_i^{(t )} \right]^2 }{ \left[v_j^{(t)} \right]^2 } \left( 1 +  \eta C_1(a_i - \beta_i^{(t)} - \gamma_i^{(t)})  -  \eta C_1(a_j - \beta_j^{(t)} - \gamma_j^{(t)}) + \eta R_i^{(t)} - \eta R_j^{(t)} \pm \frac{c_t \poly(\Kappa)}{d^2} \right).
\end{align*}}%
Now, if $|\bar{v}_i^{(t)}|^2, |\bar{v}_j^{(t)}|^2 \leq \frac{c_t}{d}$, we also have that
{\small\begin{align*}
	\frac{ \left[v_i^{(t + 1)} \right]^2}{ \left[v_j^{(t + 1)} \right]^2 } = \frac{\left[ v_i^{(t )} \right]^2 }{ \left[v_j^{(t )} \right]^2 } \left( 1 +  \eta C_1(a_i - \beta_i^{(t)} - \gamma_i^{(t)})  -  \eta C_1(a_j - \beta_j^{(t)} - \gamma_j^{(t)})  \pm  \eta \frac{c_t \poly(\Kappa)}{d^2} \right).
\end{align*}}%
This implies that
{\small\begin{align*}
	\frac{ \left[v_i^{(t + 1)} \right]^2}{ \left[v_j^{(t + 1)} \right]^2 } = \frac{\left[ v_i^{(0)} \right]^2 }{ \left[v_j^{(0)} \right]^2 }\exp \left\{ \pm 2 \eta C_1 \sum_{s \leq t}  \left( |a_i - \beta_i^{(t)} - \gamma_i^{(t)} | +   |a_j - \beta_j^{(t)} - \gamma_j^{(t)} |  \right) \pm  \eta  \frac{c_t \poly(\Kappa)}{d^2} t  \pm \eta^2 t \right\}.
\end{align*}}%
Hence using Proposition~\ref{prop:lb_beta_gamma} we show that when $|\bar{v}_i^{(0)}|^2, |\bar{v}_j^{(0)}|^2 \leq \frac{c_0}{d} $, then $|\bar{v}_i^{(t)}|^2, |\bar{v}_j^{(t)}|^2 \leq \frac{c_t}{d}$ as well for every $t  \leq T_3$.
Now, we need to give an upper bound on the the coordinates of the neurons.
For every $v \notin \set{S}_{pot}$, we know that all coordinates $j \in [d]$ satisfies that $|\bar{v}_j^{(0)}|^2 \leq \frac{c_0}{d}$. Hence, by Eq~\eqref{eq:fasjoasifjaosifj}, we have that
{\small\begin{align*}
	\left[v_i^{(t + 1)} \right]^2 &\leq \left[ v_i^{(t )} \right]^2 \left( 1 + \eta  |\Delta^{(t)} | + \eta  C_1 |a_i - \beta_i^{(t)} - \gamma_i^{(t)}| + \eta \kappa O( \| v^{(t)} \|_4^4) \right), \\
  & \leq \left[ v_i^{(t )} \right]^2\left( 1 + \eta  |\Delta^{(t)} | + \eta  C_1 |a_i - \beta_i^{(t)} - \gamma_i^{(t)}| + \eta \kappa  \frac{c_t^2}{d^2} \right).
\end{align*}}%
Hence we have that
\begin{align*}
	\left[v_i^{(t + 1)} \right]^2 \leq  \left[ v_i^{(0)} \right]^2 \exp \left\{  \eta \sum_{s \leq t} \left( |\Delta^{(s)} | +   C_1 |a_i - \beta_i^{(s)} - \gamma_i^{(s)}| \right) + \eta \kappa \frac{c_t^2}{d} t  \right\}.
\end{align*}
Hence using Proposition~\ref{prop:lb_beta_gamma} and Claim~\ref{claim:delta_21}, we have proved  Eq~\eqref{eq:fsajfoiofasjcjuiefh} and Eq~\eqref{eq:bnodsifahsoia}.

Next, we proceed to the norm of neurons $v \in \set{S}_{i, good}$.
For this neuron, using the fact that $|v_i^{(t)}| \geq d^6$ and equation~\ref{eq:cor_upppp}, we have that
{\small\begin{align*}
v_i^{(t + 1)} = \left( 1 - \eta \Delta^{(t)} + \eta C_1(a_i - \beta_i^{(t)} - \gamma_i^{(t)})   + \eta C_2 (a_i - \gamma_i^{(t)}) \pm \eta \frac{\poly(\Kapppa ) c_t}{d^2}   \right) v_i^{(t)}.
\end{align*}}%
Hence, for every $t$, using Eq~\eqref{eq:fjaijasofsajffjs} we obtain that: %
\begin{align*}
\left[v_i^{(t)} \right]^2  \geq \frac{\left[v_i^{(0)} \right]^2}{d} \geq \frac{1}{\lambda_0 \poly(d)}.
\end{align*}
Notice that for every neuron $ v \in \set{S}_g$, we have that $|\bar{v}_i^{(0)}|^2 \leq \frac{c_0}{d}$ can happen for at most $O(\log^{0.01} d)$ many $i \in [d]$.
Denote this set as $\set{Q}_v$, we have that
$\left[v_i^{(t + 1)} \right]^2 / \left[ v_i^{(t )} \right]^2$ is at most
{\small\begin{align*}
  & 1  - \eta \Delta^{(t)} + \eta C_1(a_i - \beta_i^{(t)} - \gamma_i^{(t)})  + \eta C_2 (a_i - \gamma_i^{(t)})+ \eta O\left( \sum_{p \in \set{Q}_v}  [( \gamma_p^{(t)} - a_p )]^+ \right) + \eta  \frac{c_t \poly(\Kappa)}{d^2}  \\
  &\leq 1  - \eta \Delta^{(t)} + \eta C_1(a_i - \beta_i^{(t)} - \gamma_i^{(t)})  + \eta C_2 (a_i - \gamma_i^{(t)})+ \eta O\left( \sum_{p \in \set{Q}_v}  |a_p - \beta_p^{(t)} - \gamma_p^{(t)}| \right) + \eta  \frac{c_t \poly(\Kappa)}{d^2}.
\end{align*}}%
Hence, for every $t \leq T \leq T_3$, using Eq~\eqref{eq:fakvjbdabfaj} and Eq~\eqref{eq:fjaijasofsajffjs}, by working on $\hat{\gamma}_i$ and notice that $\hat{\gamma}_i \leq \gamma_i$, we conclude that for every $i \in [d]$.
\begin{align*}%
	\eta \sum_{t \leq T}  |a_i - \gamma_i^{(t)} |  \leq \log \left( \frac{d}{\hat{\gamma}_i^{(0)}} \right).
\end{align*}
Combining the above equation with Eq~\eqref{eq:fjaijasofsajffjs} we have that  for every $i \in [d]$:
\begin{align*}
\left[v_i^{(t )} \right]^2 \leq \left[v_0^{(t )} \right]^2 \exp \{ \Gamma_i \} \leq \left[v_0^{(t )} \right]^2 \frac{d}{\hat{\gamma}_i^{(0)}}
\end{align*}
This proves that gradient truncation never happens during this substage.
Now, apply Lemma~\ref{lem:stage_1_final}, which says that
\begin{align}\label{eq_error_T3}
\hat{\gamma}_i^{(0)} \geq \mu(\set{S}_{i, good}) \frac{\lambda_0}{\poly(d)} \geq \poly(d) \mu(\set{S}_{i, bad}).
\end{align}
We complete the proof of the first our statements.
For the last statement on $\gamma,\beta$, Claim \ref{claim:zero_order_update_new_2} also proves the upper bound on $\gamma_i^{(t)} + \beta_i^{(t)}$ as in equations \eqref{eq_gamma} and \eqref{eq_beta}.
Taking $\delta = \frac{1}{\kappa d}$, we can show that
\[ \beta_i^{(t)} + \gamma_i^{(t)} \leq \frac{\poly(\Kapppa)}{ d}. \]

Next verify the running inductive hypothesis $\set{H}_2$ for $T_3 \le t \le T_4$.
Based on Lemma \ref{lem:final_333}, we have the following bounds on the update of each coordinate of each neuron.
For every $i \in [d]$, using Eq~\eqref{eq_gradient_approx}, we have that
\begin{align*}
|v_i^{(t + 1)}  |=    |v_i^{(t)} | \left( 1 \pm \eta O\left( |\Delta^{(t)}| +  |a_i - \beta_i^{(t)} - \gamma_i^{(t)}| + |a_i - \gamma_i^{(t)}| \right) \pm \frac{c_t \Kappa}{d^2} \right)
\end{align*}
Note that by the definition of $\Phi^{(t)}$ at Lemma~\ref{lem:final_333}, we have that
\begin{align}
|a_i - \gamma_i | &\leq  O \left( \left| C_1( a_i - \gamma_i) + \frac{C_2 \gamma_i}{\beta_i + \gamma_i} \left( a_i - \gamma_i \right)  \right| \right) + O(| \beta_i|)
\nonumber\\
& \leq O\left(  \beta_+ + \delta_+ + \delta_- \right) = O(\Phi) \label{eq_err_gamma}
\end{align}
Hence, we obtain that for $t \in [T_3, T_4]$, with Lemma~\ref{lem:final_333}:
\begin{align*}
|v_i^{(t)}  |   = |v_i^{(T_3)}| \exp \left\{ \pm O\left( \log \frac{t}{d}  \right)  \right\}
\end{align*}
Hence as long as $T_4 \leq \frac{ d^{1 + 10\cz}}{\eta}$, we obtain the running hypothesis $\set{H}_2$ at this substage.
\end{proof}

\section{Proof of the Finite-Width Case}
\label{app_finite}

We begin by describing the connection between the finite-width dynamic and the infinite-width dynamic.
For a vector $w \in \set{S}_g$, let $\tw^{(0)}$ and $w^{(0)}$ be a neuron with initialization $w$ in the finite-width and infinite-width case, respectively.
Let $\cP$ denote the infinite neuron distribution and $\tilde{\cP}$ denote the finite neuron population with $m$ samples.
Our idea is to track the difference between $\tw^{(t)}$ and $w^{(t)}$, denoted by $\xi_w^{(t)}$, throughout the update.
The neuron $w$ denotes a weight vector from the infinite width case that we specify below.
Specifically, the truncated gradient descent update of $\tw$ for the finite-width case and the update of $w$ for the infinite-width case is equal to the following.
\begin{align}
	w^{(t + 1)} &= w^{(t)}- \eta \cdot \indi{\norm{w^{(t)}}^2\le 1/\lambda} \nabla_{w^{(t)}} L_{\infty}(\cP),\label{eq_inf_update}\\
	\tw^{(t+1)} &= \tw^{(t)} - \eta \cdot \indi{\norm{\tw^{(t)}}^2 \le 1/\lambda} {\nabla}_{\tw^{(t)}} L(\tilde{\cP}) +\eta \Xi_w^{(t)}, \label{eq_finite_update}
\end{align}
where $\Xi_w^{(t)}$ is an extra error term that arises from the sampling error of the empirical loss.

Our main result in this section is that provided with polynomially many neuron samples and training samples, the errors $\xi_w^{(t)}$ and $\Xi_w^{(t)}$ in equation \eqref{eq_finite_update} remain polynomially small throughout Algorithm \ref{alg}.
We first state the result for Stage 1.
\begin{lemma}[Error propagation of Stage 1]\label{lem:error_final_stage_1}
	In the setting of Theorem \ref{thm:main}, let $\tilde{\cP}^{(0)}$ be a uniform distribution over $m$ i.i.d. samples from $\cP$.
	There exists a fixed value $\Xi \in [0, {\lambda_0}/{\poly(d)}]$ such that for every iteration $t \leq T_2$, the average norm of the error is small:
	$\E_{\tw \sim \tilde{\cP}^{(t)}} \| \xi_w\|_2^2 \leq \poly(d) \Xi.$
	Furthermore, for every $\tw^{(t)}$ in $\tilde{\cP}^{(t)}$, the individual error terms are small:
	$\|  \Xi_w^{(t)} \|_2^2 \leq \Xi$ and $\| \xi_w^{(t)}\|_2^2 \leq {\poly(d)}\cdot  \Xi / {\lambda_0}$.
\end{lemma}

The proof of Lemma \ref{lem:error_final_stage_1} can be found in Section \ref{sec_proof_error_stage1}.
Next, we consider the error propagation of Stage 2.1.
We show that the norm of $\xi_w$ is much smaller than that of $w$.
\begin{lemma}[Error propagation of Stage 2.1]\label{lem:error_final_stage_2}
	In the setting of Theorem \ref{thm:main}, let $\tilde{\cP}^{(T_2 + 1)}$ be a uniform distribution over $m$ i.i.d. samples from $\cP^{T_2+1}$.
	There exists a fixed value $\Xi  \in \left[0, \frac{1}{\poly_{\kappa} (d)} \right]$ such that for every iteration $T_2 < t \le T_3$ and every neuron $\tw^{(t)}$ in $\tilde{\cP}^{(t)}$, the error terms are small:
	$\|\Xi_w^{(t)} \|_2^2 \leq \Xi$,
	$\|\xi_w^{(t)}\|_2^2 \leq \min(\poly_{\kappa}(d) \Xi, \|\xi_w^{(t)}\|_2^2  \leq \| w^{(t)}\|_2^2/ {d^{20}} )$.
\end{lemma}

The proof of Lemma \ref{lem:error_final_stage_2} involves carefully studying the error term and follows a similar argument to Lemma \ref{lem:error_final_stage_1}.
The details can be found in Section \ref{sec_error_stage21}.
Finally, we consider the error terms in the final stage.
We use a different error analysis.
At iteration $T_3$, let us consider the set
\begin{align*}
	\set{S}_{i, singleton} := \left\{  v \in \set{S}_g \mid \| \bar{v} - e_i \|_2 \leq \frac{1}{\poly(d)}  \right\}, \text{ for } 1\le i\le d.
\end{align*}
Let $\set{S}_{singleton} = \cup_{i = 1}^d \set{S}_{i , singleton}$.
Consider the set
\begin{align*}
	\set{S}_{ignore} := \left\{v \in \set{S}_g \mid v \in \set{S}_{pot}, v \notin \set{S}_{singleton} \right\},
\end{align*}
where we recall the definition of $\set{S}_{pot}$ in Proposition \ref{def_H1}.
We state the error propagation of the final substage as follows.
\begin{lemma}[Error propagation of Stage 2.2]\label{lem:error_final_stage_22}
	In the setting of Theorem \ref{thm:main}, let $\tilde{\cP}^{(T_3+1)}$ be a uniform distribution over $m$ i.i.d. samples from $\cP^{(T_3+1)}$.
	There exists a fixed value $\Xi \in \left[0, {1}/{\poly_{\kappa}(d)} \right]$
	such that for every iteration $T_3 < t \le T_4$ and for every $\tw^{(t)}$,
	the error term $\|  \Xi_w^{(t)} \|_2^2 \leq \Xi$.
	For every neuron $w$ with $w \notin \set{S}_{ignore}$,
	we have that $\|\xi_w^{(t)}\|_2^2  \leq \| w^{(t)}\|_2^2 / {\poly(d)}$ and
	$\E_{\tilde{w} \sim \ctP^{(t)}, w \in \set{S}_{ignore}} \| \xi_w^{(t)} \|_2^2 \leq {1} / {\poly(d)}$.
\end{lemma}
The proof of Lemma \ref{lem:error_final_stage_22} can be found in Section \ref{sec_error_stage22}.
Based on the analysis of error propagation, we are now ready to prove our main result.
We prove Theorem \ref{thm:main} as follows.

\begin{proof}[Proof of Theorem \ref{thm:main}]
Let us denote $\tw^{(t)}$ to be the weight of the neuron $\tw$ at iteration $t$, following the update of Algorithm \ref{alg}.
For the next iteration, we have that
	\[ \tw^{(t + 1)} = \tw^{(t)} - \eta \cdot \indi{\| \tw^{(t)} \|_2^2 \leq \frac{1}{2 \lambda}} \nabla_{\tw_i^{(t)}} \obje(\tilde{\cP}^{(t)}), \]
where $\obje(W)$ denotes the empirical loss. %

Recall that we assume that the learning rate $\eta \le \frac{1}{\poly_{\kappa}(d)}$.
Using Claim~\ref{claim:911}, we can see that when $N$ for a sufficiently large polynomial in $d$, we have that with probability at least $1 - e^{- \log^2 d}$, for every $w \in W$ and every $t \leq T_4$: $\left\|  \nabla_{\tw^{(t + 1)}} \obje(\tilde{\cP}^{(t+1)}) -  \nabla_{w^{(t + 1)}}L_{\infty}(\cP) \right\|_2 \leq \frac{1}{\poly_{\kappa}(d)}$.

For Stage 1, we can first apply Lemma~\ref{lem:error_final_stage_1} with $\Xi = \frac{1}{\poly_{\kappa}(d)} $ and $\ctP$ being a uniform distribution over $W$.
Using $m = \poly_{\kappa}(d)$, we can conclude that for every $w \in W$ and $t \leq T_2$: $\|\tw^{(t)} - w^{(t)} \|_2 \leq   \frac{1}{\poly_{\kappa}(d)} $.

For Stage 2.1, we can use  Lemma~\ref{lem:error_final_stage_2}  with $\Xi = \frac{1}{\poly_{\kappa}(d)}$ to conclude that for every $t \leq T_3$, $\|\tw^{(t)} - w^{(t)} \|_2 \leq   \frac{1}{\poly_{\kappa}(d)} $ as well.
For Stage 2.2, we shall use Lemma~\ref{lem:error_final_stage_22} with $\Xi = \frac{1}{\poly(d)}$ to conclude that
for every neuron $w \notin \set{S}_{ignore}$, we have:
	\[ \| \tw^{(T_4)} - w^{(T_4)}\|_2^2  \leq \frac{1}{\poly(d)} \| w\|_2^2  \text{ and } \E_{\tw^{(T_4)} \sim \ctP^{(T_4)}, w \in \set{S}_{ignore}} \| \tw^{(T_4)} - w^{(T_4)} \|_2^2 \leq \frac{1}{\poly(d)}. \]
These statements together give us the following
	$$\E_{x \sim \cN(0, \id_{d \times d})}\left(f_{\{\tw^{(T_4)} | w \in W\}}(x) - f_{\{ w^{(T_4)} | w \in W \}} (x) \right)^2 \leq \frac{1}{\poly(d)}.$$
Finally, combined with Claim~\ref{claim:911} and Theorem~\ref{thm_inf} we complete the proof of Theorem \ref{thm:main}.
\end{proof}
\subsection{Stage 1.1: Analysis of 0th and 2nd Order Tensor Decompositions}

In this substage, we consider the error terms in the gradients of the 0th and 2nd order tensor decompositions.
Let $\ctH_0$ is the running hypothesis $\cH_0$ in Proposition \ref{def_H0} without the conditionally-symmetric property. %
For a neuron $v$, similar to the definition of $\nabla_v$ in Claim \ref{claim:symmetry}, we define the error gradient $\tilde{\nabla}_v$ as
	$\tilde{\nabla}_v := \nabla_v \objinf(\ctP) = \sum_{j \geq 0} \tilde{\nabla}_{2j, v}.$
Using the update of equation \eqref{eq_inf_update} for the infinite-width case, the gradient of $v$ for the 0th and 2nd order terms over the population loss is given by
{\begin{align*}
	\nabla_{\leq 2, v} = b_0 \left( \E_{\cP} \| w\|_2^2 - 1  \right) v  +  b_2 \left( \E_{\cP} ww^{\top}  -  A \right)  v,
\end{align*}}%
where $A = \text{diag}( \{a_i\}_{i \in [d]})$.
Using the update of equation \eqref{eq_finite_update} for the finite-width case, the gradient of $v$ for the 0th and 2nd order terms over the population loss is given by
\begin{align*}
	{\tilde{\nabla}}_{\leq 2, v} = b_0 \left( \E_{\ctP} \| \tilde{w}\|_2^2 - 1  \right) \tv   +  b_2 \left( \E_{\ctP} \tw \tw^{\top} -  A \right)  \tv.
\end{align*}
The first order terms of the error term $\xi_v$ for the neuron $v$ is given by
{\begin{align}
	\tilde{\nabla}_{0, v, 1} &:=  2 b_0 \left( \E_{\ctP} \langle w, \xi_w \rangle \right) v  +  b_0 \left( \E_{\ctP} \| {w}\|_2^2 - 1  \right)  \xi_v, \nonumber \\
	\tilde{\nabla}_{2, v, 1} &:= b_2 \left( \E_{\ctP} w w^{\top} -  A \right)  \xi_v + b_2 \left( \E_{\ctP} w \xi_w^{\top} \right) v + b_2 \left( \E_{\ctP}  \xi_w w^{\top} \right) v. \label{eq:fajoisfjafjasijf}
\end{align}}%
Let $\tilde{\nabla}_{\leq 2, v, 1}$ denote the sum of $\tilde{\nabla}_{0, v, 1}$ and $\tilde{\nabla}_{2, v, 1}$.
We have the following claim for the error of the 0th and 2nd order terms.
We use the notation $(w, \xi_w)\sim \ctP$ to denote a neuron $\tw = w + \xi_w$ sampled from $\ctP$.%
\begin{claim}[Error of 0th and 2nd order gradients]\label{claim:0_2_error_F}
In the setting of Theorem \ref{thm:main}, at every iteration $t \le T_4$, the following is true for the neuron distribution $\ctP = \ctP^{(t)}$,
{\begin{align*}
	& \left( \E_{(w, \xi_w) \sim \ctP} \langle w, \xi_w \rangle \right) \E_{(v, \xi_v) \sim \ctP}  \langle v, \xi_v \rangle \geq 0, \text{ and} \\
	& \E_{(v, \xi_v) \sim \ctP} \left( \left\langle  \left( \E_{(w, \xi_w) \sim \ctP} w \xi_w^{\top} \right) v, \xi_v \right\rangle  \right) + \E_{(v, \xi_v) \sim \ctP} \left( \left\langle  \left( \E_{(w, \xi_w) \sim \ctP}  \xi_w w^{\top} \right) v, \xi_v \right\rangle  \right) \geq 0,
\end{align*}}%
As a corollary,
{\begin{align}
	\E_{(v, \xi_v ) \sim \ctP}\langle  \tilde{\nabla}_{\leq 2, v , 1} , \xi_v \rangle &\geq b_0 \left( \E_{\ctP} \| {w}\|_2^2 - 1  \right)  \E_{\ctP}[\|\xi_v\|_2^2] +  b_2  \E_{\ctP}\left[\xi_v^{\top} \left( \E_{\ctP} w w^{\top} -  A \right)  \xi_v \right]. \label{eq:fjasoifajfasifasjfjfsi}
\end{align}}
\end{claim}
\begin{proof}%
The first inequality is obviously true. Now we consider the second inequality, we have that
{\begin{align*}
	\langle w \xi_w^{\top} v , \xi_v \rangle + \langle  \xi_w w^{\top} v , \xi_v \rangle  &= \xi_v^{\top} w \xi_w^{\top} v + \xi_v^{\top} \xi_w w^{\top} v \\
	& = \frac{1}{2} \Tr \left( \left( \xi_v v^{\top} + v \xi_v^{\top} \right)  \left( \xi_w w^{\top} + w \xi_w^{\top} \right) \right).
\end{align*}}%
This implies that
{\begin{align*}
	 \E_{(v, \xi_v) \sim \ctP, (w, \xi_w) \sim \ctP} \langle w \xi_w^{\top} v , \xi_v \rangle + \langle  \xi_w w^{\top} v , \xi_v \rangle
  &=\frac{1}{2} \E_{(v, \xi_v) \sim \ctP, (w, \xi_w) \sim \ctP}  \Tr \left( \left( \xi_v v^{\top} + v \xi_v^{\top} \right)  \left( \xi_w w^{\top} + w \xi_w^{\top} \right) \right) \\
  &= \frac{1}{2}  \Tr \left( \E_{(w, \xi_w) \sim \ctP}    \left( \xi_w w^{\top} + w \xi_w^{\top} \right) \right)^2 \geq 0.
\end{align*}}
\end{proof}

Next we consider the first order tensor.
The first order gradient in the finite-width case for the population loss is
\begin{align*}
	\tilde{\nabla}_{1 , v} &= b_1 \E_{\tilde{\cP}}  \left( \langle \tw, \tv \rangle \|\tw \|_2 \bar{\tv} + \| \tw\|_2 \|\tv \|_2 \tw\right).
\end{align*}
The 1st order loss in the gradient is zero in the infinite-width case of Section \ref{app_inf}.
The first-order expansion of the error is given by:
\begin{align*}
\tilde{\nabla}_{1, v, 1} &:= b_1 \E_{\tcP} \langle \bar{w}, \xi_w \rangle \langle w, v \rangle \bar{v} + b_1 \E_{\tcP}  \langle \bar{v}, \xi_w \rangle \| w\|_2 v
\\
&+ b_1 \E_{\tcP} \langle \bar{w}, \xi_w \rangle \|v\|_2 w +  b_1 \E_{\tcP}\|w\|_2 \| v\|_2 \xi_w
\end{align*}

We have the following claim for the error in the first order gradients.
\begin{claim}[Error of 1st order gradient]\label{claim:1_error_F}
	In the setting of Theorem \ref{thm:main}, the following holds for any distribution $\ctP$ on $w, \xi_w$ ($v, \xi_v$ follows the same distribution).
	\begin{align*}
		\langle \tilde{\nabla}_{1, v, 1} , \xi_v \rangle =  b_1 \left\| \E_{\ctP} \|w \|_2 \xi_w +  w \langle \bar{w}, \xi_w \rangle  \right\|_F^2 \geq 0
	\end{align*}
\end{claim}

\subsection{Stage 1.2: Analysis of Higher Order Tensor Decompositions}\label{sec:error_4_plus_1}

In this substage, we consider the error terms of the gradients of the higher order tensor decompositions.
Towards showing the error propagation in Lemma \ref{lem:error_final_stage_1}, our proof outline is as follows.
\begin{itemize}
	\item We decompose the error of the gradients into individual terms that we analyze one by one.
	\item In Proposition \ref{prop:error_total_stage_1}, we provide an upper bound on the average norm of the error.
	In Proposition \ref{lem:ind_err_bound_1}, we bound the error of the individual terms from the decomposition.
	Finally, we present the proof of Lemma \ref{lem:error_final_stage_1} in Section \ref{sec_proof_error_stage1}.
\end{itemize}
We begin by writing down the gradient of higher order terms for the population loss..
{\begin{align*}
\tilde{\nabla}_{2j , v} &= \left( b_{2j} + b_{2j}'\right) \left( \E_{\ctP} \langle \tw, \tv \rangle \langle \bar{\tw} , \bar{\tv} \rangle^{2j - 2} \tw -   \sum_{i} a_i \langle e_i, \tv \rangle \langle e_i, \bar{\tv} \rangle^{2j - 2} e_i \right) \\
& - b_{2j}'\left( \E_{\ctP} \langle \tw, \tv \rangle \langle \bar{\tw} , \bar{\tv} \rangle^{2j - 2} \langle \tw, \bar{\tv} \rangle-   \sum_{i} a_i\langle e_i, \tv \rangle \langle e_i, \bar{\tv} \rangle^{2j - 1}  \right)  \bar{\tv}.
\end{align*}}%
A crucial result is a bound on the average norm of the error.
Let us define $\tDelta:= \left|\E_{w \sim \ctP} \| w\|_2^2 - 1 \right|$ and $\tdelta = \left\| \E_{w \sim \ctP}  w w^{\top} - A \right\|_2.$
We have the following result.
\begin{proposition}[Average error bound]\label{prop:error_total_stage_1}
	In the setting of Lemma \ref{lem:error_final_stage_1},
	suppose the running hypothesis $\ctH_0$ holds for every $t \in [T_2]$.
	In addition,
	(i) For every neuron $w$, it holds that  $\| \xi_w \|_2 \leq \frac{1}{d^{20}}  \| w\|_2$;
	(ii) $ |\E_{w \sim \ctP}[w \| w\|_2]| \leq \frac{1}{d^{40}}$ and $\| \E_{w \sim \ctP} w w^{\top} \|_2 \leq 1$.

	As long as for every $w \in \set{S}$ (recalling its definition in Def. \ref{defn:no_win}), $\| \xi_w \|_2 \leq \frac{1}{\poly(d)}$, then we have
{\begin{align*}
		\E_{\ctP^{(t + 1)}, w \in \set{S}} \| \xi_w \|_2^2 \leq&  \left( 1 + \eta \frac{c_t \poly(\kappa)}{d^2} \right) \E_{\ctP^{(t)}, w \in \set{S}} \| \xi_w \|_2^2 + \eta O\left( \frac{1}{\lambda_0} \right) \left( \E_{\ctP^{(t)}, w \notin \set{S}} \| \xi_w \|_2 \right) \left(\E_{\ctP^{(t)}, w \in \set{S}} \| \xi_w \|_2 \right) \\
		& + \eta O\left( \max \left\{ \tDelta^{(t)}, \tdelta^{(t)} \right\} \right) \E_{\ctP^{(t)}, w \in \set{S}} \| \xi_w \|_2^2.
\end{align*}}%
\end{proposition}
Next we show that the norm of the error in each individual neuron can also be bounded.
\begin{proposition}[Individual error bound]\label{prop:err_ind_stage_1}
	In the setting of Lemma \ref{lem:error_final_stage_1},
	suppose that for every $t \in [T_2]$, we have
		$\E_{w, \xi_w \sim \ctP_t} \| \xi_w\|_2^2 \leq \Xi$.
	Then for every $v \in \set{S}_g$ and every $t \in [T_2]$:
	\begin{align*}
		\| \xi_v^{(t)} \|_2^2 \leq \frac{\poly(d)}{\lambda_0} \left( \| \xi_v^{(0)} \|_2^2 + \Xi \right).
	\end{align*}
\end{proposition}
The proof of Proposition \ref{prop:error_total_stage_1} and Proposition \ref{prop:err_ind_stage_1} is left to Section \ref{sec_proof_error_stage1}.

\subsubsection{Decomposition of the Gradient}

We focus on the leading term that contains the first order term in $\xi$.
We decompose $\tilde{\nabla}_{2j , v}$ into the following terms.
In particular, these include terms from {\small$\overline{v + \xi_v} = \bar{v} + \frac{\xi_v}{\| v\|_2} - \frac{\langle \xi_v, \bar{v} \rangle \bar{v}}{\| v\|_2} + O\left(\left( \frac{\|\xi_v\|_2}{\|v \|_2} \right)^2 \right)$.}
{\small
\begin{align*}
\tilde{\nabla}_{2j, v, 1} &:= (2j - 1)\left( b_{2j} + b_{2j}'\right)  \left( \E_{\ctP} \langle \xi_w , v \rangle \langle \bar{w} , \bar{v} \rangle^{2j - 2} w + \E_{\ctP} \langle w, \xi_v \rangle \langle \bar{w} , \bar{v} \rangle^{2j - 2} w \right),
\\
\tilde{\nabla}_{2j, v, 2}  &:= \left( b_{2j} + b_{2j}'\right)  \E_{\ctP} \langle w, v \rangle \langle \bar{w} , \bar{v} \rangle^{2j - 2} \xi_w,
\\
\tilde{\nabla}_{2j, v, 3} &:= - (2j - 2) \left( b_{2j} + b_{2j}'\right) \left( \E_{\ctP} \langle \bar{w}, v \rangle \langle \bar{w} , \bar{v} \rangle^{2j - 2} \langle \xi_w, \bar{w} \rangle w  \right),
\\
\tilde{\nabla}_{2j, v, 4} &:= - (2j - 2)  \left( b_{2j} + b_{2j}'\right)\left( \E_{\ctP} \langle w, \bar{v} \rangle \langle \bar{w} , \bar{v} \rangle^{2j - 2} \langle \xi_v, \bar{v} \rangle w  \right),
\\
\tilde{\nabla}_{2j, v, 5} &:= -(2j)b_{2j}'\left( \E_{\ctP} \langle \xi_w, v \rangle \langle \bar{w} , \bar{v} \rangle^{2j - 2} \langle w, \bar{v} \rangle   \bar{v} \right),
\tilde{\nabla}_{2j, v, 6} := -(2j)b_{2j}'\left( \E_{\ctP} \langle \xi_v, w \rangle \langle \bar{w} , \bar{v} \rangle^{2j - 2} \langle w, \bar{v} \rangle   \bar{v} \right),
\\
\tilde{\nabla}_{2j, v, 7} &:= (2j - 2)b_{2j}'\left( \E_{\ctP} \langle \xi_w, \bar{w} \rangle \langle \bar{w} , \bar{v} \rangle^{2j - 1}  \langle w, v  \rangle  \bar{v} \right),
\tilde{\nabla}_{2j, v, 8} := (2j - 1)b_{2j}'\left( \E_{\ctP} \langle \xi_v, \bar{v} \rangle \langle \bar{w} , \bar{v} \rangle^{2j - 2} \langle w, \bar{v} \rangle^2   \bar{v} \right),
\\
\tilde{\nabla}_{2j, v, 9} &:=- b_{2j}'\left( \E_{\ctP} \langle w, \bar{v} \rangle \langle \bar{w} , \bar{v} \rangle^{2j - 2} \langle w, \bar{v} \rangle  \xi_v\right),
\tilde{\nabla}_{2j, v, 10} := b_{2j}'\left( \E_{\ctP} \langle w, \bar{v} \rangle \langle \bar{w} , \bar{v} \rangle^{2j - 2} \langle w, \bar{v} \rangle \langle \xi_v, \bar{v} \rangle \bar{v}  \right),
\\
\tilde{\nabla}_{2j, v, 11} &:=  - (2j - 1) \left( b_{2j} + b_{2j}'\right) \left( \sum_{i} a_i \langle e_i, \xi_v \rangle \langle e_i, \bar{v} \rangle^{2j - 2} e_i  \right),
\\
\tilde{\nabla}_{2j, v, 12} &:=   (2j - 2) \left( b_{2j} + b_{2j}'\right) \left( \sum_{i} a_i \langle  \xi_v, \bar{v} \rangle \langle e_i, \bar{v} \rangle^{2j - 1} e_i  \right),
\\
	\tilde{\nabla}_{2j, v, 13}  &:=  (2j)b_{2j}'\left(  \sum_{i} a_i\langle e_i, \xi_v \rangle \langle e_i, \bar{v} \rangle^{2j - 1}  \right)  \bar{v},
	\tilde{\nabla}_{2j, v, 14}  :=  - (2j - 1)b_{2j}'\left(  \sum_{i} a_i\langle  \xi_v , \bar{v }\rangle \langle e_i, \bar{v} \rangle^{2j}  \right)  \bar{v}, \\
	\tilde{\nabla}_{2j, v, 15}  &:= b_{2j}'\left(  \sum_{i} a_i \langle e_i, \bar{v} \rangle^{2j }  \right)\xi_v,
	\tilde{\nabla}_{2j, v, 16}  := -b_{2j}'\left(  \sum_{i} a_i \langle e_i, \bar{v} \rangle^{2j }  \right) \langle \xi_v , \bar{v} \rangle \bar{v}.%
\end{align*}}%
In addition, we show that the second order terms in $\xi$ that contains $\| \xi_w \|_2^p$ and $ \| \xi_v \|_2^q$ for $p + q \geq 3$ are of a lower order compared to the first order terms. %
Informally, we know that $\| \xi_w \|_2$ and $\| \xi_v \|_2$ are less than $ \lambda_0^2$.
Meanwhile, $\| w\|_2$ and $v\|_2 $ are at least $\Omega\left( \frac{1}{d} \right)$,  for every $w, v \in \set{S}_g$ by Lemma~\ref{lem:grow_4}. %
Combined together, we show the following result.
\begin{proposition}\label{prop:err_higher}
	In the setting of Proposition \ref{prop:error_total_stage_1}, let $(w, \xi_w)$ be a random sample of $\ctP$.
	{\begin{align*}
		&\| \tilde{\nabla}_{\leq 2, v} - \tilde{\nabla}_{0, v, 1} - \tilde{\nabla}_{2, v, 1}  \|_2 + \| \tilde{\nabla}_{1, v, 1} - \tilde{\nabla}_{1, v} \|_2 + \sum_{j \geq 0} \left\|   \tilde{\nabla}_{2j , v} - \sum_{p } \tilde{\nabla}_{2j, v, p}   \right\|_2 \\
		\leq&  O\left(  \frac{1}{d^{10}}   \left(   \| v \|_2  \sqrt{\E_{ \xi_w \sim \ctP} \| \xi_w \|_2^2}   +  \| \xi_v \|_2 \right)\right).
	\end{align*}}
\end{proposition}

\subsubsection{Individual Error Norm bound}\label{app_proof_lem_stage1}

Based on the decomposition above, we provide several helper claims for bounding the error of the gradient terms.
First, for $v \in \set{S}$, we have the following claim.
\begin{claim}\label{claim:err_bound_four_plus_1}
	In the setting of Proposition \ref{prop:error_total_stage_1}, we have that
	\begin{align*}
    &\E_{(v, \xi_v) \sim \ctP, (w, \xi_w) \sim \ctP} \left( \ \langle \xi_w , v \rangle \langle \bar{w} , \bar{v} \rangle^{2j - 2} \langle w, \xi_v\rangle +   \langle w, \xi_v \rangle \langle \bar{w} , \bar{v} \rangle^{2j - 2} \langle w, \xi_v \rangle \right) \geq 0, \text{ and} \\
    &\E_{(v, \xi_v) \sim \ctP, (w, \xi_w) \sim \ctP} \langle w, v \rangle \langle \bar{w} , \bar{v} \rangle^{2j - 2} \langle \xi_w , \xi_v \rangle \geq 0.
	\end{align*}
	This implies that for $p = 1, 2$:
	\begin{align*}
		\E_{(v, \xi_v) \sim \ctP} \langle \xi_{v}, \tilde{\nabla}_{2j, v, p} \rangle \geq 0.
	\end{align*}
\end{claim}
\begin{proof}%
For the first inequality, we know that
	\begin{align*}
		&  \E_{(v, \xi_v) \sim \ctP, (w, \xi_w) \sim \ctP} \left( \ \langle \xi_w , v \rangle \langle \bar{w} , \bar{v} \rangle^{2j - 2} \langle w, \xi_v\rangle +   \langle w, \xi_v \rangle \langle \bar{w} , \bar{v} \rangle^{2j - 2} \langle w, \xi_v \rangle \right) \\
    &=  \E_{(v, \xi_v) \sim \ctP, (w, \xi_w) \sim \ctP} \left(  \langle \bar{w} , \bar{v} \rangle^{2j - 2}  \left(  \langle \xi_w , v \rangle \langle w, \xi_v\rangle +  \langle w, \xi_v \rangle^2  \right)\right) \\
    &=  \E_{(v, \xi_v) \sim \ctP, (w, \xi_w) \sim \ctP} \left(  \langle \bar{w} , \bar{v} \rangle^{2j - 2}  \left(  \langle \xi_w , v \rangle \langle w, \xi_v\rangle + \frac{1}{2} \langle w, \xi_v \rangle^2 + \frac{1}{2}\langle v, \xi_w \rangle^2  \right)\right) \\
    &=  \frac{1}{2}\E_{(v, \xi_v) \sim \ctP, (w, \xi_w) \sim \ctP} \left(  \langle \bar{w} , \bar{v} \rangle^{2j - 2}  (\langle w, \xi_v\rangle + \langle v, \xi_w\rangle)^2  \right) \geq 0.
\end{align*}
The second inequality in the Lemma follows from the fact that $(\langle w, v \rangle)_{w, v}, ( \bar{w} , \bar{v} \rangle)_{w, v}, ( \langle \xi_w , \xi_v \rangle)_{w, v}$ forms PSD matrices, and the Hadamard product of PSD matrices is PSD.
\end{proof}

We also have the following claim, which serves as an upper bound of
\[ \tilde{\nabla}_{2j, v, 3}, \tilde{\nabla}_{2j, v, 4}, \tilde{\nabla}_{2j, v, 5}, \tilde{\nabla}_{2j, v, 6}, \tilde{\nabla}_{2j, v, 7}, \tilde{\nabla}_{2j, v, 8},  \tilde{\nabla}_{2j, v, 9},  \tilde{\nabla}_{2j, v, 10}.\]
\begin{claim}~\label{claim:err_bound_four_plus_2}
	In the setting of Proposition \ref{prop:error_total_stage_1}, we have that
	\begin{align*}
    & \E_{(v, \xi_v) \sim \ctP, (w, \xi_w) \sim \ctP, w, v \in \set{S}}  \left[ \langle \bar{w} , \bar{v} \rangle^{ 2} |    \langle \xi_w, v \rangle \langle \xi_v, w \rangle |  \right] \leq  \frac{c_t \kappa}{d^2} \E_{(w, \xi_w) \sim \ctP, w \in \set{S}} \| \xi_w \|_2^2, \text{ and} \\
    & \E_{(v, \xi_v) \sim \ctP, (w, \xi_w) \sim \ctP, w, v \in \set{S}}  \left[ \langle \bar{w} , \bar{v} \rangle^{ 2}    \langle \xi_w, v \rangle^2   \right] \leq  \frac{c_t \kappa}{d^2} \E_{(w, \xi_w) \sim \ctP, w \in \set{S}} \| \xi_w \|_2^2,
	\text{ and} \\
    & \E_{(v, \xi_v) \sim \ctP, (w, \xi_w) \sim \ctP, w, v \in \set{S}}  \left[ \langle \bar{w} , \bar{v} \rangle^{ 2}    \langle w, \bar{v} \rangle^2   \right] \leq  \frac{c_t \kappa}{d^2}.
	\end{align*}
	As a corollary, combine the above inequality with Proposition~\ref{def_H0}, we obtain %
	\begin{align*}
		\sum_{j \geq 2} \E_{(v, \xi_v) \sim \ctP , v \in \set{S}} |\langle \xi_{v}, \tilde{\nabla}_{2j, v, p} \rangle| = O\left( \frac{c_t \kappa}{d^2}\E_{(v, \xi_v) \sim \ctP , v \in \set{S}}  \| \xi_v \|_2^2  +\frac{1}{\lambda_0}\E_{(w, \xi_w) \sim \ctP , w \notin \set{S}}  \| \xi_w \|_2^2 \right),
	\end{align*}
	where $p = 3, 4 ,5, 6, 7, 8, 9, 10 $.
\end{claim}

\begin{proof}%
The proof is a direct calculation, using $\langle \bar{w} , \bar{v} \rangle^{ 2} \leq \frac{c_t}{d}$ for $w \in \set{S}$, we have that
\begin{align*}
& \E_{(v, \xi_v) \sim \ctP, (w, \xi_w) \sim \ctP, w, v \in \set{S}}  \left[ \langle \bar{w} , \bar{v} \rangle^{ 2} |  \langle \xi_w, v \rangle \langle \xi_v, w \rangle |  \right] \\
\leq~ & \frac{c_t}{d}  \E_{(v, \xi_v) \sim \ctP, (w, \xi_w) \sim \ctP, w, v \in \set{S}}    | \langle \xi_w, v \rangle \langle \xi_v, w \rangle |
\\
\leq~ & \frac{c_t}{d}  \E_{(v, \xi_v) \sim \ctP, (w, \xi_w) \sim \ctP, w, v \in \set{S}}    |  \langle \xi_w, v \rangle \langle \xi_v, w \rangle |
\\
\leq~ &  \frac{c_t}{d}\E_{(v, \xi_v) \sim \ctP, (w, \xi_w) \sim \ctP, w, v \in \set{S}}  \left( \langle \xi_v, w \rangle^2 +    \langle \xi_w, v \rangle^2 \right).
\end{align*}
Now, we can easily calculate that (using the Eq~\eqref{eq:fajosifsajfasjif})
\begin{align*}
\E_{(v, \xi_v) \sim \ctP, v \in \set{S}} \langle \xi_w, v \rangle^2  \leq \E_{(v, \xi_v) \sim \ctP} \langle \xi_w, v \rangle^2 \leq \frac{2 \kappa}{d} \| \xi_w \|_2^2,
\end{align*}
which completes the proof.
For the other two inequalities, we can bound them in the exact same way.
\end{proof}

The final claim aims to bound the rest of the terms.
\begin{claim}~\label{claim:err_bound_four_plus_3}
In the setting of Proposition \ref{prop:error_total_stage_1}, we have that
	\begin{align*}
		\left(\sum_{i=1}^d a_i \langle e_i, \bar{v} \rangle^{2 } \langle e_i, \xi_v \rangle^2 \right) = O\left( \frac{c_t \kappa}{d^2} \| \xi_v\|_2^2 \right) \text{ and }~
		\left(\sum_{i=1}^d a_i \langle e_i, \bar{v} \rangle^{2} \right) = O\left( \frac{c_t \kappa}{d^2} \right).
	\end{align*}
As a corollary, combining the above inequality with Proposition~\ref{def_H0}, we obtain %
\begin{align*}
	\sum_{j \geq 2}\E_{(v, \xi_v) \sim \ctP , v \in \set{S}} |\langle \xi_{v}, \tilde{\nabla}_{2j, v, p} \rangle| = O\left( \frac{c_t \kappa}{d^2}\E_{(v, \xi_v) \sim \ctP , v \in \set{S}}  \| \xi_v \|_2^2 \right),
\end{align*}
where $p = 11, 12, 13, 14, 15, 16$.
\end{claim}

\paragraph{Individual Error Norm Bound for $v \in \set{S}$.}
Below we also consider the error individually, we will mainly focus on the error term with $\xi_v$.
\begin{claim}\label{claim:err_3}
	In the setting of Proposition \ref{prop:error_total_stage_1}, we have that
	\begin{align*}
		& \E_{ (w, \xi_w ) \sim \ctP} \langle w, \xi_v \rangle^2 \langle \bar{w} , \bar{v} \rangle^{2}  \leq O\left( \frac{\kappa c_t}{d^2} \| \xi_v \|_2^2\right), \text{ and }\\
		& \left| \E_{(w, \xi_w )  \sim \ctP } \langle w, \bar{v} \rangle \langle \bar{w} , \bar{v} \rangle^{2} \langle \xi_v, \bar{v} \rangle \langle w , \xi_v \rangle \right|  \leq O\left( \frac{\kappa c_t}{d^2} \| \xi_v \|_2^2\right).
	\end{align*}.
\end{claim}
\begin{proof}
	The first inequality is almost trivial. To see the second one, using $\E_{w \in \set{S}, w \sim \ctP}[  \langle \bar{w} , \bar{v} \rangle^{2}] \leq \frac{c_t}{d}$, we have that
	\begin{align*}
		&\left| \E_{(w, \xi_w ) \sim \ctP, w \in \set{S}} \langle w, \bar{v} \rangle \langle \bar{w} , \bar{v} \rangle^{2} \langle \xi_v, \bar{v} \rangle \langle w , \xi_v \rangle \right| \\
		& \leq \frac{c_t}{d}  \E_{(w, \xi_w ) \sim \ctP} \left| \langle w, \bar{v} \rangle \langle w , \xi_v \rangle \right| \| \xi_v \|_2 \\
		& \leq \frac{c_t}{d} \|\xi_v \|^2  \E_{(w, \xi_w ) \sim \ctP} \left( \langle w, \bar{v} \rangle^2 +  \frac{\langle w , \xi_v \rangle ^2}{\|\xi_v\|_2^2} \right) \\
	& \leq O\left( \frac{\kappa c_t}{d^2} \| \xi_v \|_2^2\right).
\end{align*}
For $ w \notin \set{S}$, we can naively bound $| \langle w, \bar{v} \rangle \langle \bar{w} , \bar{v} \rangle^{2} \langle \xi_v, \bar{v} \rangle \langle w , \xi_v \rangle  | \leq \| w\|_2^2 \| \xi_v \|_2^2 $. Hence, using Eq~\eqref{lem:norm_bound}, we have:
\begin{align*}
	\left| \E_{(w, \xi_w ) \sim \ctP, w \notin \set{S}} \langle w, \bar{v} \rangle \langle \bar{w} , \bar{v} \rangle^{2} \langle \xi_v, \bar{v} \rangle \langle w , \xi_v \rangle \right|  \leq \Lambda \| \xi_v \|_2^2 = \frac{1}{\poly(d)} \| \xi_v \|_2^2.
\end{align*}
\end{proof}

This claim together with Claim~\ref{claim:err_bound_four_plus_3} implies that
\begin{claim}[Error bound, $v \in \set{S}$]
	In the setting of Proposition \ref{prop:error_total_stage_1}, for every $v \in \set{S}$, we have that
	\begin{align*}
    \| \xi_v^{(t + 1)} \|_2^2 &\leq  \left( 1 + \eta O\left( \max \left\{\tDelta^{(t)}, \tdelta^{(t)}  \right\}  \right) + \eta \frac{c_t \poly(\kappa)}{d^2} \right) \| \xi_v^{(t)} \|_2^2 + O\left( \frac{d^2}{\lambda_0} \right) \left( \E_{\ctP^{(t)}} \| \xi_w \|_2^2 \right).
	\end{align*}
\end{claim}

\paragraph{Individual error bound for all the other neurons.}
Now we move on to the harder terms, we have the following claim.
\begin{claim}\label{claim:err_bound_four_plus_4}
	In the setting of Proposition \ref{prop:error_total_stage_1}, for every $v \in \set{S}$, we have that for $p = 11, 13$:
	\begin{align*}
		& \sum_{j \geq 2} |\langle \xi_v, \nabla_{2j, v, p} + \nabla_{2j, v, p  + 1}\rangle| = O\left(\frac{\kappa}{d} \| \bar{v} \|_{\infty}^2  \| \xi_v \|_{2}^2 \right), \text{ and} \\
		& \langle  \tilde{\nabla}_{2j, v, 15}  + \tilde{\nabla}_{2j, v, 16}, \xi_v \rangle \geq 0.
	\end{align*}
\end{claim}
\begin{proof}%
We first consider $p = 11$.
Let $Q_{2j,  v, 11}'  = - \left( b_{2j} + b_{2j}'\right) \sum_i \left( a_i   \langle e_i, \bar{v} \rangle^{2j - 2} e_i  e_i^{\top} \right).$
We have that
{\begin{align*}
	&\tilde{\nabla}_{2j, v, 11}  + \tilde{\nabla}_{2j, v, 12} \\
	= & - (2j - 1) \left( b_{2j} + b_{2j}'\right) \left( \sum_{i} a_i \langle e_i, \xi_v \rangle \langle e_i, \bar{v} \rangle^{2j - 2} e_i  \right) + (2j - 2) \left( b_{2j} + b_{2j}'\right) \left( \sum_{i} a_i \langle  \xi_v, \bar{v} \rangle \langle e_i, \bar{v} \rangle^{2j - 1} e_i  \right) \\
	= & Q_{2j,  v, 11}' \xi_v  - (2j - 1)  \left( b_{2j} + b_{2j}'\right)  \left( \sum_{i} a_i \langle e_i, \xi_v \rangle \langle e_i, \bar{v} \rangle^{2j - 2} e_i -  \sum_{i} a_i \langle  \xi_v, \bar{v} \rangle \langle e_i, \bar{v} \rangle^{2j - 1} e_i  \right).
\end{align*}}
Let us assume that $\|\bar{v}\|_{\infty} = 1 - \delta$ for some value $\delta \geq 0$, then we have that $\|\bar{v} - e_r \|_2^2 = O(\delta)$.
\begin{align*}
	&\left| \sum_{i \in [d]} \left( a_i \langle \xi_v, e_i \rangle^2 \langle e_i, \bar{v} \rangle^2 - a_i \langle e_i, \bar{v} \rangle^3 \langle e_i, \xi_v \rangle \langle \xi_v , \bar{v} \rangle \right) \right| \\
	\leq~& \left| \sum_{i \in [d], i \not= r} \left( a_i \langle \xi_v, e_i \rangle^2 \langle e_i, \bar{v} \rangle^2 - a_i \langle e_i, \bar{v} \rangle^3 \langle e_i, \xi_v \rangle \langle \xi_v , \bar{v} \rangle \right) \right| \\
	&~+ \frac{\kappa}{d}  \left|  \langle \xi_v, e_r \rangle^2 \langle e_r, \bar{v} \rangle^2 -  \langle e_r, \bar{v} \rangle^3 \langle e_r, \xi_v \rangle \langle \xi_v , \bar{v} \rangle  \right| \\
	\leq~&  O\left( \frac{\kappa}{d}  (1 - \delta)^2\sqrt{\delta} \| \xi_v \|_2^2 \right).
\end{align*}
Using the fact that $b_{2j} , b_{2j}' = \Theta(\frac{1}{j^2})$, we know that
\begin{align*}
	&\sum_{j \geq 2}  (2j - 1)  \left( b_{2j} + b_{2j}'\right)  \left| \sum_{i} a_i \langle e_i, \xi_v \rangle \langle e_i, \bar{v} \rangle^{2j - 2} \langle e_i, \xi_v \rangle -  \sum_{i} a_i \langle  \xi_v, \bar{v} \rangle \langle e_i, \bar{v} \rangle^{2j - 1}  \langle e_i, \xi_v \rangle  \right| \\
	\leq&  \sum_{j \geq 2} O\left( \frac{1}{j}(1 - \delta)^{j}   \sqrt{\delta}  \frac{\kappa}{d}  \| \xi_v \|_2^2 \right).
\end{align*}
Note that $\sum_{j \geq 2} \frac{1}{j}(1 - \delta)^{j}  =    (1 - \delta)\log \frac{1}{\delta} $  we obtain:
\begin{align*}
	\sum_{j \geq 2} \left| \langle \tilde{\nabla}_{2j, v, 11}  + \tilde{\nabla}_{2j, v, 12}  -  Q_{2j,  v, 11}'  \xi_v , \xi_v \rangle \right| \leq O\left(  \frac{\kappa}{d}  (1- \delta)^2 \log \frac{1}{\delta} \sqrt{\delta}  \| \xi_v \|_2^2 \right) = O\left(  \frac{\kappa}{d}  \| \xi_v \|_2^2 \right).
\end{align*}
Similarly, we can also show that
\begin{align*}
	&\sum_{j \geq 2} \| Q_{2j,  v, 11}'\|_2  =  \sum_{j \geq 2} \left( b_{2j} + b_{2j}'\right)\left\|  \sum_{i } \left( a_i   \langle e_i, \bar{v} \rangle^{2j - 2} e_i  e_i^{\top} \right) \right\|_2 \\
	 \leq & O\left( \frac{\kappa}{d} \| \bar{v} \|_{\infty}^2 \right),
\end{align*}
which completes the proof.
On the other hand, for $p = 13$, let $Q_{2j,  v, 13}' = b_{2j}'\left(   \sum_{i} a_i \langle e_i, \bar{v} \rangle^{2j - 1}   \bar{v} e_i^{\top} \right).$.
We have that
\begin{align*}
	&\tilde{\nabla}_{2j, v, 13}   + \tilde{\nabla}_{2j, v, 14} \\
	= & (2j)b_{2j}'\left(  \sum_{i} a_i\langle e_i, \xi_v \rangle \langle e_i, \bar{v} \rangle^{2j - 1}  \right)  \bar{v}  - (2j - 1)b_{2j}'\left(  \sum_{i} a_i\langle  \xi_v , \bar{v }\rangle \langle e_i, \bar{v} \rangle^{2j}  \right)  \bar{v} \\
	= & Q_{2j,  v, 13}' \xi_v + (2j - 1) b_{2j}' \left(  \sum_{i} a_i\langle e_i, \xi_v \rangle \langle e_i, \bar{v} \rangle^{2j - 1}   -  \sum_{i} a_i\langle  \xi_v , \bar{v }\rangle \langle e_i, \bar{v} \rangle^{2j}  \right)  \bar{v},
\end{align*}
We can bound the terms in a similar way.
\end{proof}

Using the aforementioned claims, we conclude the proof of the following proposition.
\begin{proposition}[Individual error bound]\label{lem:ind_err_bound_1}
In the setting of Proposition \ref{prop:error_total_stage_1}, for every $v$, we have that
\begin{enumerate}
\item For $p = 4, 6, 8, 9, 10$, using Claim \ref{claim:err_3}, we have
\begin{align*}
\sum_{j \geq 2} |\langle \xi_v, \tilde{\nabla}_{2j, v, p} \rangle| \leq O\left( \frac{c_t  \poly(\kappa)}{d^2} \| \xi_v\|_2^2 \right)
\end{align*}
\item When $\| w \|_2, \| \xi_w \|_2 \leq \frac{1}{\lambda_0}$, for $p= 2, 3, 5, 7$, the following is true %
{\begin{align*}
&\sum_{j \geq 2} |\langle \xi_v, \tilde{\nabla}_{2j, v, p} \rangle| \leq  O\left( \frac{1}{\lambda_0} \right) \left( \E_{(w, \xi_w) \sim \ctP} \| \xi_w \|_2 \right) \| \xi_v \|_2 \leq O\left( \frac{d^2}{\lambda_0} \right)  \left( \E_{(w, \xi_w) \sim \ctP} \| \xi_w \|_2^2\right) + \frac{1}{d^2} \| \xi_v \|_2^2
\end{align*}}
\item For $p = 1$, similarly we have:
\begin{align*}\sum_{j \geq 2} |\langle \xi_v, \tilde{\nabla}_{2j, v, p} \rangle| &\leq  O\left( \frac{c_t \kappa}{d^2} \| \xi_v\|_2^2 + \frac{1 }{\lambda_0} \left(\E_{(w, \xi_w) \sim \ctP} \| \xi_w \|_2^2 \right) \right)
\end{align*}
\item For $p = 15$:
\begin{align*}\sum_{j \geq 2} \langle \xi_v, \tilde{\nabla}_{2j, v, p} + \tilde{\nabla}_{2j, v, p + 1}  \rangle \geq 0
\end{align*}
\item For $p = 11, 12, 13, 14$, we have that for $p = 11, 13$, using Claim~\ref{claim:err_bound_four_plus_4}, we get %
	\begin{align*}
		\sum_{j \geq 2} |\langle \xi_v, \nabla_{2j, v, p} + \nabla_{2j, v, p  + 1}\rangle| = O\left(\frac{\kappa}{d}  \| \bar{v} \|_{\infty}^2\| \xi_v \|_{2}^2  \right)
	\end{align*}
\end{enumerate}
\end{proposition}

\subsubsection{Proof of Error Propagation}\label{sec_proof_error_stage1}

Based on the individual error norm bound and the average error norm bound, we are ready to prove the main result of stage 1.
We first state the proof of the individual error norm bound.

\begin{proof}[Proof of Proposition \ref{prop:err_ind_stage_1}]
	We consider the error caused by gradient clipping, since we might clip $v$ and $\tv$ at different time step. We have that at the iteration $t$ when $\| v^{(t)} \|_2^2 \geq \frac{1}{2\lambda_0}$ or $\| v^{(t)}  + \xi_v^{(t)}  \|_2^2 \geq \frac{1}{2 \lambda_0}$, we have that
	\begin{align*}
		\| v^{(t)}  + \xi_v^{(t)}  \|_2^2 , \| v^{(t)}  \|_2^2 \geq \frac{1}{2 \lambda_0} - \frac{2}{\sqrt{\lambda_0}} \| \xi_v^{(t)}  \|_2.
	\end{align*}
	On the other hand, by Eq~\eqref{eq:bhoaifhaosifhasif}, we have that for this $v$, if the gradient clipping is not performed, then by the definition of $\set{S}_g$, we have that $\| \bar{v}^{(t)} \|_2 \geq \frac{1}{\log d}$.
	Therefore,
	\begin{align*}
		\| v^{(t + 1)}\|_2^2 \geq \| v^{(t)}\|_2^2 \left( 1 + \eta \Omega \left( \frac{1}{d \log^3 d} \right) \right),
	\end{align*}
	which implies that after $t' = O\left( \frac{\sqrt{\lambda_0}d \log^3 d\| \xi_v^{(t)}  \|_2}{\eta} \right)$, many iterations, if gradient clipping is not performed,  we should have that
	\begin{align*}
		\| v^{(t + t')}  + \xi_v^{(t + t')}  \|_2^2 , \| v^{(t + t')}  \|_2^2 \geq \frac{1}{2 \lambda_0}.
	\end{align*}
	Since each iteration shall introduce at most $O\left(\eta \frac{1}{\lambda_0} \right)$ amount error, so we have:
	\begin{align*}
		\|\xi_v^{(t + t')} \|_2 \leq O\left(  \frac{d \log^3 d\| \xi_v^{(t)}  \|_2}{\sqrt{\lambda_0}} \right).
	\end{align*}
	This gives us the final error bound of the individual error when combined with Claim~\ref{claim:err_bound_four_plus_3}.
\end{proof}

Next we state the proof of the average error norm bound.
\begin{proof}[Proof of Proposition \ref{prop:error_total_stage_1}]
	Using Proposition \ref{lem:ind_err_bound_1} (together with Eq~\eqref{eq:fjasoifajfasifasjfjfsi}) and by the definition of Eq~\eqref{eq:fajoisfjafjasijf}, we can obtain the desired result.
\end{proof}

Based on Proposition \ref{prop:error_total_stage_1}, Proposition \ref{prop:err_ind_stage_1}, and Proposition \ref{lem:ind_err_bound_1}, we are ready to prove Lemma \ref{lem:error_final_stage_1}.
\begin{proof}[Proof of Lemma~\ref{lem:error_final_stage_1}]
Clearly, when  $m \geq \frac{\poly(d)}{\poly(\lambda_0   )}$, then the running hypothesis $\ctH_0 $ is satisfied for every $t \leq T_2$.
To prove this Lemma, we shall maintain using induction that at every iteration $t \in [T_2]$,
\begin{align*}
	\E_{w, \xi_w \sim \ctP^{(t)}} \| \xi_w\|_2^2 \leq  \poly(d) \Xi,
\end{align*}
and for every neuron $v$, $\| \xi_v^{(t)}\|_2^2 \leq \frac{\poly(d)}{\lambda_0} \Xi$.
Suppose this is true for all $t \leq T_0$, then consider $t  = T_0+1$.
We apply Proposition~\ref{prop:error_total_stage_1}, which says that as long as for every $w \in \set{S}$, $\| \xi_w \|_2 \leq \frac{1}{d^3}$, we have that
\begin{align*}
	\E_{\ctP^{(t + 1)}, w \in \set{S}} \| \xi_w \|_2^2 &\leq  \left( 1 + \eta \frac{c_t \poly(\kappa)}{d^2} \right) \E_{\ctP^{(t)}, w \in \set{S}} \| \xi_w \|_2^2 + \eta O\left( \frac{1}{\lambda_0} \right) \left( \E_{\ctP^{(t)}, w \notin \set{S}} \| \xi_w \|_2 \right) \left(\E_{\ctP^{(t)}, w \in \set{S}} \| \xi_w \|_2 \right) \\
	& + \eta O\left( \max \left\{ |\tDelta^{(t)}|, \tdelta^{(t)}  \right\} \right) \E_{\ctP^{(t)}, w \in \set{S}} \| \xi_w \|_2^2 + \eta  \E_{\ctP^{(t)}, w \in \set{S}} \| \xi_w \|_2 \| \Xi_w\|_2.
\end{align*}
Combining this with Proposition~\ref{prop:err_ind_stage_1} and $\E_{\ctP^{(t)}, w \in \set{S}} \| \xi_w \|_2 \| \Xi_w\|_2 \leq   \frac{1}{d^2}\E_{\ctP^{(t)}, w \in \set{S}} \| \xi_w \|_2 + d^2 \Xi$, we have that
\begin{align*}
	\E_{\ctP^{(t + 1)}, w \in \set{S}} \| \xi_w \|_2^2   &\leq  \left( 1 + \eta \frac{c_t \poly(\kappa)}{d^2} \right) \E_{\ctP^{(t)}, w \in \set{S}} \| \xi_w \|_2^2 +  \eta O\left( \frac{\poly(d)}{\lambda_0^2} \right) \mu(\set{S}) \poly(\Kapppa) \Xi \\
	& + \eta O\left( \max \left\{ |\tDelta^{(t)}|, \delta^{(t)} \right\} \right) \E_{\ctP^{(t)}, w \in \set{S}} \| \xi_w \|_2^2 + \eta d^2 \Xi.
\end{align*}
Hence, denote $\varepsilon_t = \E_{\ctP^{(t)}, w \in \set{S}} \| \xi_w \|_2^2$, we have that for every $s \leq t$:
\begin{align*}
	\varepsilon_{s + 1}  \leq \varepsilon_s  \left( 1 + \eta  \frac{c_t \poly(\kappa)}{d^2}  + \eta O\left( \max \left\{ |\Delta^{(t)}|, \delta_+^{(t)}, \delta_-^{(t)}  \right\} \right) \right) + \eta d^2 \Xi.
\end{align*}
By $m \geq \frac{\poly(d)}{\poly(\lambda_0   )} $, a simple Chernoff bound gives us:
$$\max \left\{ |\tDelta^{(t)}|, \tdelta^{(t)}  \right\}  \leq O\left( \max \left\{ |\Delta^{(t)}|, \delta_+^{(t)}, \delta_-^{(t)}  \right\} \right)+ \frac{1}{\poly(d)}.$$
Now, using the update rule of Eq~\eqref{eq:fsafjaoifjasfoiajsoiaj} and in Proposition~\ref{lem:zero_two2},  we have that
\begin{align*}
\sum_{t \leq T_2}  \left(\eta  \frac{c_t \poly(\kappa)}{d^2}  + \eta \max \left\{ |\Delta^{(t)}|, \delta_+^{(t)}, \delta_-^{(t)}  \right\}  \right) \leq \poly(\Kappa)
\end{align*}
Note that at iteration 0, $\varepsilon_0 = 0$.
This implies that for  $t  + 1$: $\varepsilon_{t + 1 } \leq  \poly(d)  \Xi$ as well.
Combine this with Proposition~\ref{lem:ind_err_bound_1} on the individual norm bound we complete the proof.
\end{proof}

\subsection{Stage 2.1: Analysis After Reducing the Gradient Truncation Parameter}
\label{sec_error_stage21}

In this section, we prove Lemma \ref{lem:error_final_stage_2}, which analyzes the error propagation of this substage.
We analyze the formula of $\tilde{\nabla}_{2j, v, p}$ in Section~\ref{sec:error_4_plus_1} and show the following claim.
\begin{claim}\label{claim:err_stage_2}
	In the setting of Lemma \ref{lem:error_final_stage_2},
	let $\sigma_{\max} =  \max \left\{ \| \E_{w \sim \ctP} ww^{\top} \|_2, \frac{\kappa}{d } \right\}$. %
	Then we have the following average error norm bound
	\begin{align*}
		\sum_{j \geq 2} \sum_{p} \E_{(v, \xi_v) \sim \ctP}  \langle \tilde{\nabla}_{2j, v, p}, \xi_v \rangle \geq - O\left(  \sigma_{\max}  \E_{\xi_w \sim \ctP}  \| \xi_w\|_2^2\right).
	\end{align*}
	For every individual neuron $v$ and every value $\alpha \geq 1$, the following holds
	\begin{align*}
		\sum_{j \geq 2} \sum_{p} | \langle \tilde{\nabla}_{2j, v, p}, \xi_v \rangle| \leq  O \left( \alpha  \sigma_{\max} \| \xi_v \|_2^2+  \frac{1}{\alpha}  \| v \|_2^2   \E_{\xi_w \sim \ctP}  \| \xi_w\|_2^2 \right).
	\end{align*}
\end{claim}
\begin{proof}%
The proof of this claim is quite straightforward.
We have that for $p = 1, 2$, we use Claim~\ref{claim:err_bound_four_plus_1}, which gives us:
\begin{align*}
	\E_{(v, \xi_v) \sim \ctP} \langle \xi_{v}, \tilde{\nabla}_{2j, v, p} \rangle \geq 0.
\end{align*}
For $p = 3$, we use that $|\langle \bar{w}, \bar{v} \rangle| \leq 1$ and
\begin{align*}
  \E_{(w, \xi_w), (v, \xi_v) \sim \ctP} |\langle \xi_w, \bar{w} \rangle| |\langle \xi_v, w \rangle| |\langle \bar{w}, v \rangle| &\leq  \frac{1}{2}\E_{(w, \xi_w), (v, \xi_v) \sim \ctP}\left( |\langle \xi_w, \bar{w} \rangle|^2 \langle w, v \rangle^2 +  |\langle \xi_v, w \rangle| ^2 \right) \\
  &\leq \sigma_{\max}  \E_{\xi_w \sim \ctP}  \| \xi_w\|_2^2.
\end{align*}
For $p = 4$, we use that
\begin{align*}
	\E_{(w, \xi_w), (v, \xi_v) \sim \ctP} |\langle \xi_v, \bar{v} \rangle| |\langle \xi_v, w \rangle| |\langle \bar{v}, w \rangle| &\leq   \E_{(w, \xi_w), (v, \xi_v) \sim \ctP}  \|\xi_v \|_2 |\xi_v^{\top}   ww^{\top} \bar{v}| \\
  &\leq \sigma_{\max}  \E_{\xi_w \sim \ctP}  \| \xi_w\|_2^2.
\end{align*}
For $p = 5$, we use that
\begin{align*}
  \E_{(w, \xi_w), (v, \xi_v) \sim \ctP} |\langle \xi_w,v \rangle| |\langle \xi_v, \bar{v} \rangle| |\langle \bar{v}, w \rangle| &\leq   \frac{1}{2} \E_{(w, \xi_w), (v, \xi_v) \sim \ctP}   \left(\langle \xi_w,v \rangle^2 +  \langle \xi_v, \bar{v} \rangle^2\langle \bar{v}, w \rangle^2 \right) \\
  &\leq \sigma_{\max}  \E_{\xi_w \sim \ctP}  \| \xi_w\|_2^2.
\end{align*}
For $p = 6$, we use that
\begin{align*}
  \E_{(w, \xi_w), (v, \xi_v) \sim \ctP} |\langle \xi_v,w \rangle| |\langle \xi_v, \bar{v} \rangle| |\langle \bar{v}, w \rangle| &\leq   \frac{1}{2} \E_{(w, \xi_w), (v, \xi_v) \sim \ctP}   \left(\langle \xi_v,w \rangle^2 +  \langle \xi_v, \bar{v} \rangle^2\langle \bar{v}, w \rangle^2 \right) \\
  &\leq \sigma_{\max}  \E_{\xi_w \sim \ctP}  \| \xi_w\|_2^2.
\end{align*}
For $p = 7$, we use that
\begin{align*}
	&\E_{(w, \xi_w), (v, \xi_v) \sim \ctP} |\langle \xi_w,\bar{w} \rangle| |\langle \xi_v, \bar{v} \rangle| |\langle v, w \rangle| |\langle \bar{w}, \bar{v } \rangle| \\
  &= \E_{(w, \xi_w), (v, \xi_v) \sim \ctP} |\langle \xi_w,\bar{w} \rangle| |\langle \xi_v, \bar{v} \rangle| |\langle v, \bar{w} \rangle| |\langle w, \bar{v } \rangle| \\
  &\leq   \frac{1}{2} \E_{(w, \xi_w), (v, \xi_v) \sim \ctP}   \left(\langle \xi_w,\bar{w} \rangle^2  \langle v , \bar{w} \rangle^2+  \langle \xi_v, \bar{v} \rangle^2\langle \bar{v}, w \rangle^2 \right) \\
  &\leq \sigma_{\max}  \E_{\xi_w \sim \ctP}  \| \xi_w\|_2^2.
\end{align*}
For $p = 8, 9, 10, 11$, the result can be obtained similarly.
For $ p = 12, 13$, we use that
\begin{align*}
	&\sum_i a_i |\langle \xi_v ,\bar{v} \rangle | |\langle e_i , \bar{v} \rangle |  | \langle e_i, \xi_v \rangle| \\
	& \leq \frac{\kappa}{d}  \| \xi_v \|_2 \left| \xi_v^{\top} \sum_i e_i e_i^{\top} \bar{v}\right| \leq  \frac{\kappa}{d}  \| \xi_v \|_2^2.
\end{align*}
For $p = 14$, we use that
\begin{align*}
	&\sum_i a_i |\langle \xi_v ,\bar{v} \rangle | |\langle e_i , \bar{v} \rangle |^2  | \langle \bar{v}, \xi_v \rangle| \\
	& \leq \| \xi_v \|_2^2 \sum_i a_i \langle e_i , \bar{v} \rangle^2 \leq \frac{\kappa}{d} \| \xi_v \|_2^2.
\end{align*}
Finally, the individual error bound  comes from the following simple calculation.
\begin{align*}
	\sum_{ j \geq 2} |\langle \tilde{\nabla}_{2j, v, p} , \xi_v \rangle | &\leq   O \left( \frac{\kappa}{d} \| \xi_v \|_2^2 +  \E_{w, \xi_w \sim \ctP} \| w\|_2^2 \| \xi_v \|_2^2 + \| w \|_2 \| \xi_w \|_2 \| v \|_2 \|\xi_v\|_2 \right) \\
	&  \leq O \left( \alpha  \sigma_{\max} \| \xi_v \|_2^2+  \frac{1}{\alpha}  \| v \|_2^2 \E_{\xi_w \sim \ctP}  \| \xi_w\|_2^2 \right).
\end{align*}
\end{proof}

\begin{proof}[Proof of Lemma \ref{lem:error_final_stage_2}]
	Note that this substage has  $T_3$ many iterations, where $T_3$ is upper bounded by
	$\frac{d C(\kappa) \log d}{\eta}$ for some value $C(\kappa) > 0$ that only depends on $\kappa$.
	By by taking $\alpha = 1$ in Claim \ref{claim:err_stage_2}, the rest of the proof is similar to the proof of Lemma~\ref{lem:error_final_stage_1}.
	We omit the details.
\end{proof}

\subsection{Stage 2.2: The Final Substage}\label{sec_error_stage22}

We provide the proof of Lemma \ref{lem:error_final_stage_22}, which analyzes the error propagation in the final substage.
Recall that $\set{S}_{i, singleton}$ and $\set{S}_{ignore}$ have been defined in the beginning of this section.
At the beginning of Stage 2.2 when $t = T_3 + 1$, we do a modification:
\begin{enumerate}
	\item If $v$ in $\set{S}_{i, singleton} $ we will just  set  $\bar{v} = e_i$ and keep the norm not changed.
	\item If $v$ in $\set{S}_{ignore}$, then we will just set $v = 0$.
\end{enumerate}
Thus, we can see that $v^{(t)} = 0$ for every $v \in \set{S}_{ignore}$ and for every $t > T_3$.
We define a new update for the infinite neuron process at this substage for $v \in \set{S}_{i, singleton}$.
We define $v_+ , v_-$ such that at every iteration $t \geq T_3$:
$$v_+^{(t )} = - v_-^{(t )} = \langle v^{(t + 1)}, e_i \rangle.$$
We will replace $v$ in the infinite neuron process with two neurons $v_+, v_-$.
For the simplicity of notation, we write $v_+$ simply as $v$.
For the other neurons, the update does not change.

We can see that this new initial state also satisfies the running hypothesis $\cH_1$ and the conditional-symmetric property as well.
Thus, the update in Claim~\ref{claim:grad_21} still holds.
We consider the new infinite neuron process starting from this initial state.
We can see that when $v \in \set{S}_{i, singleton}$, then $\bar{v}^{(t)} = e_i$  for every $t \geq T_3 + 1$.
Moreover, when $v \in \set{S}_{i, singleton}$, we define $\tv = v + \xi_v$ where $\langle \xi_v, e_i \rangle = 0$.
Thus, we do not consider the scaling difference between the singleton neurons in the infinite-width case and the finite-width case as an error.

By the running hypothesis $\set{H}_1$ in Proposition \ref{def_H1}, we have that at iteration $T_3$,
	\[ \E_{w \sim \cP^{(T_3)}, w \in \set{S}_{ignore}} \| w \|_2^2 \leq \frac{1}{\poly(d)}. \]
Therefore, the following is also true.
	\[ \E_{w \sim \cP^{(T_3 + 1)}, w \in \set{S}_{ignore}} \| \xi_w \|_2^2 \leq \frac{1}{\poly(d)}. \]
Moreover, throughout the entire process, by Lemma~\ref{lem:final_333}, we will always have that
\begin{align*}
	\E_{w \sim \ctP^{(t)} , w \in \set{S}_{ignore}   } \| \xi_w\|_2^2 \leq \frac{1}{\poly(d)}.
\end{align*}
Therefore, we only need to consider the error of $w \notin \set{S}_{ignore}$. We denote the new running hypothesis $\ctH_1$ as for every $t \leq T_4$:
\begin{enumerate}
	\item For every $v \in \set{S}_g$, we have that:
		\begin{align*}
			\| v^{(t)} \|_2^2 \leq \frac{1}{2\lambdaa}.
		\end{align*}
	\item For every $v \notin \set{S}_{pot}$ (cf. Lemma~\ref{lem:stage_1_final} for the definition),
		\begin{align*}
			\|\bar{v}\|_{\infty}^2, \|{v}\|_{\infty}^2  \leq \frac{c_t}{d}.
		\end{align*}
	\item The mass of the ignore set is small.
		\begin{align*}
			\E_{w \sim \ctP, w \in \set{S}_{ignore}} \| \xi \|_2^2  \leq \frac{1}{\poly(d)}.
		\end{align*}
\end{enumerate}

We first prove the following claim.
\begin{claim}\label{lem:final_err_3333}
	In the setting of Lemma \ref{lem:error_final_stage_22}, suppose that the distribution $\ctP = \ctP^{(t)}$ satisfies the running hypothesis $\ctH_1$.
	For every $i \in [d]$ and for $v \in \set{S}_{i, singleton}$, the following holds:
	\begin{align*}
		&\sum_{j \geq 2} \langle \tilde{\nabla}_{2j, v}, \xi_v \rangle = \sum_{j \geq 2} \langle \tilde{\nabla}_{2j, v, 1} + \tilde{\nabla}_{2j, v, 2}, \xi_v \rangle \\
		&\pm O \left( |\tPhi| \| \xi_v \|_2^2+ \frac{c_t \poly(\Kapppa)}{d^2}  \| \xi_v \|_2^2  + \frac{c_t \poly(\Kapppa)}{d^2} \| v \|_2^2 \E_{(w, \xi_w) \sim \ctP, w \notin \set{S}_{pot}} \|  \xi_w\|_2^2  +  \frac{1}{\poly(d)} \| \xi_v \|_2\right),
	\end{align*}
	where $\Phi$ is defined as in Lemma~\ref{lem:final_333} with $\ctP$ instead of $\cP$.
\end{claim}

\begin{proof}%
	We consider $v \in \set{S}_{i, singleton}$. For these neurons, we have that
	\begin{align*}
		\langle \tilde{\nabla}_{2j, v, p}, \xi_v \rangle = 0.
	\end{align*}
	Let
	\begin{align} \label{eq_2j_v_w}
		\nabla_{2j, v, w}  &\define \left( b_{2j} + b_{2j}'\right) \left( \langle w, v \rangle \langle \bar{w} , \bar{v} \rangle^{2j - 2} w \right)  - b_{2j}'\left(  \langle w, v \rangle \langle \bar{w} , \bar{v} \rangle^{2j - 2} \langle w, \bar{v} \rangle \right)  \bar{v}
	\end{align}
	as the gradient of $v$ involving only a single neuron $w$.
	For $p = 5, 6, 7, 8, 10, 13, 14, 16$. Now, for $p = 3 $ we have that
	\begin{align*}
		& \E_{(w, \xi_w) \sim \ctP} \langle \bar{w}, v \rangle \langle \bar{w}, \bar{v} \rangle^{2j - 2} \langle \xi_w , \bar{w} \rangle  \langle \bar{w}, \xi_v \rangle \\
		=& \E_{(w, \xi_w) \sim \ctP, w \notin \set{S}_{pot}} \langle \bar{w}, v \rangle \langle \bar{w}, \bar{v} \rangle^{2j - 2} \langle \xi_w , \bar{w} \rangle  \langle \bar{w}, \xi_v \rangle \\
		& + \sum_{j \in [d]} \E_{(w, \xi_w) \sim \ctP, w \in \set{S}_{j, singleton}} \langle \bar{w}, v \rangle \langle \bar{w}, \bar{v} \rangle^{2j - 2} \langle \xi_w , \bar{w} \rangle  \langle \bar{w}, \xi_v \rangle
		+  \sum_{w \in \set{S}_{ignore}}\langle \nabla_{2j, v, w}, \xi_v \rangle,
	\end{align*}
	where $\nabla_{2j, v, w}$ is defined in Eq~\eqref{eq_2j_v_w}.
	For the first term, using the running hypothesis $\ctH_1$ that for every $w \notin \set{S}_{pot}$, $\| \bar{w} \|_{\infty}^2 \leq \frac{c_t}{d}$, we have that $\langle \bar{w}, \bar{v} \rangle^2 \leq \frac{c_t}{d}$. This implies that
	\begin{align*}
		& \sum_{j \geq 2} \left| \E_{(w, \xi_w) \sim \ctP, w \notin \set{S}_{pot}} \langle \bar{w}, v \rangle \langle \bar{w}, \bar{v} \rangle^{2j - 2} \langle \xi_w , \bar{w} \rangle  \langle \bar{w}, \xi_v \rangle  \right| \\
    \leq & \frac{c_t \poly(\Kapppa)}{d^2} \|v\|_2  \E_{(w, \xi_w) \sim \ctP, w \notin \set{S}_{pot}} \|  \xi_w\|_2 \| \xi_v \|_2 \\
    \leq & \frac{c_t \poly(\Kapppa)}{d^2}  \| \xi_v \|_2^2 + \| v\|_2^2 \frac{c_t \poly(\Kapppa)}{d^2} \E_{(w, \xi_w) \sim \ctP, w \notin \set{S}_{pot}} \|  \xi_w\|_2^2.
	\end{align*}
	For the second term, we have that when $j = i$,  we have for every $w \in \set{S}_{j, singleton}$: $ \langle \bar{w}, \xi_v \rangle = 0$. Otherwise, when $j \not= i$, we have that $\langle \bar{w} , \bar{v }\rangle  = 0$.
	Therefore,
	\begin{align*}
		\sum_{j \in [d]} \E_{(w, \xi_w) \sim \ctP, w \in \set{S}_{j, singleton}} \langle \bar{w}, v \rangle \langle \bar{w}, \bar{v} \rangle^{2j - 2} \langle \xi_w , \bar{w} \rangle  \langle \bar{w}, \xi_v \rangle = 0.
	\end{align*}
	For the third term, we have that
	\begin{align*}
		\left| \sum_{w \in \set{S}_{ignore}}\langle \nabla_{2j, v, w}, \xi_v \rangle \right| &\leq \sum_{w \in \set{S}_{ignore}} \| w\|_2^2 \| \xi_v \|_2 \leq \frac{1}{\poly(d)} \| \xi_v \|_2.
	\end{align*}
	Hence we have that
	\begin{align*}
		\left| \langle \tilde{\nabla}_{2j, v, p}, \xi_v \rangle \right| \leq  \frac{c_t \poly(\Kapppa)}{d^2}  \| \xi_v \|_2^2  + \frac{c_t \poly(\Kapppa)}{d^2} \E_{(w, \xi_w) \sim \ctP, w \notin \set{S}_{pot}} \|  \xi_w\|_2^2  +  \frac{1}{\poly(d)} \| \xi_v \|_2.
	\end{align*}
	For $p = 4$, we also have:
	\begin{align*}
		& \left| \E_{(w, \xi_w) \sim \ctP} \langle {w}, \bar{v} \rangle \langle \bar{w}, \bar{v} \rangle^{2j - 2} \langle \xi_v , \bar{v} \rangle  \langle \bar{w}, \xi_v \rangle \right| \\
		\leq& \left| \E_{(w, \xi_w) \sim \ctP, w \notin \set{S}_{pot}} \langle {w}, \bar{v} \rangle \langle \bar{w}, \bar{v} \rangle^{2j - 2} \langle \xi_v , \bar{v} \rangle  \langle \bar{w}, \xi_v \rangle \right|  +  \frac{1}{\poly(d)} \| \xi_v \|_2 \\
		\leq& \frac{c_t \poly(\Kapppa)}{d^2}  \| \xi_v \|_2^2  + \frac{c_t \poly(\Kapppa)}{d^2} \E_{(w, \xi_w) \sim \ctP, w \notin \set{S}_{pot}} \|  \xi_w\|_2^2  +  \frac{1}{\poly(d)} \| \xi_v \|_2.
	\end{align*}
	Now for $p = 11, 12$, we also know
	\begin{align*}
		\langle \tilde{\nabla}_{2j, v, p}, \xi_v \rangle = 0.
	\end{align*}
	For $p = 9, 15$, following the same calculation by dividing $w$ into three parts we can easily conclude that
	\begin{align*}
		&\left|\langle   \tilde{\nabla}_{2j, v, 9} +  \tilde{\nabla}_{2j, v, 15 } \rangle  \right| \\
		\leq& \left| b_{2j}' (a_i - \tilde{\gamma_i})  \right| \| \xi_v \|_2^2   + \frac{c_t \poly(\Kapppa)}{d^2}  \| \xi_v \|_2^2  + \frac{c_t \poly(\Kapppa)}{d^2} \E_{(w, \xi_w) \sim \ctP, w \notin \set{S}_{pot}} \|  \xi_w\|_2^2  +  \frac{1}{\poly(d)} \| \xi_v \|_2,
	\end{align*}
	where $\tilde{\gamma}_i = \E_{w \sim \ctP, w \in  \set{S}_{i, pot} \backslash \set{S}_{bad}} w_i^2$.
	Therefore we finish the proof with Eq~\eqref{eq_err_gamma}.
\end{proof}

Next, we consider the error of the neurons not in $\set{S}_{pot}$.
We use a direct corollary of Claim~\ref{claim:err_stage_2} , except that for every $v \notin \set{S}_{pot}$, it holds that $\langle \bar{v}, \bar{w} \rangle^2 \leq \frac{c_t}{d}$ instead of $1$ for every vector $w$.
We state the result as follows.
\begin{claim}\label{lem:err_stage_22}
	In the setting of Lemma \ref{lem:error_final_stage_22}, let $\sigma_{\max} =  \max \left\{ \| \E_{w \sim \ctP} ww^{\top} \|_2, \frac{\kappa}{d } \right\}$.
	Let $\ctP = \ctP^{(t)}$ denote the distribution of the neurons.
	For every $v$ such that $\|\bar{v}\|_{\infty}^2, \|v\|_{\infty}^2 \leq \frac{c_t}{d}$ and any $\alpha \geq 1$, we have:
	\begin{align*}
		\sum_{j \geq 2} \sum_{p} | \langle \tilde{\nabla}_{2j, v, p}, \xi_v \rangle| \leq  O \left( \alpha \frac{c_t}{d} \sigma_{\max} \| \xi_v \|_2^2+  \frac{1}{\alpha} \frac{c_t}{d} \| v \|_2^2   \E_{\xi_w \sim \ctP}  \| \xi_w\|_2^2 \right).
	\end{align*}
\end{claim}

Based on Claim \ref{lem:final_err_3333} and \ref{lem:err_stage_22}, we prove Lemma \ref{lem:error_final_stage_22}.
\begin{proof}[Proof of Lemma \ref{lem:error_final_stage_22}]%
	Let $\alpha = \sqrt{d}$ in Claim \ref{lem:err_stage_22}, we show the following result the bound.
	For every $v \notin \set{S}_{pot}$,
	{\small\begin{align*}
		 \sum_{j \geq 2}\langle \tilde{\nabla}_{2j, v}, \xi_v \rangle
		\geq  - O \left( |\tPhi| \| \xi_v \|_2^2   + \frac{c_t \poly(\Kapppa)}{d^{1.5}}  \| \xi_v \|_2^2 + \frac{c_t \poly(\Kapppa)}{d^{1.5}} \|v \|_2^2 \E_{(w, \xi_w) \sim \ctP } \|  \xi_w\|_2^2  +  \frac{1}{\poly(d)} \| \xi_v \|_2\right).
	\end{align*}}%
	Together with the individual error bound as in Claim~\ref{claim:err_stage_2}, we can obtain the desired result using a similar proof to Lemma~\ref{lem:error_final_stage_1}.
	The details are omitted.
\end{proof}

\section{Proof of Lower Bound}\label{app_lb}

We follow the proof of Theorem 2 in \citet{AL2019-resnet} for proving the lower bound.
We first describe the construction of the hardness distribution $\set{W}$.
We first show the following lemma.
\begin{lemma}\label{lem:design}
For a positive integer $r$, for every $d \geq r^2$ which is a multiple of $r$, there exists at least $H = d^{\Omega(r)}$ many sets $\set{C}^{(j)} =  \{ \set{C}_{1}^{(j)} \in [d], \cdots, \set{C}_{d/ r}^{(j)} \in [d] \}$ for $j = 1,\dots,Q$ such that \tnote{what's the rang of $j$?}
\begin{enumerate}
\item For every $1\le i\le d/r$ and $1\le j\le H$, $\set{C}_i^{(j)}$ is a subset of $[d]$ of size $r$.
\item  For every $1\le i \not= i'\le d/r$ and $1\le j\le H$, $\set{C}_i^{(j)} \cap \set{C}_{i'}^{(j)} = \emptyset$.
\item For every $1\le i, i'\le d/r$ and $1\le j \not= j'\le H$, $\set{C}_i^{(j)} \not= \set{C}_{i'}^{(j')}$.
\end{enumerate}
\end{lemma}

\begin{proof} %
We consider a uniformly at random distribution over the set $\set{C} =  \{ \set{C}_{1}, \cdots, \set{C}_{d/ r} \}$, where $\set{C}_i$ is a subset of $[d]$ of size $r$ and for every $i \not= i'$,  we have that $\set{C}_i \cap \set{C}_{i'} = \emptyset$. Let us sample $Q$ many sets $\{ \set{C}^{(j)} \}_{j \in [Q]}$ from it, then using union bound, we have that:
\begin{align*}
\Pr\left[ \exists j \not= j' , i, i' \text{ such that } \set{C}_i^{(j)} = \set{C}_{i'}^{(j')} \right] \leq Q^2\left( \frac{d}{r} \right)^2  \left( \frac{r}{d} \right)^r.
\end{align*}
Hence when $d \geq r^2$, for some $H = d^{O(r)}$, the above probability is smaller than one.
This proves the existence of these sets.
\end{proof}

Now, we define the distribution $\set{W}$.
\textcolor{black}{Recall that the Hadamard transform of dimension $r$ is a unitary matrix in dimension $r$ whose entries are all $\in \{-1/\sqrt{r}, 1/\sqrt{r} \}$.}
\begin{definition}[The hardness distribution for the lower bound]\label{defn:lb}
For every $r$ that is a power of $2$, for every $d$ that is a multiple of $r$ bigger than $r^2$, we generate $\set{W}$ as:
\begin{enumerate}
\item Pick $\set{C}$ uniformly at random from the set $\{ \set{C}^{(j)}  \}_{j \in [H]}$ given by Lemma~\ref{lem:design}.
\item Define $w_{i}^{\star} \in \mathbb{R}^d$ with $i = pr + q$, for $p \in \{0, 1, \cdots d/r - 1\}$ and $q \in [r]$ as:
\begin{align*}
w_{i}^{\star} = (0^{ p r}, h^{\star}_q, 0^{d - ( p + 1) r}),
\end{align*}
where $h^{\star}_q $ is the i-th column of the Hadamard transform of dimension $r$.
\item Sample $b_1, \cdots, b_d$ independent from $[1, 2]$ uniformly at random. Define
\begin{align*}
	a_i = \frac{b_i}{\sum_{j \in [d]} b_j}.
\end{align*}
\end{enumerate}
\end{definition}

The proof of the lower bound relies on the following Lemma.
\begin{lemma}[The boolean analysis lemma]\label{lem:lb_critical}
For every even $r \in  \mathbb{N}^{\star}$, let $\mu = (\mu_1, \mu_2, \cdots, \mu_{r}) \in \real^r$ be sampled from the Gaussian distribution $\cN(0, \id_{r \times r})$.
With probability at least $r^{- O(r)}$ over the choice of $\mu$, it holds that:
{\small\begin{align*}
\lambda_{\mu} := \left| \E_{\tau \sim Uniform(\{-1, 1\}^r)}\left[ \Abs{ \sum_{i \in [r]} \mu_i \tau_i } \prod_{i \in [r]} \tau_i \right] \right| \geq r^{-O(r)}
\end{align*}}
\end{lemma}

To prove this Lemma, we use Lemma $F.2$ and the proof of Corollary $7.1$  in \citet{AL2019-resnet}, which says the following.
\begin{corollary}[Lemma $F.2$ and Corollary $7.1$ in \citet{AL2019-resnet}] \label{cor:critical}
For every $\veps > 0$, there exists a value $V_{r, \veps} = ( r \log \frac{1}{\veps} )^{O(r)}$ and a function $h: \mathbb{R}^r \to [V_{r, \veps} , V_{r, \veps} ]$ such that for every $\tau \in \{-1, 1\}^r$, it holds that:
{\small\begin{align*}
 \E_{\mu \sim  \cN(0, \id_{d \times d})}\left[  \left| \sum_{i \in [r]} \mu_i \tau_i \right| h(\mu)  \right] = \prod_{i \in [r]} \tau_i \pm \veps
\end{align*}}
\end{corollary}

Using this Corollary, we can prove Lemma~\ref{lem:lb_critical}.
\begin{proof}[Proof of Lemma~\ref{lem:lb_critical}]
By applying Corollary~\ref{cor:critical} with $\veps = 0.5$, we have that there exists a value $V =  r^{O(r)} $ and a function $h: \mathbb{R}^r \to [V , V]$ such that
{\small\begin{align*}
 \E_{\mu \sim  \cN(0, \id_{d \times d})}\left[  \abs\left( \sum_{i \in [r]} \mu_i \tau_i \right) h(\mu)  \right] = \prod_{i \in [r]} \tau_i \pm 0.5.
\end{align*}}
Hence we have that by $\tau_i \in \{-1, 1 \}$:
{\small\begin{align*}
 \E_{\mu \sim  \cN(0, \id_{d \times d})}\left[  h(\mu) \abs\left( \sum_{i \in [r]} \mu_i \tau_i \right)  \prod_{i \in [r]} \tau_i  \right] = 1  \pm 0.5,
\end{align*}}
which means that
{\small\begin{align*}
 \E_{\mu \sim  \cN(0, \id_{d \times d}); \tau \sim Uniform(\{-1, 1\}^r)}\left[  h(\mu) \abs\left( \sum_{i \in [r]} \mu_i \tau_i \right)  \prod_{i \in [r]} \tau_i  \right] \geq \frac{1}{2}.
\end{align*}}
This immediately implies that 
{\small\begin{align*}
 \E_{\mu \sim  \cN(0, \id_{d \times d})}\left[  |h(\mu)| \left| \E_{\tau \sim Uniform(\{-1, 1\}^r)}\left[ \abs\left( \sum_{i \in [r]} \mu_i \tau_i \right) \prod_{i \in [r]} \tau_i \right] \right|  \right] \geq \frac{1}{2}.
\end{align*}}
Using the fact that $|h(\mu)| \leq V$, we have that 
{\small\begin{align*}
 \E_{\mu \sim  \cN(0, \id_{d \times d})}\left[ \lambda_{\mu} \right] \geq \frac{1}{2 V}
\end{align*}}
Notice that with probability at least  $1 - e^{r^2}$ over $\mu$, we have that $\lambda_{\mu} \leq  r^{O(r)}$. Note that $\lambda_{\mu } \geq 0$ as well. Thus, using Markov's inequality we complete the proof.
\end{proof}

Next we can derive the following corollary of Lemma~\ref{lem:lb_critical}.
For two vectors $x,y$ with the same dimension, we denote $x \circ y$ as the entry-wise product of $x,y$.
\begin{corollary}\label{cor:lb_critical}
Let $p_1, \cdots, p_r$ be $r$ vectors in $\{-1/\sqrt{r}, 1/\sqrt{r} \}^r$, let  $q_1, \cdots, q_r$ be i.i.d. random variable chosen uniformly at random from $[1, 2]$, define $F_{\mu}(\tau) = \sum_{i \in [r]} q_i \Abs{\langle p_i \circ \mu, \tau\rangle}$, we have that with probability at least $r^{- O(r)}$ over $\mu \sim  \cN(0, \id_{d \times d})$ and $q$:
\begin{align*}
\lambda_{\mu}^{\star} := \left| \E_{\tau \sim Uniform(\{-1, 1\}^r)}\left[  F_{\mu}(\tau) \prod_{i \in [r]} \tau_i \right] \right|  \geq r^{-O(r)}.
\end{align*}
\end{corollary}

Finally, we can complete the proof of Theorem \ref{thm_lb}.
\begin{proof}[Proof of Theorem \ref{thm_lb}]
We prove by contradiction. Suppose on the contrary that equation \eqref{eq_lb_result} does not hold.
Then, there exists $\geq 0.01$ fraction of $\{ a_i, w^{\star} \}_{i \in [d]}$ generated from $\set{W}$ such that for some $w^{(R)}$ we have $\set{R}(x) :=  w_R^{\top} \phi(x)$, and it holds that
\begin{align}\label{eq:Fvxnkascfhjsaf}
 \E_{x \sim \cN(0, \id_{d \times d})} \left( f^{\star}(x) - \set{R}(x) \right)^2 = o\left( \frac{1}{d } \right)
\end{align}
We consider $x = \bar{x} \circ \tau$ where $\bar{x} \sim  \cN(0, \id_{d \times d}) $ and $\tau \sim  Uniform(\{-1, 1 \}^d )$. Clearly, $x \sim  \cN(0, \id_{d \times d})$ as well.
Thus, 
\begin{align*}
\E_{x \sim \cN(0, \id_{d \times d})} \left( f^{\star}(x) - \set{R}(x) \right)^2 &=  \E_{\tau \sim Uniform(\{-1, 1 \}^d ) ; \bar{x} \sim \cN(0, \id_{d \times d})} \left( f^{\star}(\bar{x} \circ \tau) - \set{R}(\bar{x} \circ \tau) \right)^2
\end{align*}
Therefore, by Markov's inequality we have that with probability at least $0.999$ over the choice of $\bar{x}$, we have that
\begin{align*}
  \E_{\tau \sim Uniform(\{-1, 1 \}^d ) }   \left( f^{\star}(\bar{x} \circ \tau) - \set{R}(\bar{x} \circ \tau) \right)^2 = O\left( \E_{x \sim \cN(0, \id_{d \times d})} \left( f^{\star}(x) - \set{R}(x) \right)^2\right)
  \end{align*}
  Now we perform Boolean Fourier analysis over $f^{\star}(\bar{x} \circ \tau) $ and $\set{R}(\bar{x} \circ \tau)$. For a function $f : \{-1, 1\}^d \to \mathbb{R}$, we define it's Fourier expansion as:
\begin{align*}
  f(\tau) = \sum_{\set{B} \subseteq [d]} \lambda_{\set{B}} \prod_{j \in \set{B}} \tau_j,
\end{align*}
where $\lambda_{\set{B}}$ if the Fourier coefficient of the subset $B$.
Now, define $\lambda_{\set{B}}^{\star}$ to be the Fourier coefficients of $f^{\star}(\bar{x} \circ \tau)$ and $\lambda_{\set{B}}^{\set{R}}$ to be the Fourier coefficient of $\set{R}(\bar{x} \circ \tau)$, we can observe that if we sample $\set{C}$ from $\set{W}$ to generate $w^{\star}$ according to Definition~\ref{defn:lb}, then it holds that for every $\set{B} \subset [d]$ of size $r$, we have:
\begin{align*}
	\set{B} \notin \set{C} \implies  \lambda_{\set{B}}^{\star} = 0.
\end{align*}
Moreover, using Corollary~\ref{cor:lb_critical}, we can conclude that w.p. at least $0.999$ over $\set{W}$,
\begin{align} \label{eq:fjaiosfjaofjasi}
  \sum_{\set{B} \in \set{C}} \left(\lambda_{\set{B}}^{\star} \right)^2 \geq \frac{r^{- O(r)}}{d}.
\end{align}
On the other hand, for every $\veps > 0$, as long as $  \E_{\tau \sim Uniform(\{-1, 1 \}^d ) }   \left( f^{\star}(\bar{x} \circ \tau) - \set{R}(\bar{x} \circ \tau) \right)^2 \leq \veps$, we have that
\begin{align*}
	\sum_{\set{B} \in \set{C}} \left(\lambda_{\set{B}}^{\star} - \lambda_{\set{B}}^{\set{R}} \right)^2 +  \sum_{\set{B} \notin \set{C} } \left( \lambda_{\set{B}}^{\set{R}} \right)^2 \leq \veps
\end{align*}
Let us consider the set $\set{S}_{gd}$ of $\{ a_i, w^{\star} \}_{i \in [d]}$ generated from $\set{W}$. We call $\{ a_i, w_i^{\star} \}_{i \in [d]} \in \set{S}_{gd}$ if and only if the function $f^{\star}$ defined using $\{ a_i, w^{\star} \}_{i \in [d]}$ satisfies Eq~\eqref{eq:fjaiosfjaofjasi} and there is a $w^{(R)}$ such that for $\set{R}(x) :=  w_R^{\top} \phi(x)$ with
\begin{align*}%
 \E_{x \sim \cN(0, \id_{d \times d})} \left( f^{\star}(x) - \set{R}(x) \right)^2 = o\left( \frac{1}{d } \right)
\end{align*}
We already know that there are at least $0.999$ fraction $\{ a_i, w^{\star} \}_{i \in [d]}$ generated from $\set{W}$ that satisfies Eq~\eqref{eq:fjaiosfjaofjasi}.
By our assumption, there are $\geq 0.01$ fraction of $\{ a_i, w^{\star} \}_{i \in [d]}$ generated from $\set{W}$ satisfying that for some $w^{(R)}$ such that $\set{R}(x) :=  w_R^{\top} \phi(x)$, it holds that
\begin{align*}%
 \E_{x \sim \cN(0, \id_{d \times d})} \left( f^{\star}(x) - \set{R}(x) \right)^2 = o\left( \frac{1}{d } \right)
\end{align*}
Thus, we can conclude $|\set{S}_{gd}| \geq 0.005 |\set{W}|$.
Together with Lemma~\ref{lem:design} which shows that $|\set{W}| \geq d^{\Omega(r)}$, we know that $|\set{S}_{gd}| \geq d^{\Omega(r)}$.

Now, we consider a matrix $M$, whose rows are indexed by each set of $\{ a_i, w_i^{\star} \}_{i \in [d]} \in \set{S}_{gd}$  and Eq~\eqref{eq:Fvxnkascfhjsaf}, whose columns are indexed by $\lambda_{\set{B}}^{\star}$ with $|\set{B}| = r$.

We know that this matrix is of size $d^{\Omega(r)} \times d^{\Omega(r)}$. Moreover, for any matrix $M'$ satisfies that
\begin{align*}
\forall i,  \|M_i - M'_{i} \|_2^2 = o\left( \frac{1}{d } \right)
\end{align*}
where $M_i$ is the $i$-th row of $M$.
It must holds that $\rank(M') =  d^{\Omega(r)}$. We immediately complete the proof by contradiction, following exactly the same argument in the lower bound proof in \cite{AL2019-resnet} while taking $r$ to be a sufficiently large constant.
\end{proof}

\end{document}